\documentclass{article}

\usepackage[main, final,  nonatbib]{neurips_2026}
 \usepackage{lineno}

\usepackage[utf8]{inputenc} \usepackage[T1]{fontenc}    \usepackage{hyperref}       \usepackage{url}            \usepackage{booktabs}       \usepackage{amsfonts}       \usepackage{nicefrac}       \usepackage{microtype}      \usepackage{xcolor}         

\usepackage{titletoc}

\usepackage[sort&compress]{natbib}
\usepackage{graphicx}
\usepackage{caption}
\usepackage{wrapfig}

\usepackage{algorithm}
\usepackage{algpseudocode}

\usepackage{enumitem}
\usepackage{empheq}

\usepackage{xcolor}

\usepackage{amsmath,amssymb,mathtools}
\usepackage{bbm}
\usepackage{xfrac}

\newcommand{\EE}{\mathbb{E}}

\newcommand{\NN}{\mathbb{N}}

\newcommand{\RR}{\mathbb{R}}

\newcommand{\Iin}{\mathbf{1}}

\newcommand{\OW}{{\rm OW}}
\newcommand{\ET}{{\rm ET}}
\newcommand{\Pe}{{\rm Pe}}

\newcommand{\GSen}{\Cc}

\newcommand{\EmpKly}{{\rm Emp_{Kly}}}

\newcommand{\Aa}{\mathcal{A}}

\newcommand{\Cc}{\mathcal{C}}

\newcommand{\Ff}{\mathcal{F}}
\newcommand{\Gg}{\mathcal{G}}

\newcommand{\Oo}{\mathcal{O}}

\newcommand{\Ss}{\mathcal{S}}

\newcommand{\Zz}{\mathcal{Z}}

\newcommand{\goal}{{\rm goal}}
\newcommand{\skill}{{\rm skill}}

\newcommand{\tS}[1]{_{#1}}

\newcommand\independent{\protect\mathpalette{\protect\independenT}{\perp}}
\def\independenT#1#2{\mathrel{\rlap{$#1#2$}\mkern2mu{#1#2}}}

\usepackage{aliascnt}
\usepackage{amsthm}

\newtheorem{mytheor}{Theorem}[section]

\newtheorem{mylemma}{Lemma}[section]
\newtheorem{myprop}{Proposition}[section]

\newtheorem{mycorrol}{Corollary}[section]

\definecolor{mygrey}{RGB}{125, 125, 125}
\definecolor{mygrey2}{RGB}{100, 100, 100}
\definecolor{mygrey3}{RGB}{50, 50, 50}
\definecolor{mygreen}{RGB}{0, 120, 0}
\definecolor{myred}{RGB}{120, 0, 0}
\definecolor{myred2}{RGB}{200, 0, 0}
\definecolor{myblue}{RGB}{0, 0, 120}
\definecolor{myblue2}{RGB}{0, 0, 180}
\definecolor{mypurple}{HTML}{862990}

\hypersetup{
    colorlinks,
linkcolor={myred},
citecolor={mygrey2},
    urlcolor={blue!80!black}
}

\usepackage{pifont}

\newcommand{\cbl}[1]{\textcolor{blue}{#1}}

\usepackage{setspace}

\newcommand{\appropto}{\mathrel{\vcenter{
  \offinterlineskip\halign{\hfil$##$\cr
    \propto\cr\noalign{\kern2pt}\sim\cr\noalign{\kern-2pt}}}}}

\newenvironment{exmode}{\par\color{myblue}}{\par \vspace{0.3cm}}

\usepackage{pgfplotstable}
\usepackage{booktabs}
\pgfplotsset{compat=1.18} 

\title{Unifying Goal-Conditioned RL and Unsupervised Skill Learning via Control-Maximization}

\author{Alireza Modirshanechi$^{1,2,3,*}$,
    $\,\,$Benjamin Eysenbach$^{4}$,
    $\,\,$Peter Dayan$^{2,5,\dagger}$,
    $\,\,$Eric Schulz$^{1,\dagger}$
    \vspace{0.2cm}\\
    \small{$^1$ Helmholtz Munich, Germany} \quad
    \small{$^2$ Max Planck Institute for Biological Cybernetics, Germany}\\
    \small{$^3$ University of Göttingen, Germany} \quad
    \small{$^4$ Princeton University, USA}\\
    \small{$^5$ University of Tübingen, Germany
    \vspace{0.2cm}}\\
    \small{$^*$ \texttt{alireza.modirshanechi@uni-goettingen.de}}\\
    \small{$^\dagger$ These authors contributed equally to this work}
}

\begin{document}

\maketitle
\begin{abstract}
    Unsupervised pretraining has driven empirical advances in goal-conditioned reinforcement learning (GCRL), but its theoretical foundations remain poorly understood. In particular, an influential class of methods, mutual information skill learning (MISL), discovers behaviorally diverse skills that can later be used for downstream goal-reaching. However, it remains a theoretical mystery why skills learned through MISL should support goal-reaching. A subtle challenge is that both GCRL and MISL are umbrella terms: Different GCRL tasks use distinct criteria for measuring goal-reaching performance, while different MISL methods optimize distinct notions of behavioral diversity. We address this challenge and unify GCRL and MISL as instances of control maximization. We identify three canonical GCRL formulations and prove that they are fundamentally inequivalent: They can induce incompatible optimal policies even in the same environment. Nevertheless, they all share a common interpretation: A well-performing goal-conditioned policy is one whose future trajectory is highly sensitive to the commanded goal, with the precise notion of sensitivity determined by the GCRL formulation. Noting that MISL objectives can be understood as measures of skill sensitivity akin to goal sensitivity, we show that MISL objectives are bounded by formulation-specific downstream goal sensitivities. These bounds establish a precise correspondence between MISL methods and downstream GCRL tasks: For every GCRL formulation, there exists a matching MISL objective for which more diverse skills afford greater downstream goal sensitivity. Our results thus lay a theoretical foundation for RL pretraining and have important practical implications, such as suggesting which pretraining objectives to use when a user cares about a specific class of downstream tasks.
\end{abstract}

\section{Introduction}

Much of the past success of reinforcement learning (RL), from mastering Atari games~\citep{mnih2015human} to defeating human champions in Go~\citep{silver2016mastering}, has focused on solving single-objective tasks.
However, a burgeoning recent line of work explores the multi-task setting, where a single agent must learn solutions to several diverse tasks, akin to the success of multi-task pretraining that has driven advances in natural language processing~\citep{brown2020language, wei2021finetuned} and computer vision~\citep{radford2021learning, he2022masked}.
In particular, \textit{goal-conditioned RL} (GCRL) addresses an important special case where agents must learn to be able to pursue arbitrary goals, whether internally generated or externally commanded~\citep{kaelbling1993learning, sutton2011horde, schaul2015universal, veeriah2018many, florensa2018automatic, park2025dual, bortkiewicz2024accelerating, wang2023optimal}.
GCRL has seen striking empirical success in recent years, largely driven by pretraining methods that learn reusable skills and representations prior to downstream goal-reaching~\citep{florensa2018automatic, eysenbach2020c, eysenbach2022contrastive, mendonca2021discovering, ma2022offline, park2023metra}.
Among these methods, \textit{mutual information skill learning} (MISL) has emerged as an influential framework for discovering diverse behaviors, where diversity is measured by the mutual information (MI) between skills and their behavioral consequences~\citep{gregor2016variational, achiam2018variational, eysenbach2018diversity, sharma2019dynamics, zheng2024can} (see~\cite{park2023metra, park2022lipschitz} for alternatives).
MISL skills have been empirically shown to support downstream GCRL, yet why this is the case remains theoretically puzzling.
Specifically, from an information-theoretic perspective, MISL trains agents to learn skills that are maximally \textit{communicable} through behavior: an observer can reliably infer which skill an agent is executing simply by watching it act.
But why should the skills that are useful for this communication problem also be useful for goal-reaching?

A central obstacle in addressing this question is the multiplicity of MISL methods and GCRL formulations.
On the MISL side, different methods rely on different quantifications of behavioral diversity (\autoref{fig:general}, right).
For example, some methods consider two skills to be distinct if they lead the agent to different terminal states~\citep{sharma2019dynamics, levy2023hierarchical} or different state occupancies~\citep{eysenbach2021information}, while others consider the entire behavioral trajectory to differ~\citep{achiam2018variational, laskin2022cic}.
On the GCRL side, different tasks rely on distinct criteria to measure goal-reaching performance (\autoref{fig:general}, left).
For example, in some settings the goal \textit{persistently} exists, and the agent is rewarded repeatedly for revisiting it~\citep{eysenbach2020c, eysenbach2022contrastive, mendonca2021discovering, ma2022offline}; in others, reward is given only if the agent reaches the goal \textit{at an exact time}~\citep{ghosh2019learning, pong2019skew, warde2018unsupervised}; in others still, reward is given if the goal is reached at any point \textit{within an opportunity window}~\citep{schaul2015universal, veeriah2018many, florensa2018automatic, park2025dual}.
While these differences are often treated as minor modeling choices, we show that they define genuinely distinct optimization problems with fundamentally different optimal policies.
This raises a critical question.
\textit{How can we link MISL to GCRL theoretically when both terms refer to distinct families of problems?}
Specifically, \textit{is there a theoretical account to identify which MISL objective is best suited to which GCRL formulation?}

We answer these questions by observing that all MISL methods and GCRL formulations, despite their differences, share a common interpretation: they all seek to maximize the agent's degree of \textit{control} over its environment.
This shared perspective allows us to derive systematic bounds linking MISL objectives to controllability and, in turn, to downstream GCRL performance.
The practical takeaway is clear: No MISL objective is optimal for \textit{every} flavor of goal-reaching, and the appropriate pretraining objective should be chosen to match the goal-reaching criterion expected at test time.
Accordingly, our work bridges GCRL and MISL through four main contributions (\autoref{fig:general}); \textit{all standing assumptions and limitations are explicitly discussed in \autoref{sec:prelim} and \autoref{sec:conc}.}

\begin{figure}[b!]
    \centering 
    \includegraphics[width=0.9\textwidth]{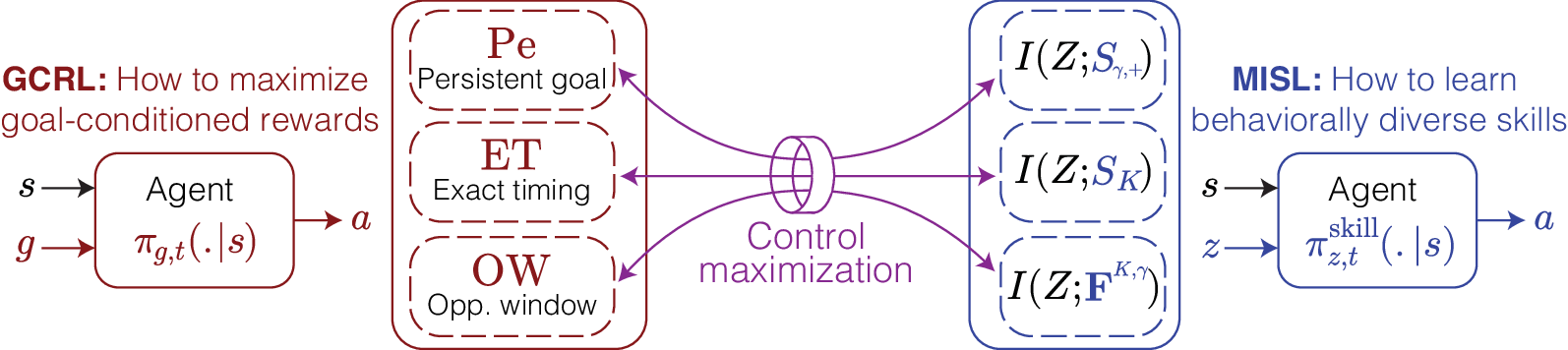}
    \caption{We unify GCRL and MISL as control-maximization problems.
    \textit{Left.} GCRL aims to find a policy for `reaching' any goal $g$.
    Different GCRL algorithms define `goal reaching' according to different formulations (see \autoref{fig:formulations}).
    \textit{Right.} MISL aims to find a policy that exhibits `diverse behavior' depending on the choice of skill $z$, quantified by the mutual information $I(Z; S')$ between skill $Z$ and its behavioral consequences $S'$.
    Different MISL algorithms rely on different definitions of $S'$ (see \autoref{sec:empowerment}).
    We prove a correspondence between different GCRL formulations and MISL objectives.}
    \label{fig:general}
\end{figure}

\textbf{Contribution 1. Distinct GCRL formulations are generally inequivalent} (\autoref{prop:settings_contradicting}, \autoref{fig:formulations}).
We identify three canonical GCRL formulations and show that they generally induce different optimal policies.
Hence, algorithms designed for one formulation may not perform well in another.

\textbf{Contribution 2. Distinct GCRL formulations share the same optimal policy under special conditions} (\autoref{prop:stationary}--\ref{prop:OWETwaiting}, \autoref{fig:taxonomy}).
We identify conditions under which algorithms designed for one formulation can be safely transferred to another.

\textbf{Contribution 3. GCRL objectives are equivalent to maximizing formulation-specific goal sensitivity} (\autoref{theor:decomp}, \autoref{fig:Emp}).
Goal sensitivity is a novel measure of an agent's degree of control that reflects both the environment's intrinsic controllability and the agent's competence at goal-directed behavior.
This serves as a unifying lens across all three GCRL formulations.

\textbf{Contribution 4. Each GCRL formulation has a matching MISL objective} (\autoref{theor:MI_empowerment}, \autoref{prop:MIZunifvsf}, \autoref{fig:GSenEmp}).
We show that formulation-specific goal sensitivities bound distinct MISL objectives, revealing which MISL objective is most suitable for each downstream GCRL formulation.

\section{Related work} \label{sec:related-work}

We will build our theory upon the idea that control maximization is the key to bridging GCRL and MISL, but we should first ask what it means for an agent to be `in control.'
Informally, an agent is in control if it can steer the course of events in accordance with its `intentions.'
We argue that GCRL optimizes quantities that intuitively align with this notion of control, while MISL objectives are formal variants of existing measures that explicitly quantify it.
After briefly reviewing this related work, the subsequent sections show that these intuitive connections can be made precise.

\textbf{GCRL performance as an implicit measure of controllability.}
GCRL algorithms aim to train agents that can reach arbitrary goals in their environment~\citep{kaelbling1993learning, sutton2011horde, bortkiewicz2024accelerating, agarwal2023f, schaul2015universal}.
A well-performing agent must therefore be able to steer the course of events toward its goal.
This is possible only if the agent has sufficient control over its future trajectory, a perspective that aligns with recent theories that define control as preparedness for performing diverse tasks~\citep{modirshanechi2025integrative}.
From this perspective, modern GCRL methods can be seen as different approaches to control maximization, whether through training on difficult-to-reach goals~\citep{florensa2018automatic}, replaying past experiences~\citep{andrychowicz2017hindsight, mendonca2021discovering}, or learning contrastive representations of goals~\citep{eysenbach2020c, eysenbach2022contrastive}.
Yet despite sharing this intuition, these methods rely on different goal-reaching formulations without explicitly distinguishing or relating them.
We make this control-maximization view precise and show when these differences have serious consequences.

\textbf{Explicit accounts of control and empowerment.}
Explicit mathematical definitions of control generally fall into two categories: control-theoretic~\citep{bansal2017hamilton, chilakamarri2025reachability, bansal2021deepreach, abate2008probabilistic, thorpe2019model, thorpe2021approximate, sontag2013mathematical, ogata2020modern} and information-theoretic~\citep{klyubin2005empowerment, salge2014empowerment, jung2011empowerment, leibfried2019unified}.
RL algorithms most commonly draw on the latter, with Klyubin empowerment as the most representative example: the maximum MI between actions and future states~\citep{klyubin2005empowerment}.
Informally, Klyubin empowerment measures the degree to which an agent can steer the course of events through its choice of actions.
This notion and its variants have been used in RL in various ways, from intrinsic rewards that guide agents toward controllable parts of the environment~\citep{leibfried2019unified, bharadhwaj2022information, jung2011empowerment, becker2021exploration, gruaz2024merits} to learning signals that help agents build controllable representations~\citep{cao2025towards, burda2018large, pathak2017curiosity}.
However, both the proper multi-step definition of empowerment and its relationship to RL objectives remain debated~\citep{klyubin2005empowerment, capdepuy2011informational, myers2024learning, abel2025plasticity, modirshanechi2025integrative}.
We show that our precise control-maximization view of GCRL closely relates to existing empowerment-like measures.

\textbf{MISL objectives as variants of empowerment.}
Alternatively, we can quantify control not through the choice of actions, but through the choice of \textit{skills}, where a skill describes a potentially complex action policy.
This leads to variants of Klyubin empowerment, defined as the MI between skills and future states.
This is precisely how MISL objectives quantify behavioral diversity~\citep{gregor2016variational, achiam2018variational, eysenbach2018diversity, sharma2019dynamics, zheng2024can}, but existing methods differ substantially in the specifics: They may consider different variables as the agent's `future state'~\citep{sharma2019dynamics, levy2023hierarchical, eysenbach2021information, achiam2018variational, laskin2022cic}, optimize marginal or conditional MI~\citep{eysenbach2018diversity, gregor2016variational, zheng2024can, sharma2019dynamics, levy2023hierarchical}, use fixed or infinite horizons~\citep{sharma2019dynamics, zheng2024can}, and rely on different variational bounds~\citep{choi2021variational, mohamed2015variational, poole2019variational, laskin2022cic}.
Which variant works best is often decided empirically, based on both qualitative measures (e.g., skill interpretability to a human inspector) and quantitative ones (e.g., skill classification accuracy)~\cite{achiam2018variational}.
A particularly important quantitative measure is the downstream GCRL performance when using the learned skills~\cite{achiam2018variational, zheng2024can, gregor2016variational, eysenbach2018diversity}.

Here, we ask why maximizing an MISL objective should necessarily result in good downstream GCRL performance.
This is a fundamental open question in RL pretraining, and prior work has provided only partial answers for specific formulation-objective pairs: Some works show that, under a one-to-one goal-state correspondence, a specific variational bound on an MISL objective equals to a specific GCRL performance~\citep{choi2021variational, levy2023hierarchical}, while others show how learned skill distributions may minimize certain notions of regret~\citep{eysenbach2021information} or reflect ground-truth environmental structure~\citep{reizinger2025skill}.
Our contribution is complementary: Without restricting to a single formulation-objective pair or imposing a strict skill-to-goal correspondence, we establish that each MISL objective is theoretically linked to the downstream GCRL performance of its matching formulation via precise information-theoretic bounds. 
\section{Notation and preliminaries}
\label{sec:prelim}
We consider an agent that interacts with an environment with the state space $\Ss$, a state-dependent action space $\Aa(s)$ at state $s$, and the probability $p(s'|s,a)$ for the transition $(s,a) \to s'$.
We assume that $\Ss$ and $\Aa(s)$ are finite, with $N_s := |\Ss| > 1$ and $N_a(s) := |\Aa(s)|$.
We use capital letters to represent random variables, but we omit the capital letter notation when no ambiguity arises.

\paragraph{Goal-conditioned policy and value.}
In its most general case, we assume that, at time $t \in \NN$, given a goal $g \in \Gg$, the agent follows a \textit{goal-conditioned}, \textit{non-stationary policy} $\pi_{g,t}$ (\autoref{fig:general}, left):
\begin{linenomath*} \begin{equation} \begin{aligned} \label{eq:GCpolicy}
    \pi_{g,t}(a|s) := 
p^{\pi_{g,t}}(A_t = a|S_t = s).
\end{aligned} \end{equation} \end{linenomath*}
We use $\pi_{\{.,.\}} := \{ \pi_{g,t} \}_{g \in \Ss, t \in \NN}$ to denote the set of all \textit{non-stationary}, \textit{goal-conditioned} policies.
As special cases, we denote a goal-conditioned \textit{stationary} policy by $\pi_{\{g,-\}}$, and a non-stationary but \textit{goal-independent} policy by $\pi_{\{-,t\}}$.
Given $\pi_{\{g,.\}}$, we define $p^{\pi_{\{g,.\}}}(\tau_t | s_0)$ as the probability of the agent's trajectory $\tau_t := (a_{0}, s_{1}, \dots, a_{t-1}, s_{t})$, starting from $s_{0}$.
The agent's objective will be defined in terms of a goal-dependent reward function $R_t(s; g): \NN \times \Ss \times \Gg \to \RR$ and discount factor $\gamma_t(s; g) : \NN \times \Ss \times \Gg \to [0,1]$, with $\gamma_0(s; g) := 1$. 
While these terms are typically constant over time, our analysis shows that making them time-dependent enables us to describe different GCRL formulations with the same unified notation (\autoref{sec:divergence});
given $g \in \Gg$, the GCRL objective is
\begin{linenomath*} \begin{equation} \begin{aligned} 
    \label{eq:general_value}
    J(s, g, \pi_{\{.,.\}}) &:= 
    \EE^{\pi_{\{g,.\}}} \left[ 
    \sum_{t=1}^\infty  R_t(S_t; g) \prod_{k=0}^{t-1} \gamma_k(S_k; g)
    \Big| S_0 = s\right].
\end{aligned} \end{equation} \end{linenomath*}

\paragraph{Uninformative goal distribution for the test-time performance.}
We consider a task-agnostic, general-purpose evaluation setting. 
We assume the space of goals is the same as the state space (i.e., $\Gg = \Ss$), and posit a uniform goal distribution, i.e., $G \sim p_\goal = {\rm Unif}(\Ss)$.
Then, the test-time performance measures how well the agent performs for reaching a uniformly sampled goal state, i.e.,
\begin{linenomath*} \begin{equation} \begin{aligned} 
    \label{eq:general_expvalue}
    J(s, \pi_{\{.,.\}}) &:= 
    \underbrace{\EE_{\cbl{G} \sim p_\goal}}_{\text{averaging over all goals}} 
    \Big[
    \underbrace{J(s, \cbl{G}, \pi_{\{.,.\}})}_{\text{performance for reaching $\cbl{G}$}}
    \Big] =
    \frac{1}{N_s} \sum_{\cbl{g} \in \Ss} J(s, \cbl{g}, \pi_{\{\cbl{g},.\}}).
\end{aligned} \end{equation} \end{linenomath*}
An optimal policy $\pi^*_{\{.,.\}}$ is one that maximizes the expected return $J(s, g, \pi^*_{\{.,.\}})$ for reaching any goal $g$ from any state $s$, i.e., $J(s, g, \pi^*_{\{.,.\}}) = J^*(s, g) := \max_{\pi_{\{.,.\}}} J(s, g, \pi_{\{.,.\}})$~\citep{puterman2014markov}.
As a result, $\pi^*_{\{.,.\}}$ also maximizes the test-time performance: $J(s, \pi^*_{\{.,.\}}) = J^*(s) := \max_{\pi_{\{.,.\}}} J(s, \pi_{\{.,.\}})$.
\textit{Almost all our results also hold for non-uniform goal distributions} $p_\goal$ (see \autoref{append:goal_dist}).

\paragraph{Unsupervised skill discovery and MISL.}
Unsupervised skill discovery assumes a skill set $\Zz$ and a skill-conditioned policy $\pi_{z,t}^\skill (a |s)$ for each skill $z \in \Zz$ (notation analogous to \autoref{eq:GCpolicy}).
We assume that $\Zz$ is finite, with $N_z := |\Zz| > 1$.
MISL typically assumes the fixed uniform skill prior, $Z \sim {\rm Unif}(\Zz)$, and learns diverse skill-conditioned policies by maximizing MI between skills $Z$ and a behavioral variable $S'$ (e.g., $S_K$~\citep{sharma2019dynamics, gregor2016variational} or $\tau_K$~\citep{achiam2018variational, laskin2022cic}) that is defined based on the agent's trajectory~\citep{eysenbach2018diversity, gregor2016variational, achiam2018variational, zheng2024can, sharma2019dynamics}:
\begin{linenomath*} \begin{equation} \begin{aligned} \label{eq:MISLobj}
    \text{Skill-behavior MI:} \quad
    J_{\rm MISL}(s, \pi_{\{.,.\}}^\skill; S')
    :=
    I^{\pi_{\{.,.\}}^\skill}_{Z \sim {\rm Unif}(\Zz)}
    \!\left(Z; S' \mid S_0=s \right),
\end{aligned} \end{equation} \end{linenomath*}
where skills are sampled uniformly, and $S'|Z$ is sampled by running $\pi_{\{Z,.\}}^\skill$.
Intuitively, maximizing this objective results in skills that are maximally distinguishable based on $S'$.
A successful choice~\citep{eysenbach2021information} of objective is $S' = S_{\gamma,+}$ which is defined as $S_K$ for $K \sim {\rm Geom}(1-\gamma)$ with $\gamma \in [0,1)$. 
The distribution of $S_{\gamma,+}$ is given by the normalized discounted state occupancy~\citep{eysenbach2020c, eysenbach2022contrastive, dayan1993improving}.

Pretraining via MISL provides a skill-conditioned policy that likely induces different behaviors across skills.
The pretrained policy can then be used for downstream GCRL by learning (or assuming) a downstream mapping $f: \Gg \to \Zz$ from goals to skills~\citep{zheng2024can, park2023metra, park2022lipschitz}.
This provides a goal-conditioned policy:
\begin{wrapfigure}{r}{0.4\textwidth}
  \centering
  \includegraphics[width=\linewidth]{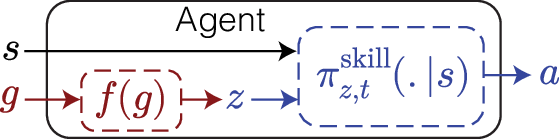}
\end{wrapfigure}
\begin{linenomath*} \begin{equation} \begin{aligned} \label{eq:skillpolicy}
    \pi_{g,t}(a|s) := \pi_{f(g),t}^\skill(a|s).
\end{aligned} \end{equation} \end{linenomath*}
Consistent with prior work~\citep{zheng2024can, park2023metra, park2022lipschitz, choi2021variational}, we assume that the mapping $f$ is deterministic.

\paragraph{Summary of standing assumptions.}
Unless otherwise stated, Sections~\ref{sec:divergence}--\ref{sec:empowerment} assume (i) finite state, action, and skill spaces, (ii) state-goal equivalence ($\Gg=\Ss$), (iii) a uniform test-time goal distribution, (iv) a uniform MISL skill prior, and (v) a deterministic goal-to-skill map $f:\Gg\to\Zz$. 
\autoref{append:goal_dist} relaxes the uniform-goal assumption; the remaining restrictions are discussed in \autoref{sec:conc}. 
\section{Different GCRL formulations are incompatible} 
\label{sec:divergence}

We first ask how different studies specify GCRL through the choice of $R_t(s; g)$ and $\gamma_t(s; g)$ in \autoref{eq:general_value}.
Typically, $R_t(s; g)$ decreases with the `distance' between $s$ and $g$ so that the states closer to the goal receive higher reward~\cite{choi2021variational, pong2018temporal, nair2018visual}.
We focus on sparse rewards, i.e., $R_t(s; g)=0$ for all $s \neq g$, and consider three representative formulations (\autoref{fig:formulations}, left).

\textbf{1. Persistent goal}, $\Pe(\gamma)$, models infinite-horizon settings with a persistent goal~\cite{eysenbach2020c, eysenbach2022contrastive, mendonca2021discovering, ma2022offline}.
At $t=0$, the agent is given a random goal $g \sim p_\goal$.
It then receives a reward $1-\gamma \in (0,1]$ whenever it visits $g$, with future rewards discounted by $\gamma$.
In the limit $\gamma \to 1$, $\Pe(\gamma)$ recovers target-occupancy optimization~\cite{krass2002achieving, dufour2022maximizing} (\autoref{prop:stationary}; \autoref{fig:taxonomy}).
Using \autoref{eq:general_value}, $\Pe(\gamma)$ is defined, for $t\geq1$, by
\begin{linenomath*} \begin{equation} \begin{aligned}
    R_t(s; g) := (1 - \gamma) \Iin_{s=g} 
    \quad \text{and} \quad
    \gamma_t(s; g) := \gamma,
\end{aligned} \end{equation} \end{linenomath*}
which yields
\begin{linenomath*} \begin{equation} \begin{aligned}
    \label{eq:Pe_J}
    J(s, g, \pi_{\{.,.\}}) &=
    J_\Pe(s, g, \pi_{\{.,.\}}, \gamma) =
    p^{\pi_{\{g,.\}}}(S_{\gamma,+} = g | S_0 = s).
\end{aligned} \end{equation} \end{linenomath*}
\textbf{2. Exact timing}, $\ET(K)$, models settings in which the agent must reach the goal in exactly $K \in \NN$ steps~\cite{ghosh2019learning, pong2019skew, warde2018unsupervised}.
At $t=0$, the agent is given a random goal $g \sim p_\goal$ and receives reward 1 only if it reaches $g$ at $t=K$.
Using \autoref{eq:general_value}, $\ET(K)$ is defined, for $t\geq1$, by
\begin{linenomath*} \begin{equation} \begin{aligned}
    R_t(s; g) := \Iin_{s=g}\Iin_{t=K}
    \quad \text{and} \quad
    \gamma_t(s; g) := 1,
\end{aligned} \end{equation} \end{linenomath*}
yielding
\begin{linenomath*} \begin{equation} \begin{aligned}
    \label{eq:ET_J}
    J(s, g, \pi_{\{.,.\}})
    = J_\ET(s, g, \pi_{\{.,.\}}, K)
    = p^{\pi_{\{g,.\}}}(S_K = g | S_0 = s).
\end{aligned} \end{equation} \end{linenomath*}
If the agent does not know $K$ and instead assumes $K \sim {\rm Geom}(1-\gamma)$, then $\ET(K)$ becomes equivalent to $\Pe(\gamma)$ (\autoref{pro:PeETbar}; \autoref{fig:taxonomy}).

\textbf{3. Opportunity window}, $\OW(K,\gamma)$, models settings in which the agent must reach the goal \emph{within} $K \in \NN$ steps~\cite{schaul2015universal, veeriah2018many, florensa2018automatic, park2025dual}.
At $t=0$, the agent is given a random goal $g \sim p_\goal$ and receives reward 1 if it reaches $g$ at some time $t \leq K$.
To favor faster goal reaching, rewards within the window are discounted by $\gamma \in [0,1]$.
Thus, $\OW(K,\gamma)$ generalizes stochastic shortest-path and maximum reward-rate objectives~\cite{kaelbling1993learning, veeriah2018many, wang2023optimal} (\autoref{prop:min_lenght}; \autoref{fig:taxonomy}).
Using \autoref{eq:general_value}, $\OW(K,\gamma)$ is defined, for $t\geq1$, by
\begin{linenomath*} \begin{equation} \begin{aligned}
    R_t(s; g) := \Iin_{s=g}\Iin_{t \leq K}
    \quad \text{and} \quad
    \gamma_t(s; g) := \gamma \Iin_{s \neq g}.
\end{aligned} \end{equation} \end{linenomath*}
Let $T_g := \min \{t \geq 1 : S_t = g\}$.
Then
\begin{linenomath*} \begin{equation} \begin{aligned}
    \label{eq:OW_J}
    J(s, g, \pi_{\{.,.\}})
    = J_\OW(s, g, \pi_{\{.,.\}}, K, \gamma)
    = \EE^{\pi_{\{g,.\}}}\!\left[
    \gamma^{T_g-1} \Iin_{T_g \leq K}
    | S_0 = s \right].
\end{aligned} \end{equation} \end{linenomath*}
These formulations cover much of the GCRL literature, and are natural, application-relevant, and theoretically distinct. 
However, we note that they are also not exhaustive; e.g., one could define a variant of $\OW$ in which the goal remains persistent but only \textit{within} the opportunity window~\cite{agarwal2023f}.

\begin{figure}[b!]
    \centering 
    \includegraphics[width=1\textwidth]{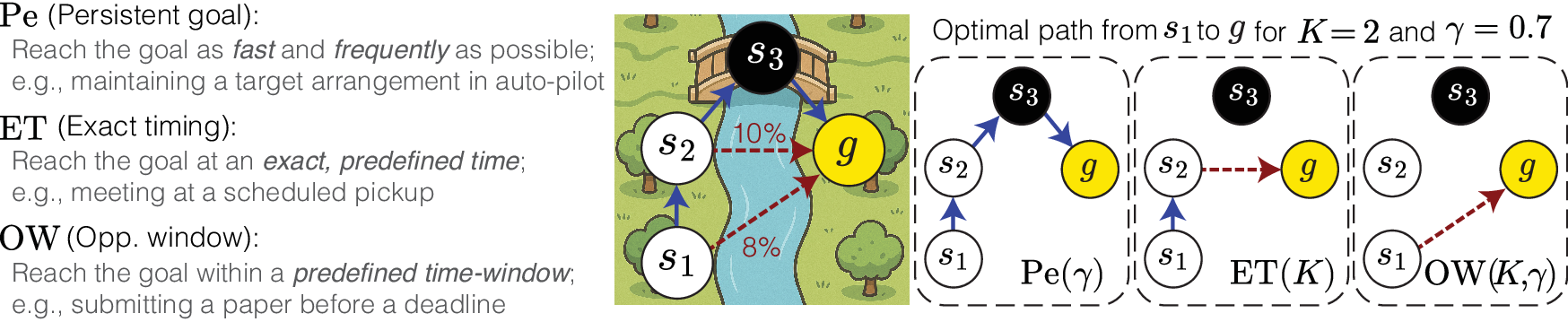}
    \caption{Different GCRL formulations yield incompatible optimal policies (\autoref{prop:settings_contradicting}).
    }
    \label{fig:formulations}
\end{figure}

Our first theoretical result is that the three formulations can induce different optimal policies.
\begin{myprop}
    \label{prop:settings_contradicting}
    Consider the three formulations $\OW(K, \gamma)$, $\Pe(\gamma)$, and $\ET(K)$.
    There exist an environment $p(.|.,.)$, a horizon $K \in \NN$, and a discount factor $\gamma \in [0,1)$ such that the optimal policies under each formulation are different.
\end{myprop}
The proof is provided in \autoref{append:proofs}, but the intuition is that incompatibility becomes important when the agent needs to choose to reach the goal either quickly \textit{or} reliably.
For example, consider an agent that must cross a river to reach $g$ from $s_1$ (\autoref{fig:formulations}, right).
Suppose jumping has a low success rate (e.g., 8\% from $s_1$), while taking the bridge is slower.
Then, depending on $\gamma$ and $K$, the formulations disagree on whether or when to jump (\autoref{fig:formulations}, right).

\subsection{Why the identified theoretical incompatibility matters in practice}

The small differences in these objectives often translate into large differences in algorithms for each setting.
For example, C-Learning~\cite{eysenbach2020c}, contrastive RL~\cite{eysenbach2022contrastive}, LEXA~\cite{mendonca2021discovering}, and GoFAR~\cite{ma2022offline} either exploit the stationary, recursive Bellman structure of $\Pe$ or directly rely on the equivalence between its objective and discounted state occupancy (\autoref{eq:Pe_J}).
Because $\ET$ and $\OW$ do not share these properties, those algorithms do not transfer to these finite-horizon formulations without non-trivial changes.
Conversely, many goal-sampling and hindsight-relabeling methods are designed for finite-horizon episodic settings~\cite{andrychowicz2017hindsight, pong2019skew, nair2018visual, ghosh2019learning, florensa2018automatic}.
For instance, the simplest relabeling rule in Hindsight Experience Replay (HER)~\cite{andrychowicz2017hindsight} replaces the commanded goal with the final state of a length-$K$ trajectory, inducing a terminal-state bias closely aligned with $\ET(K)$.
By contrast, Goal GAN~\cite{florensa2018automatic} is built around the probability of reaching a goal within $K$ steps and therefore aligns more naturally with $\OW$.

In short, many GCRL algorithms are coupled to a particular formulation.
The practical lesson of \autoref{prop:settings_contradicting} is therefore straightforward: Slight mismatches between the formulations assumed by the learning algorithms and those used for test-time evaluation may result in substantial performance loss. 
To help avoid such mismatches in practice, \autoref{sec:empowerment} establishes a correspondence between MISL methods and GCRL formulations that guides the choice of MISL objectives.

\subsection{Equivalence conditions: When general incompatibility does not matter in practice} \label{sec:equivalence}

Despite the general incompatibility, there are various special cases in which the formulations, or the optimal GCRL policies they induce, are identical. 
For example, consider the one-step case in which the agent is rewarded only if it reaches the goal in a single step, as soon as the goal is commanded.
This is a special case of all three formulations, meaning that methods developed for any of the three apply equally well to the others in this case (\autoref{prop:one_step_case}; \autoref{fig:taxonomy}).
\begin{figure}[b!]
    \centering
    \includegraphics[width=1\textwidth]{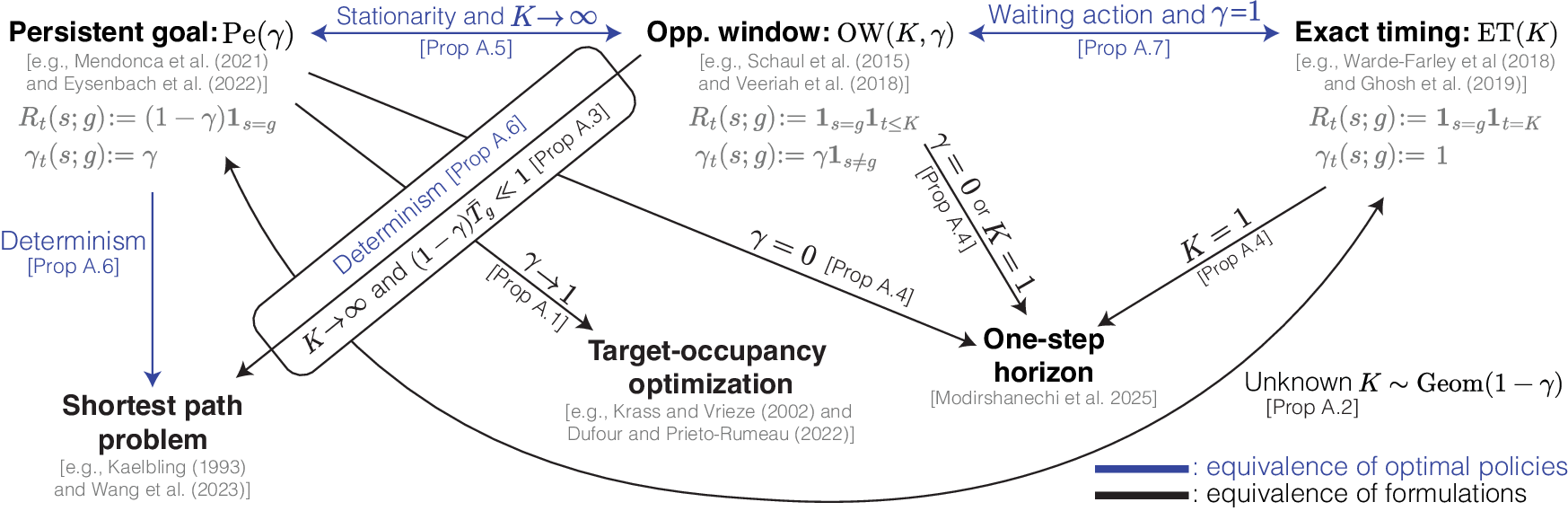}
    \caption{Equivalence conditions. Black edges indicate identical policy orderings; blue edges indicate common optimal policies. See \autoref{append:props}.}
    \label{fig:taxonomy}
\end{figure}

We have identified several natural conditions under which different formulations share optimal policies (blue edges in \autoref{fig:taxonomy}).
For example, \autoref{prop:OWPeinfK} shows that if the opportunity window is arbitrarily long ($K \to \infty$), then $\OW(K, \gamma)$ and $\Pe(\gamma)$ share the same stationary optimal policy.
An immediate consequence is that policies learned by methods developed for $\Pe$ to repeatedly visit a goal state (e.g., contrastive RL~\cite{eysenbach2022contrastive}) can also be used for $\OW(\infty, \gamma)$ to quickly reach the goal, even though there is no reward for re-visitation of the goal.
\autoref{prop:OWPeDeterm} shows that, if the environment is deterministic, then this remains true even for finite opportunity windows.
In other words, in deterministic environments (which is the case for some canonical GCRL benchmarks~\cite{yu2020meta, bortkiewicz2024accelerating}), both formulations boil down to finding the shortest path to the goal, and neither the duration of the opportunity window nor whether the goal is persistent matters.
A similar result holds for $\OW(K,\gamma = 1)$ and $\ET(K)$ when the environment provides `waiting' actions that allow the agent to stay at its current state (\autoref{prop:OWETwaiting}).
In simple terms, if the agent gets to a goal state as fast as it can and `waits' there, then its behavior is optimal for both $\OW(K,\gamma = 1)$ and $\ET(K)$.
Together, these results clarify when policies learned for one formulation can be applied to another (\autoref{fig:taxonomy}).
 
\section{A precise control-maximization account of GCRL}
\label{sec:control}

In the previous section, we showed that GCRL comprises incompatible formulations.
Yet, across all formulations, a well-performing agent must \textit{intuitively} have a high degree of control over its future trajectory to pursue its desired goal.
Here, we make this intuition precise.

Regardless of the formulation, a well-performing goal-conditioned policy $\pi_{\{.,.\}}$ should achieve a high value $J(s, g, \pi_{\{g,.\}})$ when conditioned on the correct goal $g$.
Motivated by this, we call a goal-conditioned policy $\pi_{\{.,.\}}$ \textit{consistent} if the policy is better at achieving goal $g$ when aiming for goal $g$ than when aiming for another goal \textcolor{myred2}{$g'$}, i.e.,
\begin{linenomath*} \begin{equation} \begin{aligned} \label{def:consistency}
    \text{consistency condition:}
    \underbrace{J(s, g, \pi_{\{g,.\}})}_{\text{pursuing $g$ and being rewarded by $g$}}
    \geq
    \underbrace{J(s, g, \pi_{\{\textcolor{myred2}{g'},.\}})}_{\text{pursuing \textcolor{myred2}{$g'$} while being rewarded by $g$}}
\end{aligned} \end{equation} \end{linenomath*}
for every $s,g,\textcolor{myred2}{g'} \in \Ss$ (see \autoref{append:assump} for attainability of this condition).
To measure the degree of consistency, we can compute the average performance gain from conditioning on the commanded goal rather than on a randomly chosen one:
\begin{linenomath*} \begin{equation} \begin{aligned} \label{def:control}
    \boxed{\text{goal sensitivity:} \quad
    \GSen(s, \pi_{\{.,.\}}) := \frac{1}{N_s^2} \sum_{g, \textcolor{myred2}{g'} \in \Ss}
    \Big( J(s, g, \pi_{\{g,.\}}) - J(s, g, \pi_{\{\textcolor{myred2}{g'},.\}}) \Big).}
\end{aligned} \end{equation} \end{linenomath*}
This difference, which we will refer to as \textbf{goal sensitivity}, formalizes the intuitive notion of controllability discussed in prior work (see~\autoref{sec:related-work}): It measures how much the agent's intention to pursue a particular goal influences the critical parts of its future trajectory. 
Goal sensitivity is bounded and, for consistent policies, non-negative:
\begin{linenomath*} \begin{equation} \begin{aligned} \label{eq:optimal_sensitivity}
    0 \leq \GSen(s, \pi_{\{.,.\}}) \leq \GSen^*(s) := \max_{\pi_{\{.,.\}}} \GSen(s, \pi_{\{.,.\}}) < \infty,
\end{aligned} \end{equation} \end{linenomath*}
where $\GSen^*(s)$ is the maximal goal sensitivity that an agent can afford in state $s$.
In \autoref{append:props_snes}, we show that $\GSen^*(s)$ can be precisely linked to common measures of the environment's \emph{objective controllability}, including Klyubin empowerment~\cite{klyubin2005empowerment}.
Hence, we can view the policy-dependent quantity $\GSen(s, \pi_{\{.,.\}})$ as the agent's subjective \emph{degree of control} over its future trajectory.
An agent can increase its degree of control through learning (\autoref{fig:Emp}B vs. C), but it remains fundamentally bounded by the environment's objective controllability $\GSen^*(s)$ (\autoref{fig:Emp}A).

\begin{figure}[H]
    \centering
    \includegraphics[width=1\textwidth]{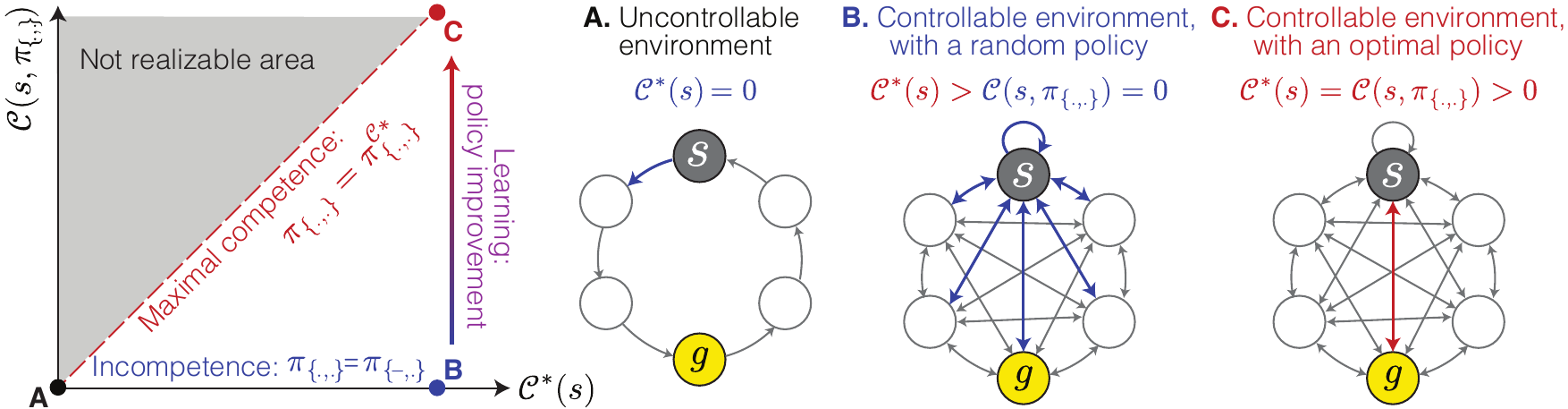}
    \caption{Goal sensitivity $\GSen(s,\pi_{\{.,.\}})$ reflects both objective controllability $\GSen^*(s)$ and agent competence. 
    In uncontrollable environments, it is zero (A); in fully controllable environments, it depends on whether the policy ignores goals (B) or reliably selects goal-reaching actions (C).}
\label{fig:Emp}
\end{figure}

Accordingly, alongside the optimal policy $\pi^{*}$ that maximizes $J$ in \autoref{eq:general_expvalue}, we can define a maximally in-control policy $\pi^{\GSen*}$ that maximizes $\GSen$ in \autoref{def:control}.
But does maximizing control in the sense of $\GSen$ always help GCRL performance?
\autoref{theor:decomp} shows that this is indeed the case: $\GSen(s, \pi_{\{.,.\}})$ is monotonically equivalent to GCRL performance for $\Pe$ and $\ET$ and provides a tight lower bound for $\OW$ or any GCRL formulation with non-negative rewards.
This implies that the maximally in-control policy $\pi^{\GSen*}_{\{.,.\}}$ is optimal for $\Pe$ and $\ET$ and tightly bounds the regret for $\OW$ (proof in \autoref{append:proofs}).
\begin{mytheor} \label{theor:decomp}
    Consider a GCRL formulation defined by $R_t(s; g)$ and $\gamma_t(s; g)$. Then,
    \begin{enumerate}[leftmargin=2.5em, itemsep=0pt]
        \item For $\Pe(\gamma)$ and $\ET(K)$, we have
        $J(s, \pi_{\{.,.\}}) = \GSen(s, \pi_{\{.,.\}}) + \frac{1}{N_s}$.
        As a result, a policy $\pi^{\GSen*}_{\{.,.\}}$ is maximally in-control iff it is optimal.

        \item For $\OW(K, \gamma)$ and any formulation with non-negative rewards and bounded returns, we have
        $J(s, \pi_{\{.,.\}})  \geq \frac{N_s}{N_s - 1}\GSen(s, \pi_{\{.,.\}}).$
        Equality holds iff $J(s, g, \pi_{\{g',.\}}) = 0$ for all $g' \neq g$.

        \item For $\OW(K, \gamma)$, there exist an environment $p(.|.,.)$, $K \in \NN$, and $\gamma \in [0,1]$ such that a maximally in-control policy is not optimal.
        Nevertheless, for $\OW(K, \gamma)$ in any environment,
        \begin{linenomath*} \begin{equation} \begin{aligned}
            0 \le J^*(s) - J(s, \pi^{\GSen*}_{\{.,.\}}) \le 1 - \frac{N_s}{N_s - 1}\GSen^*(s),
        \end{aligned} \end{equation} \end{linenomath*}
so larger $\GSen^*(s)$ yields a smaller regret bound.
    \end{enumerate}
\end{mytheor}
Hence, maximizing goal sensitivity is (approximately) sufficient for solving the underlying GCRL problem.
This result makes our control-maximization interpretation of GCRL precise and provides novel insights into why learning behaviorally diverse skills or accurate goal representations can benefit GCRL.
The next section further formalizes this connection.
 
\section{Different GCRL formulations need different MISL objectives}
\label{sec:empowerment}

We now return to our original question: \textit{why does MISL help downstream goal-reaching?}
To answer this, we study the goal-conditioned policy $\pi_{g,t}$ built from a skill-conditioned policy $\pi_{z,t}^\skill$ via the goal-to-skill mapping $z = f(g)$ as in \autoref{eq:skillpolicy}.
We proceed in two steps.
\textbf{First}, we derive an information-theoretic approximation of the goal sensitivity of $\pi_{g,t}$ (\autoref{theor:MI_empowerment}; \autoref{fig:GSenEmp}, middle).
\textbf{Second}, we show that this approximation tightly bounds the MISL objective (\autoref{prop:MIZunifvsf}; \autoref{fig:GSenEmp}, right).

\begin{figure}[H]
    \centering 
    \includegraphics[width=1\textwidth]{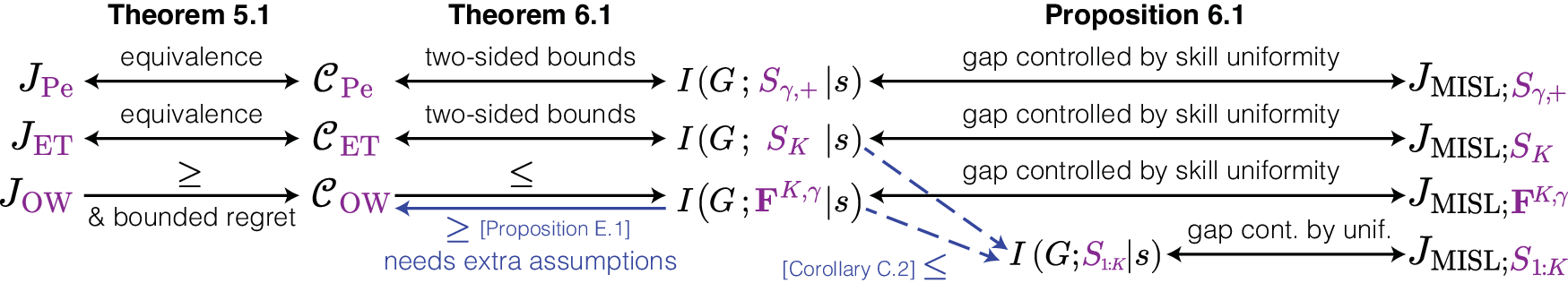}
    \caption{The precise correspondence of the MISL objectives to the downstream GCRL performance; see also \autoref{fig:GSenEmpApp} in \autoref{append:props_emp}.
    }
    \label{fig:GSenEmp}
\end{figure}

\subsection{Information-theoretic approximation of goal sensitivity}

Let $S'$ be one of the behavioral variables commonly used in MISL objectives (e.g., $S_K$ or $S_{\gamma,+}$; \autoref{eq:MISLobj}).
We can quantify how sensitively $S'$ depends on the commanded goal $G$ via their MI,
\begin{linenomath*} \begin{equation} \begin{aligned} \label{eq:MIGS}
    \text{Goal-behavior MI:} \quad
    I^{\pi_{\{.,.\}}}
    \!\left(G; S' \mid S_0=s \right),
\end{aligned} \end{equation} \end{linenomath*}
where $G$ is sampled uniformly and $S'|G$ is obtained by running $\pi_{\{G,.\}}$.
This goal-behavior MI closely resembles the skill-behavior MI in \autoref{eq:MISLobj}, but its explicit dependence on goals allows us to link it to goal sensitivity: We show that different goal-behavior MIs are bounded by the goal sensitivity of their matching GCRL formulation.
Specifically, for $\Pe(\gamma)$ and $\ET(K)$, the goal sensitivity provides attainable bounds on the goal-behavior MI with $S' = S_{\gamma,+}$ and $S' = S_K$, respectively.
On the other hand, $\OW(K,\gamma)$ exposes a gap in the existing MISL objective; our bound points to a novel behavioral variable: the discounted first-visit vector $S' = \mathbf{F}^{K,\gamma}$, an $N_s$-dimensional vector whose $g$th entry encodes the discounted reward of first reaching state $g$.
The closest existing candidate in the literature is the full trajectory $\tau_K$~\citep{achiam2018variational, laskin2022cic}, which yields only a loose upper bound (\autoref{corr:trajstate}).

The precise statement is given in \autoref{theor:MI_empowerment} (proof in \autoref{append:proofs}; summary in Figs.~\ref{fig:GSenEmp} and~\ref{fig:GSenEmpApp}).
\begin{mytheor}
\label{theor:MI_empowerment}
    Let
    $\GSen_{\Pe}(s,\pi_{\{.,.\}},\gamma)$,
    $\GSen_{\ET}(s,\pi_{\{.,.\}},K)$, and
    $\GSen_{\OW}(s,\pi_{\{.,.\}},K,\gamma)$
    be the goal sensitivities associated with
    $\Pe(\gamma)$, $\ET(K)$, and $\OW(K,\gamma)$ for policy $\pi_{\{.,.\}}$. 
    \begin{enumerate}[leftmargin=1.5em, itemsep=0pt]
    \item For $\Pe(\gamma)$ and $\ET(K)$, we have attainable lower bounds,
    \begin{linenomath*}
    \begin{empheq}[box=\fbox]{equation}
    \begin{aligned}
    I^{\pi_{\{.,.\}}}\!\left(G;\textcolor{mypurple}{S_{\gamma,+}}\mid S_0=s\right)
    &\ge
    \Phi^{\rm down}_{N_s}\!\Big(
    N_s^{-1} + 
    \GSen_{\textcolor{mypurple}{\Pe}}(s,\pi_{\{.,.\}},\gamma)
    \Big),\\
    I^{\pi_{\{.,.\}}}\!\left(G;\textcolor{mypurple}{S_K}\mid S_0=s\right)
    &\ge
    \Phi^{\rm down}_{N_s}\!\Big(
    N_s^{-1} + 
    \GSen_{\textcolor{mypurple}{\ET}}(s,\pi_{\{.,.\}},K)
    \Big),
    \end{aligned} \end{empheq} \end{linenomath*}
    where $\Phi^{\rm down}_N(x):=\log N-h(x)-(1-x)\log(N-1)$, with $h$ the binary entropy.
    $\Phi^{\rm down}_{N_s}(x)$ is increasing for $x = 1/N_s + \GSen \in [1/N_s, 1]$, which is always the case given \emph{consistency} (\autoref{def:consistency}).
    
    \item For $\Pe(\gamma)$ and $\ET(K)$, if $\pi_{\{.,.\}}$ is \emph{consistent} (\autoref{def:consistency}), then we have attainable upper bounds,
    \begin{linenomath*}
    \begin{empheq}[box=\fbox]{equation}
    \begin{aligned}
    I^{\pi_{\{.,.\}}}\!\left(G;\textcolor{mypurple}{S_{\gamma,+}}\mid S_0=s\right)
    &\le
    \Phi^{\rm up}_{N_s}\!\Big(
    N_s^{-1} + 
    \GSen_{\textcolor{mypurple}{\Pe}}(s,\pi_{\{.,.\}},\gamma)
    \Big),\\
    I^{\pi_{\{.,.\}}}\!\left(G;\textcolor{mypurple}{S_K}\mid S_0=s\right)
    &\le
    \Phi^{\rm up}_{N_s}\!\Big(
    N_s^{-1} + 
    \GSen_{\textcolor{mypurple}{\ET}}(s,\pi_{\{.,.\}},K)
    \Big),
    \end{aligned} \end{empheq} \end{linenomath*}
    where
    $\Phi^{\rm up}_N(x):=\log N
    -\big(\lceil x^{-1}\rceil x-1\big)\lfloor x^{-1}\rfloor\log\lfloor x^{-1}\rfloor
    -\big(1-\lfloor x^{-1}\rfloor x\big)\lceil x^{-1}\rceil\log\lceil x^{-1}\rceil$, with $\lfloor\cdot\rfloor$ denoting the floor and $\lceil\cdot\rceil := \lfloor\cdot\rfloor + 1$.
    $\Phi^{\rm up}_{N_s}(x)$ is increasing for $x = 1/N_s + \GSen \in [1/N_s, 1]$, which is always the case given \emph{consistency} (\autoref{def:consistency}).
    
    \item For $\OW(K,\gamma)$, let
    $\mathbf{F}^{K,\gamma}:=(F^{K,\gamma}_g)_{g\in\Ss}$, where
    $F^{K,\gamma}_g:=\gamma^{T_g-1}\Iin_{T_g\le K}$ and
    $T_g:=\min\{t\ge1:S_t=g\}$. 
    Then
    \begin{linenomath*}
    \begin{empheq}[box=\fbox]{equation}
    I^{\pi_{\{.,.\}}}\!\left(G;\textcolor{mypurple}{\mathbf{F}^{K,\gamma}}\mid S_0=s\right)
    \ge
    2\,\GSen_{\textcolor{mypurple}{\OW}}(s,\pi_{\{.,.\}},K,\gamma)^2.
    \end{empheq}
    \end{linenomath*}
    The inverse bound is possible with additional assumptions, e.g., as in \autoref{prop:OW_upper_bound_F}.
    \end{enumerate}
\end{mytheor}

\subsection{Goal-behavior MIs closely match their skill-based counterparts}

The final step is to connect the goal-behavior MIs (\autoref{eq:MIGS}) to their skill-based counterparts in the MISL objectives (\autoref{eq:MISLobj}).
Since the goal-to-skill mapping $z = f(g)$ is deterministic, standard information-theoretic identities~\cite{cover1999elements} give
\begin{linenomath*} \begin{equation} \begin{aligned} \label{eq:MIGSZS}
    I^{\pi_{\{.,.\}}}
    \!\left(\textcolor{myred2}{G}; S' \mid S_0=s \right)
    =
    I^{\pi_{\{.,.\}}}
    \!\left(\textcolor{myblue2}{Z}; S' \mid S_0=s \right),
\end{aligned} \end{equation} \end{linenomath*}
where the right-hand side is evaluated by sampling $\textcolor{myred2}{G}$ uniformly, setting $\textcolor{myblue2}{Z} = f(\textcolor{myred2}{G})$, and sampling $S'|\textcolor{myblue2}{Z}$ by running $\textcolor{myblue2}{\pi_{\{Z,.\}}^\skill}$.
The key subtlety is that this generative process does not necessarily yield a uniform distribution over skills, as assumed by MISL (\autoref{eq:MISLobj}).
Instead, the downstream skill distribution is determined jointly by the goal distribution $p_\goal$ and the mapping $f$:
\begin{linenomath*} \begin{equation} \begin{aligned} \label{eq:downspf}
    \text{Downstream skill distribution:} \quad
    p_f(z)
    =
    \sum_{g \in \Ss} p_\goal(g)\, \Iin_{z=f(g)}
    =
    \frac{1}{N_s} \sum_{g \in \Ss} \Iin_{z=f(g)}.
\end{aligned} \end{equation} \end{linenomath*}
Hence, the skill-based MI in \autoref{eq:MIGSZS} is \textit{not} necessarily equal to the MISL objective in \autoref{eq:MISLobj}, but we show that the gap between the two is controlled entirely by uniformity of $p_f$ (proof in \autoref{append:proofs}).
\begin{myprop} \label{prop:MIZunifvsf}
    Consider $\pi_{\{.,.\}}^\skill$ and a \emph{deterministic} goal-to-skill mapping $f$, and let $\pi_{\{.,.\}}$ and $p_f$ be as defined in \autoref{eq:skillpolicy} and \autoref{eq:downspf}, respectively.
    Suppose $S'$ takes at most $N_{s'} > 1$ values, e.g., $N_{s'} = N_s$ when $S' = S_K$.
    Then, if $N_z \leq N_{s'}$,
    \begin{linenomath*} \begin{equation} \begin{aligned}
        \Bigl|
        J_{\rm MISL}(s, \pi_{\{.,.\}}^\skill; S')
        -
        I^{\pi_{\{.,.\}}}\!\left(G; S' \mid S_0=s \right)
        \Bigr|
        \leq
        h(\delta)
        +
        \delta \log \big( N_{s'}^2 (N_{s'}-1) \big),
    \end{aligned} \end{equation} \end{linenomath*}
    with $\delta := || p_{f} - {\rm Unif}(\Zz) ||_{\rm TV}$ the total variation distance and $h$ the binary entropy function.
    If $p_\goal$ is uniform and $f$ partitions $\Ss$ into equal-size preimages, then $p_f$ is also uniform, and the gap is zero.
\end{myprop}
\autoref{prop:MIZunifvsf}, together with \autoref{theor:decomp} and \autoref{theor:MI_empowerment}, establishes that \textit{MISL objectives theoretically bound downstream GCRL performance} (\autoref{fig:GSenEmp}).
Consequently, a skill-conditioned policy trained via MISL is likely to perform well on downstream GCRL tasks, \textit{provided that the choice of $S'$ in MISL matches the downstream GCRL formulation identified in \autoref{fig:GSenEmp}}.

A further implication of \autoref{prop:MIZunifvsf} is that the dependence of the bounds on $\delta$ may reflect an inherent limitation of using MISL for downstream GCRL.
In practice, $\delta$ may depend on the MISL objective, network initialization, goal-to-skill mapping, and test-time goal distribution.
When $p_f$ is highly skewed, skills that appear diverse during pretraining may contribute little to downstream goal-reaching performance if they are rarely selected.
Thus, test-time skill uniformity $\delta$ may serve as a useful diagnostic when applying MISL to GCRL.
While prior work has focused on discovering increasingly large sets of distinct skills, e.g., through curriculum learning~\cite{achiam2018variational}, our results suggest that increasing skill diversity alone may not improve downstream performance if the additional skills are rarely selected.
This motivates future work on accounting for the expected test-time skill distribution during pretraining.

\subsection{When there is a mismatch between MISL and GCRL objectives}

The two-sided bounds in \autoref{fig:GSenEmp} establish a correspondence between MISL objectives and their matching GCRL formulations, but provide little intuition about what happens when there is a mismatch between the two.
To illustrate the consequences of such mismatches, consider the extreme case where skills are learned by maximizing $J_{{\rm MISL}; S_{1:K}}$, i.e., defining diversity in terms of trajectory, but are then used for downstream $\ET(K)$, i.e., rewarding the agent based on the terminal state $S_K$.
The key observation is that $J_{{\rm MISL}; S_{1:K}}$ only provides an upper bound on $J_{\ET}$ (\autoref{fig:GSenEmp}, bottom right), meaning there is no \textit{a priori} reason to expect skills that induce diverse trajectories $S_{1:K}$ to be useful for $\ET(K)$, even when the goal-to-skill mapping $f$ is optimized (see \autoref{append:assump}).
For example, consider a deterministic environment and a set of skills that follow distinct trajectories but reach the same terminal state, i.e., for every pair of distinct skills $z$ and $z'$, there exists some $t < K$ at which their states differ, while $S_K$ is identical across all skills.
Such a skill set fully maximizes $J_{{\rm MISL}; S_{1:K}}$, since the trajectories perfectly distinguish the skills, but yields zero goal sensitivity under $\ET(K)$, regardless of the goal-to-skill mapping $f$, since all skills reach the same terminal state.
Of course, skills learned using $J_{{\rm MISL}; S_{1:K}}$ may also happen to have distinct endpoints, but the point is that the objective itself provides no systematic preference for skills useful for $\ET(K)$.
Similar arguments extend to less extreme mismatches. 
\section{Conclusion}
\label{sec:conc}
We started by asking why unsupervised pretraining facilitates downstream GCRL, but we found the answer to be much more subtle than expected: GCRL itself comprises multiple distinct, incompatible formulations, and the relationship between pretraining and downstream tasks depends critically on which formulation is in play.
To address this, we developed a precise control-maximization framework that applies across all GCRL formulations, identified conditions under which different formulations coincide, and derived systematic bounds linking each MISL pretraining objective to its matching GCRL formulation.
The central message is clear: No universally optimal MISL objective exists that theoretically benefits all GCRL formulations; rather, the right pretraining objective must be chosen to match the downstream GCRL setting.
Overall, our work provides a theoretical foundation that connects GCRL and MISL under a common control-maximization perspective (\autoref{fig:general}).
We thus hope this will serve as a basis for both theoretical and algorithmic advances in GCRL and RL pretraining.

\textbf{Limitations.}
Our framework relies on several main assumptions, most of which are standard in theoretical GCRL but may not be universally satisfied by practical benchmarks, e.g.,~\cite{yu2020meta, bortkiewicz2024accelerating}.
Some of these assumptions can be relaxed relatively easily.
For example, we assume deterministic goal-to-skill mappings, as commonly done in zero-shot skill inference~\cite{zheng2024can, park2023metra}.
Extending our results to stochastic mappings is relatively straightforward by accounting for the additional entropy of skill selection in \autoref{sec:app:proofMIZunifvsf}.
Similarly, although our results in the main text assume a uniform test-time goal distribution, \autoref{append:goal_dist} extends our main theoretical connections to other goal distributions.

Some of the other assumptions, however, are more challenging to relax.
For example, our assumption of finite state/action spaces, while common in theoretical GCRL, is a technical simplification, as many practical benchmarks involve continuous spaces.
Generalizing the control-maximization view of GCRL to continuous spaces appears feasible under suitable continuity assumptions, but extending the GCRL-MISL links is less straightforward and may require stronger assumptions.
Similarly, while we assume the equivalence of goals and states ($\mathcal{G}=\mathcal{S}$), practical goals may specify only selected components of the state, e.g., an agent's location even when its state also includes velocity.
Extending our framework to general state-to-goal mappings $\phi: \Ss \to \Gg$ would require measuring skill diversity in terms of goal-relevant variables, e.g., through $I(Z;\phi(S'))$.
Finally, a further limitation of our work is that, unlike for $\Pe$ and $\ET$, our bound in \autoref{theor:MI_empowerment} for $\OW$ is one-directional in full generality; while \autoref{append:upper} provides an example inverse bound, identifying the most appropriate MISL objective for $\OW$ remains an open question.
Nevertheless, this theoretical difficulty is itself informative: it shows that, for $\OW$, establishing useful downstream guarantees from MISL is less straightforward than for $\Pe$ and $\ET$.
This suggests that practitioners should exercise caution when choosing MISL objectives for pretraining if the downstream goal-reaching task follows an $\OW$ formulation or some variants of it.

\newpage

\section*{Acknowledgement}

The authors thank Marvin Mathony and the anonymous reviewers for their insightful and constructive feedback on the manuscript.
This research was supported by the Deutsche Forschungsgemeinschaft (DFG) Emmy Noether Programme (A.M.; Project No. 572735958), the U.S. National Science Foundation (B.E.; Award No. 2441665), the Max Planck Society (P.D.), the Humboldt Foundation (P.D.), and the European Research Council (E.S.).
The authors declare no competing interests.

\newpage
\appendix

\startcontents[appendix]
\section*{Contents of the appendices}
\printcontents[appendix]{}{1}[2]{}

\vspace{1cm}
\paragraph{LLM usage:}
Beyond writing, editing, or formatting purposes, we acknowledge the use of large language models (LLMs) in this work for (i) brainstorming, (ii) helping with the theoretical analyses, (iii) finding relevant references, and (iv) generating the original graphic in \autoref{fig:formulations} (i.e., the sketch of the river and the trees).
In all of these cases, we independently evaluated the LLMs' output and verified all formal statements and proofs.
We take full responsibility for the originality and the rigor of our results.

\newpage
\section{Precise statements of the equivalences in \autoref{sec:divergence}} \label{append:props}

\subsection{Equivalence of formulations}

\begin{myprop}[Steady-state target-occupancy optimization]
    \label{prop:stationary}
    Consider the problem formulation $\Pe(\gamma)$.
    Then, for any goal-dependent, stationary policy $\pi_{\{.,-\}}$, we have
    \begin{linenomath*} \begin{equation} \begin{aligned} 
        \lim_{\gamma \to 1} J_\Pe(s, g, \pi_{\{.,-\}}, \gamma) = \rho^{\pi_{\{g,-\}}}(g|s)
    \end{aligned} \end{equation} \end{linenomath*}
    where $\rho^{\pi_{\{g,-\}}}(g|s)$ is the stationary probability of being at $g$ (more precisely, the Cesàro-average occupancy), starting from $s$ and following $\pi_{\{g,-\}}$.
    If the Markov chain induced by $\pi_{\{g,-\}}$ is irreducible, then  $\rho^{\pi_{\{g,-\}}}(g|s)$ is independent of $s$.
    Hence, $\Pe(\gamma)$ with $\gamma \to 1$ can be viewed as a steady-state target-occupancy optimization, a special case of the problem formulations in~\cite{krass2002achieving, dufour2022maximizing}.
\end{myprop}
\begin{exmode}
\textit{Proof:}
Consequence of the well-known results for average-reward MDPs; see~\cite{puterman2014markov}.
$\hfill \square$
\end{exmode}

\begin{myprop}[Equivalence of $\Pe$ with Geometric $\ET$] \label{pro:PeETbar}
    Let $\gamma \in [0,1)$ and let $K \sim {\rm Geom}(1-\gamma)$ with support on $\NN$, i.e., $p(K=t) = (1-\gamma)\gamma^{t-1}$.
    Suppose that the realization of $K$ is not observed by the agent.
    Define the \emph{`averaged exact-timing'} objective by
    \begin{linenomath*} \begin{equation} \begin{aligned}
        J_{\overline{\ET}}(s, g, \pi_{\{.,.\}}, \gamma)
        :=
        \EE_{K}\left[
        J_\ET(s, g, \pi_{\{.,.\}}, K)
        \right].
    \end{aligned} \end{equation} \end{linenomath*}
    Then, the problem formulation $\overline{\ET}(\gamma)$ (i.e., exact timing with unknown but geometrically distributed $K$) is equivalent to $\Pe(\gamma)$.
\end{myprop}
\begin{exmode}
\textit{Proof:}
By definition of $J_{\overline{\ET}}$ and $J_\ET$, we have
\begin{linenomath*} \begin{equation} \begin{aligned}
    J_{\overline{\ET}}(s, g, \pi_{\{.,.\}}, \gamma)
    &=
    \EE_{K}\left[
    p^{\pi_{\{g,.\}}}(S_K = g | S_0 = s)
    \right]
    \\
    &=
    (1-\gamma)\sum_{t=1}^\infty \gamma^{t-1}
    p^{\pi_{\{g,.\}}}(S_t = g | S_0 = s) = p^{\pi_{\{g,.\}}}(S_{\gamma,+} = g | S_0 = s).
\end{aligned} \end{equation} \end{linenomath*}
As a result, using \autoref{eq:Pe_J}, we have
\begin{linenomath*} \begin{equation} \begin{aligned}
    J_{\overline{\ET}}(s, g, \pi_{\{.,.\}}, \gamma)
    =
    J_\Pe(s, g, \pi_{\{.,.\}}, \gamma)
\end{aligned} \end{equation} \end{linenomath*}
for every policy $\pi_{\{.,.\}}$. 
Therefore, the proof is complete.
$\hfill \square$
\end{exmode}

\begin{myprop}[Shortest path problem]
    \label{prop:min_lenght}
    Consider the problem formulation $\OW(\infty,\gamma)$ with $\gamma = 1 - \epsilon$.
    If $\epsilon \EE^{\pi_{\{g,.\}}} \left[ T_g | S_0 = s\right] \ll 1$ for all $g \in \Ss$ for which $\min_{\pi_{\{g,.\}}} \EE^{\pi_{\{g,.\}}} \left[ T_g | S_0 = s\right] < \infty$, then we have
    \begin{linenomath*} \begin{equation} \begin{aligned} 
        J_\OW(s, g, \pi_{\{.,.\}}, \infty, \gamma) \approx 
        1 + \epsilon - \epsilon \EE^{\pi_{\{g,.\}}} \left[ 
        T_g
        | S_0 = s\right],
    \end{aligned} \end{equation} \end{linenomath*}
    which is equivalent to the stochastic shortest path (max reward-rate) problem formulation of~\cite{kaelbling1993learning, veeriah2018many, wang2023optimal}.
\end{myprop}
\begin{exmode}
\textit{Proof:}
Using the Taylor expansion $(1-\epsilon)^n = 1 - n \epsilon + \Oo(\epsilon^2)$.
$\hfill \square$
\end{exmode}

\begin{myprop}[One-step horizon case]
    \label{prop:one_step_case}
    The problem formulations $\Pe(0)$, $\ET(1)$, $\OW(1,\gamma)$, and $\OW(K,0)$ are equivalent to each other and to the one-step goal-seeking case of~\cite{modirshanechi2025integrative}.
\end{myprop}
\begin{exmode}
\textit{Proof:}
The statement is the direct consequence of the formulation.
$\hfill \square$
\end{exmode}

\newpage
\subsection{Equivalence of optimal policies}

\begin{myprop}[$\OW$-$\Pe$ equivalence in $K \to \infty$] \label{prop:OWPeinfK}
    A goal-conditioned stationary policy $\pi_{\{.,-\}}$ is optimal for the problem formulation $\Pe(\gamma)$ if and only if it is also optimal for the setting $\OW(\infty, \gamma)$.
\end{myprop}
\begin{exmode}
\textit{Proof:}
Consider the $\Pe(\gamma)$ setting with a goal-conditioned stationary policy $\pi_{\{.,-\}}$; then we have
\begin{linenomath*} \begin{equation} \begin{aligned} 
    &\frac{1}{1-\gamma} J_\Pe(s, g, \pi_{\{.,-\}}, \gamma) 
    = 
    \EE^{\pi_{\{g,-\}}} \left[ \sum_{t=1}^\infty \gamma^{t-1} \Iin_{S_t = g} | S_0 = s \right]\\
    &= 
    \EE^{\pi_{\{g,-\}}} \left[ \EE^{\pi_{\{g,-\}}} \left[  \sum_{t=1}^\infty \gamma^{t-1} \Iin_{S_t = g} \Big| T_g, S_0 = s \right] \Big| S_0 = s \right]\\
    &= 
    \sum_{\tau=1}^\infty 
    \left(
    p^{\pi_{\{g,-\}}}(T_g = \tau | S_0 = s) \EE^{\pi_{\{g,-\}}} \left[ \gamma^{\tau-1} + \sum_{t=\tau + 1}^\infty \gamma^{t-1} \Iin_{S_t = g} \Big| S_\tau = g \right] 
    \right)
    \\
    &= 
    \underbrace{
    \left( 
    \sum_{\tau=1}^\infty p^{\pi_{\{g,-\}}}(T_g = \tau | S_0 = s) 
    \gamma^{\tau-1}
    \right)
    }_{= J_\OW(s, g, \pi_{\{.,-\}}, \infty, \gamma)}
     \cdot
    \underbrace{
    \left( 
    \EE^{\pi_{\{g,-\}}} \left[ 1 + \sum_{t= 1}^\infty \gamma^{t} \Iin_{S_t = g} \Big| S_0 = g \right]
    \right)
    }_{= 1 + \frac{\gamma}{1-\gamma}J_\Pe(g, g, \pi_{\{.,-\}}, \gamma) }
\end{aligned} \end{equation} \end{linenomath*}
which implies that
\begin{linenomath*} \begin{equation} \begin{aligned} 
    J_\Pe(s, g, \pi_{\{.,-\}}, \gamma) 
    &= 
    J_\OW(s, g, \pi_{\{.,-\}}, \infty, \gamma) \cdot
    \Big( 
    1 - \gamma + 
    \gamma J_\Pe(g, g, \pi_{\{.,-\}}, \gamma)  
    \Big).
\end{aligned} \end{equation} \end{linenomath*}
Solving this for the case of $s = g$ and after a few lines of algebra, we have
\begin{linenomath*} \begin{equation} \begin{aligned} \label{eq:proof:peowinf}
    J_\Pe(s, g, \pi_{\{.,-\}}, \gamma) 
    &= 
    \frac
    {(1-\gamma) J_\OW(s, g, \pi_{\{.,-\}}, \infty, \gamma)}
    {1 - \gamma J_\OW(g, g, \pi_{\{.,-\}}, \infty, \gamma)}.
\end{aligned} \end{equation} \end{linenomath*}

The optimal policy for $\OW(\infty, \gamma)$ maximizes $J_\OW(s, g, \pi_{\{.,-\}}, \infty, \gamma)$ and $J_\OW(g, g, \pi_{\{.,-\}}, \infty, \gamma)$ for all $s,g \in \Ss$; hence, according to \autoref{eq:proof:peowinf}, it also maximizes $J_\Pe(s, g, \pi_{\{.,-\}}, \gamma)$ for all $s,g \in \Ss$.
Therefore, any optimal policy for $\OW(\infty, \gamma)$ is also an optimal policy for $\Pe(\gamma)$.

At the same time, the optimal policy for $\Pe(\gamma)$ maximizes $J_\Pe(s, g, \pi_{\{.,-\}}, \gamma)$ and $J_\Pe(g, g, \pi_{\{.,-\}}, \gamma)$ for all $s,g \in \Ss$; hence, according to \autoref{eq:proof:peowinf} and its variant for $s = g$, it also maximizes $J_\OW(s, g, \pi_{\{.,-\}}, \infty, \gamma)$ for all $s,g \in \Ss$.
Therefore, any optimal policy for $\Pe(\gamma)$ is also an optimal policy for $\OW(\infty, \gamma)$.

Hence, the proof is complete.
$\hfill \square$
\end{exmode}

\begin{myprop}[$\OW$-$\Pe$ equivalence in deterministic environments] \label{prop:OWPeDeterm}
    Suppose that the environment is deterministic, i.e., for every $s \in \Ss$, there exists a transition function $h(s,.): \Aa(s) \to \Ss$ such that $p_h(s'|s,a) = \Iin_{s' = h(s,a)}$ for every $a \in \Aa(s)$.
    Suppose $\pi^*_{\{.,-\}}$ is the deterministic shortest-path policy induced by $h$, i.e., $\pi^*_{\{g,-\}}$ minimizes the time to reach $g$ from $s$.
    Then, $\pi^*_{\{.,-\}}$ is optimal for $\Pe(\gamma)$ and $\OW(K, \gamma)$, for all $\gamma \in (0,1)$ and $K \in \NN$.
\end{myprop}
\begin{exmode}
\textit{Proof:}
Given a deterministic environment $p_h$ and a deterministic policy $\pi_{\{.,-\}}$, we use $T^{\pi_{\{.,-\}}}(s,g)$ to denote the time needed to move from $s$ to $g$.
Given this setup, we have
\begin{linenomath*} \begin{equation} \begin{aligned} 
    J_\OW(s, g, \pi_{\{.,-\}}, K, \gamma) = 
    \gamma^{T^{\pi_{\{.,-\}}}(s,g)-1} \Iin_{T^{\pi_{\{.,-\}}}(s,g) \leq K}
\end{aligned} \end{equation} \end{linenomath*}
and
\begin{linenomath*} \begin{equation} \begin{aligned} 
    \frac{1}{1 - \gamma} J_\Pe(s, g, \pi_{\{.,-\}}, \gamma) = 
    \frac{\gamma^{T^{\pi_{\{.,-\}}}(s,g)-1}}{1 - \gamma^{T^{\pi_{\{.,-\}}}(g,g)}}.
\end{aligned} \end{equation} \end{linenomath*}
The transition time $T^{\pi_{\{.,-\}}}(s,g)$ takes its minimum value under the minimum-distance policy $\pi^*_{\{.,-\}}$.
As a result, both $J_\OW(s, g, \pi_{\{.,-\}}, K, \gamma)$ and $J_\Pe(s, g, \pi_{\{.,-\}}, \gamma)$ take their maximum value under the minimum-distance policy $\pi^*_{\{.,-\}}$.
Hence, the proof is complete.
$\hfill \square$
\end{exmode}

\begin{myprop}[$\OW$-$\ET$ equivalence with waiting actions and $\gamma = 1$] \label{prop:OWETwaiting}
    Suppose there is a waiting action at every state, i.e., for every $s \in \Ss$, there exists an action $a_{\rm wait} \in \Aa(s)$ such that $p(s'|s,a_{\rm wait}) = \Iin_{s' = s}$.
    Then, for any $K \in \NN$, the optimal values of $\OW(K,1)$ and $\ET(K)$ are equal:
    $J_\OW^*(s, g, K, 1) = J_\ET^*(s, g, K)$.
    Moreover, there exists a goal-conditioned policy that is optimal for both formulations.
\end{myprop}
\begin{exmode}
\textit{Proof:}
For any policy $\pi_{\{.,.\}}$, we have
\begin{linenomath*} \begin{equation} \begin{aligned}
    J_\ET(s, g, \pi_{\{.,.\}}, K)
    &=
    p^{\pi_{\{g,.\}}}(S_K = g | S_0 = s)
    \\
    &\leq
    p^{\pi_{\{g,.\}}}(T_g \leq K | S_0 = s)
    =
    J_\OW(s, g, \pi_{\{.,.\}}, K, 1),
\end{aligned} \end{equation} \end{linenomath*}
since the event $\{S_K = g\}$ implies $\{T_g \leq K\}$.
Taking the maximum over $\pi_{\{.,.\}}$ on both sides, we have
\begin{linenomath*} \begin{equation} \begin{aligned}
    \label{eq:proof:OWET:ineq1}
    J_\ET^*(s, g, K)
    &\leq
    J_\OW^*(s, g, K, 1).
\end{aligned} \end{equation} \end{linenomath*}

Now, given any policy $\pi_{\{.,.\}}$, define a modified policy $\tilde{\pi}_{\{.,.\}}$ by
\begin{linenomath*} \begin{equation} \begin{aligned}
    \label{eq:proof:OWET:pitilde}
    \tilde{\pi}_{\{g,t\}}(a|s) :=
    \begin{cases}
        \Iin_{a = a_{\rm wait}}, & s = g \text{ and } t>0,\\
        \pi_{\{g,t\}}(a|s), & s \neq g \text{ or } t=0.
    \end{cases}
\end{aligned} \end{equation} \end{linenomath*}
That is, $\tilde{\pi}_{\{g,.\}}$ behaves as $\pi_{\{g,.\}}$ until reaching $g$, and then waits at $g$ forever.
Therefore, for every $s,g \in \Ss$,
\begin{linenomath*} \begin{equation} \begin{aligned}
    p^{\tilde{\pi}_{\{g,.\}}}(S_K = g | S_0 = s)
    =
    p^{\pi_{\{g,.\}}}(T_g \leq K | S_0 = s),
\end{aligned} \end{equation} \end{linenomath*}
which implies
\begin{linenomath*} \begin{equation} \begin{aligned}
    J_\ET(s, g, \tilde{\pi}_{\{.,.\}}, K)
    =
    J_\OW(s, g, \pi_{\{.,.\}}, K, 1).
\end{aligned} \end{equation} \end{linenomath*}

Now, suppose $\pi^*_{\{.,.\}}$ is an optimal policy with respect to $\OW(K,1)$, and $\tilde{\pi}_{\{g,.\}}^*$ is its modified version using \autoref{eq:proof:OWET:pitilde}.
Then we have
\begin{linenomath*} \begin{equation} \begin{aligned}
    J_\ET^*(s, g, K)
    \geq
    J_\ET(s, g, \tilde{\pi}_{\{.,.\}}^*, K)
    =
    J_\OW(s, g, \pi_{\{.,.\}}^*, K, 1) 
    = 
    J_\OW^*(s, g, K, 1).
\end{aligned} \end{equation} \end{linenomath*}
Combining this with \autoref{eq:proof:OWET:ineq1}, we have
\begin{linenomath*} \begin{equation} \begin{aligned}
    J_\ET^*(s, g, K)
    = 
    J_\OW^*(s, g, K, 1).
\end{aligned} \end{equation} \end{linenomath*}
The policy $\tilde{\pi}_{\{g,.\}}^*$ constructed from any optimal policy $\pi_{\{g,.\}}^*$ of $\OW(K,1)$ (using \autoref{eq:proof:OWET:pitilde}) is optimal for both settings.
$\hfill \square$
\end{exmode} 
\newpage
\section{Additional statements for sensitivity-control relationships in \autoref{sec:control}} \label{append:props_snes}
\begin{myprop}[One-step controllability as the number of effectively distinct actions;      paraphrased version of Proposition 1 of \cite{modirshanechi2025integrative}]
    \label{prop:one_step_control}
    Consider the one-step problem formulation in \autoref{prop:one_step_case}, i.e., $\Pe(0)$, $\ET(1)$, $\OW(1,\gamma)$, and $\OW(K,0)$.
    Then, the objective controllability at state $s$, as in \autoref{eq:optimal_sensitivity}, is given by 
    \begin{linenomath*} \begin{equation} \begin{aligned}
        \GSen^*(s) = \frac{1}{N_s} \sum_{a \in \Aa(s)} \Delta(a;s),
    \end{aligned} \end{equation} \end{linenomath*}
    where $\Delta(a;s)$ measures how distinct $a \in \Aa(s)$ is from the other actions in $\Aa(s)$: 
    \begin{linenomath*} \begin{equation} \begin{aligned}
        \Delta(a;s) := \frac{1}{N_a(s)} \sum_{a' \in \Aa(s)} \left[ \sum_{s' \in \Ss^*(a;s)} |p(s'|s,a) - p(s'|s,a')| \right] \leq 1 -  \frac{1}{N_a(s)} 
    \end{aligned} \end{equation} \end{linenomath*}
    with $\Ss^*(a;s) \subseteq \Ss$ the set of states where $a = \arg \max_{\Tilde{a}} p(s'|s,\Tilde{a})$; 
    when there are ties, we assume that $\arg \max$ returns one action based on a given ordering.
\end{myprop}
\begin{exmode}
\textit{Proof:}
The proof is a straightforward adaptation of Proposition 1 of~\cite{modirshanechi2025integrative}.
$\hfill \square$
\end{exmode}

\textbf{Klyubin empowerment.}
Following~\cite{klyubin2005empowerment, salge2014empowerment}, we define (a generalized version of) $K$-step empowerment as 
\begin{linenomath*} \begin{equation} \begin{aligned}
    \EmpKly(s; K) := \max_{\pi^{\rm NM}}  I^{\pi_{\{-,.\}}^{\rm NM}} \left( A\tS{0:K-1}, S\tS{K} | S\tS{0} = s \right),
\end{aligned} \end{equation} \end{linenomath*}
where $\pi_{\{-,t\}}^{\rm NM}(a | \tau_t )$ is a goal-independent, non-stationary, and potentially \textit{non-Markovian} policy.
$\EmpKly(s; K)$ can be intuitively interpreted as the number of states that are surely reachable in $K$ steps; this corresponds to a particular conceptualization of controllability~\citep{modirshanechi2025integrative}.
While other variants have been proposed~\citep{abel2025plasticity, capdepuy2011informational}, $\EmpKly(s; K)$ is the most relevant choice for the RL literature~\citep{jung2011empowerment, leibfried2019unified, bharadhwaj2022information}.

\vspace{0.2cm}
\begin{myprop}[Monotone equivalence of $\GSen_\ET^*$ and $\EmpKly$ in deterministic environments]
    \label{prop:ETKlyEmp}
    In deterministic environments, i.e., for every $s \in \Ss$, there exists a transition function $h(s,.): \Aa(s) \to \Ss$ such that $p_h(s'|s,a) = \Iin_{s' = h(s,a)}$ for every $a \in \Aa(s)$.
    \begin{linenomath*} \begin{equation} \begin{aligned}
    \EmpKly(s;K) = \log \left( 1 + N_s \, \GSen_\ET^*(s, K) \right),
    \end{aligned} \end{equation} \end{linenomath*}
    where $\GSen_\ET^*(s, K)$ is the maximal goal sensitivity for $\ET(K)$.
\end{myprop}
\begin{exmode}
\textit{Proof:}
Since the environment is deterministic, there exists a function $h_K$, defined based on $h$, such that
\begin{equation} \begin{aligned}
    p(S_K = s'|S_0 = s, a_{0:K-1}) = \Iin_{s' = h_K(s,a_{0:K-1})}.
\end{aligned} \end{equation}
As a result, and using \autoref{eq:ET_J}, we have
\begin{linenomath*} \begin{equation} \begin{aligned} 
    J_\ET^*(s, g, K) = 
    \begin{cases}
        1 & \text{if } \,\, \exists \,\, a_{0:K-1} \text{  s.t.  } h_K(s,a_{0:K-1}) = g \\
        0 & \text{otherwise}
    \end{cases}
\end{aligned} \end{equation} \end{linenomath*}
implying that
\begin{linenomath*} \begin{equation} \begin{aligned} 
    J_\ET^*(s, K) = 
    \frac{\left| \{ h_K(s,a_{0:K-1}) \}_{a_{0:K-1}} \right|}{N_s}
\end{aligned} \end{equation} \end{linenomath*}
and, using \autoref{theor:decomp},
\begin{linenomath*} \begin{equation} \begin{aligned} 
    \left| \{ h_K(s,a_{0:K-1}) \}_{a_{0:K-1}} \right| = 
    N_s \GSen_\ET^*(s, K) + 1
\end{aligned} \end{equation} \end{linenomath*}
At the same time, we have
\begin{equation} \begin{aligned}
    I^{\pi_{\{-,.\}}^{\rm NM}} \big[S\tS{K}, A\tS{0:K-1} \big| S\tS{0} = s\big] &=\\  
    H^{\pi_{\{-,.\}}^{\rm NM}}&\big[h_K(s,A_{0:K-1})\big] 
    - \underbrace{H^{\pi_{\{-,.\}}^{\rm NM}}\big[h_K(s,A_{0:K-1}) \big| A_{0:K-1}\big]}_{= 0}.
\end{aligned} \end{equation}
As a result,
\begin{equation} \begin{aligned}
    \EmpKly(s; K)
    = \max_{\pi^{\rm NM}} H^{\pi_{\{-,.\}}^{\rm NM}} \big[h_K(s,A_{0:K-1})\big] = \log \left| \{ h_K(s,a_{0:K-1}) \}_{a_{0:K-1}} \right|.
\end{aligned} \end{equation}
Hence, the proof is complete.
$\hfill \square$
\end{exmode}
 
\newpage
\section{Additional statements for MI-empowerment relationships in \autoref{sec:empowerment}} \label{append:props_emp}

\begin{figure}[H]
    \centering 
    \includegraphics[width=1\textwidth]{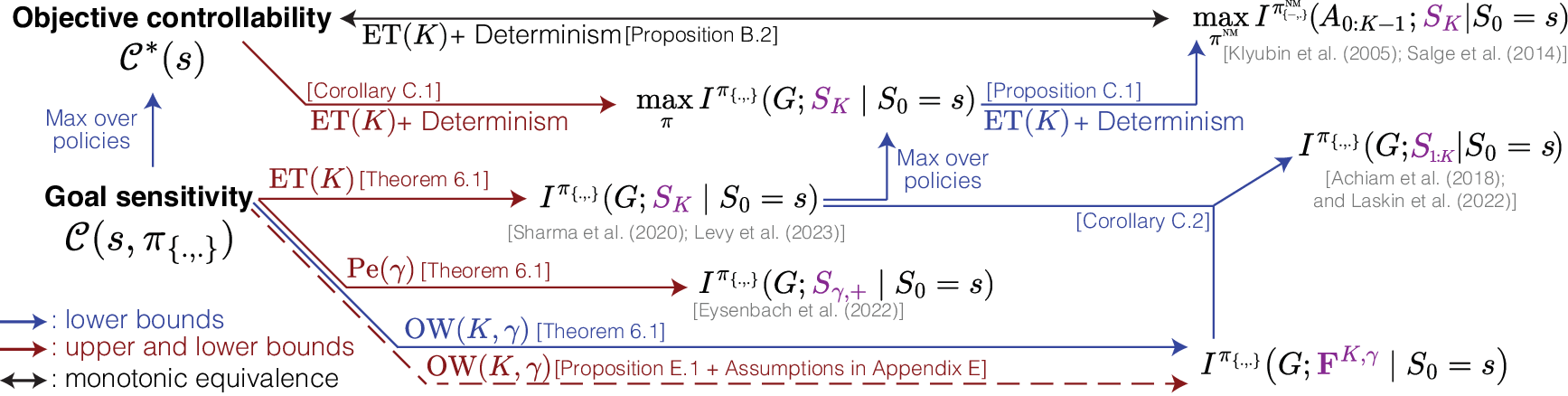}
    \caption{
    Theoretical bounds linking goal sensitivity to empowerment and goal-behavior MIs.
    }
    \label{fig:GSenEmpApp}
\end{figure}

\begin{myprop}[$G$-$S_K$ MI is upper-bounded by Klyubin empowerment in deterministic environments]
    \label{prop:KlyIG}
    Consider the problem formulation $\ET(K)$ in a deterministic environment, i.e., for every $s \in \Ss$, there exists a transition function $h(s,.): \Aa(s) \to \Ss$ such that $p_h(s'|s,a) = \Iin_{s' = h(s,a)}$ for every $a \in \Aa(s)$. Then, for any initial state $s \in \Ss$,
    \begin{linenomath*} \begin{equation} \begin{aligned}
        \max_{\pi_{\{.,.\}}}
        I^{\pi_{\{.,.\}}}\!\left(G;S_K \mid S_0=s\right)
        \le
        \EmpKly(s;K).
    \end{aligned} \end{equation} \end{linenomath*}
    Moreover, there exist an environment $p(.|.,.)$ and a horizon $K \in \NN$ such that the inequality is strict.
\end{myprop}
\begin{exmode}
\textit{Proof:}
Fix any goal-conditioned policy $\pi_{\{.,.\}}$, and define a non-Markovian, goal-independent policy $\tilde\pi_{\{-,.\}}^{\rm NM}$ as follows:
sample $\tilde G \sim {\rm Unif}(\Ss)$ at $t=0$, and then follow the branch $\pi_{\{\tilde G,.\}}$.
Under $\tilde\pi_{\{-,.\}}^{\rm NM}$, the joint distribution of $(\tilde G,A_{0:K-1},S_K)$ is the same as that of $(G,A_{0:K-1},S_K)$ under $\pi_{\{.,.\}}$.
The determinism of the environment implies that $S_K \independent \Tilde{G} | \{ A_{0:K-1}, S_0 \}$; hence, by the data-processing inequality,
\begin{linenomath*} \begin{equation} \begin{aligned}
    I^{\pi_{\{.,.\}}}(G;S_K \mid S_0=s)
    &=
    I^{\tilde\pi_{\{-,.\}}^{\rm NM}}(\tilde G;S_K \mid S_0=s)\\
    &\le
    I^{\tilde\pi_{\{-,.\}}^{\rm NM}}(A_{0:K-1};S_K \mid S_0=s)
    \le
    \EmpKly(s;K).
\end{aligned} \end{equation} \end{linenomath*}
Taking the maximum over $\pi_{\{.,.\}}$ proves the inequality.

To show strictness, consider the deterministic one-step environment with $\Ss=\{s,g_1,g_2\}$ and $\Aa(s)=\{a_1,a_2\}$, and the transition dynamics given by $p(g_1|s,a_1)=1$ and $p(g_2|s,a_2)=1$, with $g_1$ and $g_2$ absorbing states.

For $K=1$, we have $\EmpKly(s;1)=\log 2$, since the agent can choose between two perfectly distinguishable successor states.
On the other hand, under any goal-conditioned policy with $G \sim {\rm Unif}(\Ss)$, the variable $S_1$ can take only the two values $g_1$ and $g_2$.
The mutual information $I(G;S_1 \mid S_0=s)$ is maximized by deterministically mapping two goals to one action and the remaining goal to the other, which yields
\begin{linenomath*} \begin{equation} \begin{aligned}
    \max_{\pi_{\{.,.\}}} I^{\pi_{\{.,.\}}}(G;S_1 \mid S_0=s)=h(1/3)<\log 2.
\end{aligned} \end{equation} \end{linenomath*}
Therefore, the inequality can be strict.
$\hfill \square$
\end{exmode}

\begin{mycorrol}[$\GSen_\ET^*$ and $G$-$S_K$ MI in deterministic environments]
    \label{corr:IGET_deterministic}
    Suppose that the environment is deterministic, i.e., for every $s \in \Ss$, there exists a transition function $h(s,.): \Aa(s) \to \Ss$ such that $p_h(s'|s,a) = \Iin_{s' = h(s,a)}$ for every $a \in \Aa(s)$.
    Then, we have
    \begin{linenomath*} \begin{equation} \begin{aligned}
        \Phi^{\rm down}_{N_s}\!\Big(
        N_s^{-1} + 
        \GSen_\ET^*(s,K)
        \Big)
        \leq
        \max_{\pi_{\{.,.\}}}
        I^{\pi_{\{.,.\}}}\!\left(G;S_K \mid S_0=s\right)
        \leq
        \log \Big( 1 + N_s \, \GSen_\ET^*(s, K) \Big).
    \end{aligned} \end{equation} \end{linenomath*}
    where $\GSen_\ET^*(s, K)$ is the controllability for $\ET(K)$ (\autoref{eq:optimal_sensitivity}), and $\Phi^{\rm down}_{N_s}$ is defined in \autoref{theor:MI_empowerment}.
\end{mycorrol}
\begin{exmode}
\textit{Proof:}
It is a direct consequence of \autoref{prop:ETKlyEmp}, \autoref{theor:MI_empowerment}, and \autoref{prop:KlyIG}.
$\hfill \square$
\end{exmode}

\begin{mycorrol}[The data-processing bound on trajectory-based MI]
    \label{corr:trajstate}
    Consider a goal-conditioned policy $\pi_{\{.,.\}}$.
    Then
    \begin{linenomath*} \begin{equation} \begin{aligned}
        I^{\pi_{\{.,.\}}}(G;S_K \mid S_0=s) 
        &\leq I^{\pi_{\{.,.\}}}(G;S_{1:K} \mid S_0=s)
        \leq I^{\pi_{\{.,.\}}}(G;\tau_{K} \mid S_0=s)\\
        I^{\pi_{\{.,.\}}}(G;\mathbf{F}^{K,\gamma} \mid S_0=s) &\leq I^{\pi_{\{.,.\}}}(G;S_{1:K} \mid S_0=s)
        \leq I^{\pi_{\{.,.\}}}(G;\tau_{K} \mid S_0=s)
    \end{aligned} \end{equation} \end{linenomath*}
\end{mycorrol}
\begin{exmode}
\textit{Proof:}
The statement follows directly from the data processing inequality.
$\hfill \square$
\end{exmode}

\newpage
\section{Proofs of the formal statements in the main text} 
\label{append:proofs}

\subsection{Proof of \autoref{prop:settings_contradicting}}
Consider the environment $p(.|.,.)$ in \autoref{fig:envA}A for $\epsilon_1 < \epsilon_2 \in [0,1]$.
Suppose $K = 2$.
Below, we show that there exists a range of values of $\gamma \in [0,1]$ where the optimal path for reaching $g$ from $S_0 = s_1$ is different for different formulations.

For example, if $\gamma > \sqrt{\epsilon_1}$ and $\gamma > \epsilon_2$, then starting from $S_0 = s_1$, the optimal policy $\pi^*_{g,t}(a|s)$ for $\Pe(\gamma)$ is given by following the path 
\begin{linenomath*} \begin{equation} \begin{aligned}
    s_1 
    \overset{a_{\rm f}}{\longrightarrow}
    s_2 
    \overset{a_{\rm f}}{\longrightarrow}
    s_3
    \overset{a_{\rm f}}{\longrightarrow}
    g
\end{aligned} \end{equation} \end{linenomath*}
Meanwhile, if $\epsilon_1 < \epsilon_2$, then the optimal policy $\pi^*_{g,t}(a|s)$ for $\ET(2)$ is given by following the path 
\begin{linenomath*} \begin{equation} \begin{aligned}
    s_1 
    \overset{a_{\rm f}}{\longrightarrow}
    s_2 
    \overset{a_{\rm j}}{\longrightarrow}
    \begin{cases}
    g & \text{with probability } \epsilon_2\\
    T & \text{with probability } 1 - \epsilon_2
    \end{cases}
\end{aligned} \end{equation} \end{linenomath*}
Finally, if $\gamma \epsilon_2 < \epsilon_1$, then the optimal policy $\pi^*_{g,t}(a|s)$ for $\OW(2,\gamma)$ is given by following the path 
\begin{linenomath*} \begin{equation} \begin{aligned}
    s_1 
    \overset{a_{\rm j}}{\longrightarrow}
    \begin{cases}
    g & \text{with probability } \epsilon_1\\
    T & \text{with probability } 1 - \epsilon_1.
    \end{cases}
\end{aligned} \end{equation} \end{linenomath*}
Hence, if we have
\begin{linenomath*} \begin{equation} \begin{aligned}
    \max \{ \sqrt{\epsilon_1}, \epsilon_2 \}
    < 
    \gamma
    <
    \frac{\epsilon_1}{\epsilon_2}
    <
    1,
\end{aligned} \end{equation} \end{linenomath*}
then different formulations yield different optimal actions (\autoref{fig:envA}B), and the proof is complete.
$\hfill \square$

\begin{figure}[H]
    \centering
    \includegraphics[width=\textwidth]{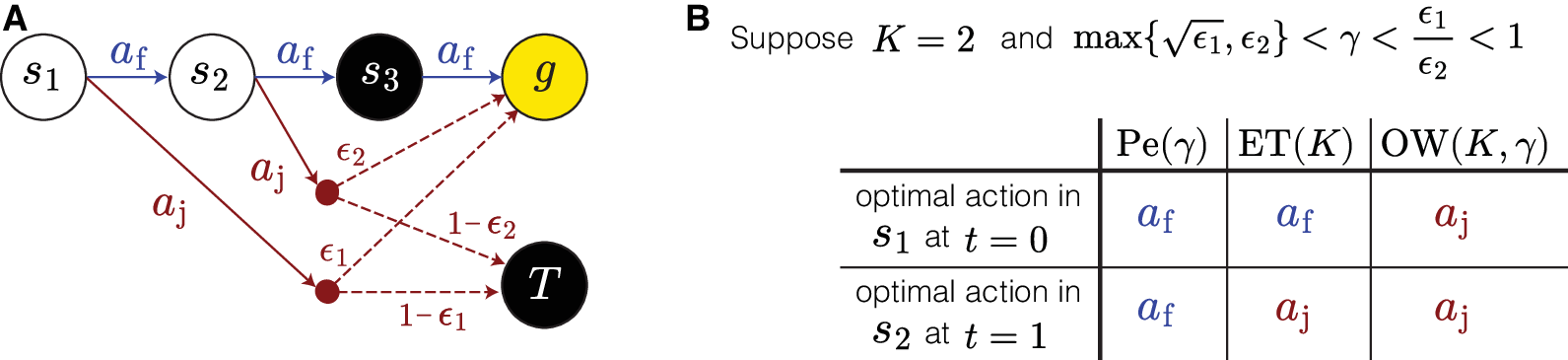}
    \caption{
    Counterexample environment (A) showing that different GCRL formulations can induce different optimal policies (B).
    In (A), solid arrows denote actions and dashed arrows stochastic transitions with probabilities $\epsilon_1$, $\epsilon_2$, $1-\epsilon_1$, and $1 - \epsilon_2$.
    In (B), the optimal policy is conditioned on $g$, in (A), as the goal.
    The example in \autoref{fig:formulations} is a special case of panel A.
    }
    \label{fig:envA}
\end{figure}

\subsection{Proof of \autoref{theor:decomp}}

As a direct consequence of \autoref{eq:general_expvalue}, we have
\begin{linenomath*} \begin{equation} \begin{aligned} \label{eq:decom_proof}
    J(s, \pi_{\{.,.\}}) - 
    J(s, \pi_{\{g',.\}}) = 
    \frac{1}{N_s} \sum_{g \in \Ss} 
    \Big(J(s, g, \pi_{\{g,.\}}) - J(s, g, \pi_{\{g',.\}})\Big)
\end{aligned} \end{equation} \end{linenomath*}
for an arbitrary goal $g' \in \Ss$.
By averaging \autoref{eq:decom_proof} over $g'$, we have
\begin{linenomath*} \begin{equation} \begin{aligned} \label{eq:decomJ}
    J(s, \pi_{\{.,.\}}) = \underbrace{\frac{1}{N_s} \sum_{g' \in \Ss} J(s, \pi_{\{g',.\}})}_{\textnormal{goal-agnostic value of $\pi$}} + 
    \underbrace{\frac{1}{N_s^2} \sum_{g, g' \in \Ss} 
    \Big ( J(s, g, \pi_{\{g,.\}}) - J(s, g, \pi_{\{g',.\}}) \Big)}_{= \GSen(s, \pi_{\{.,.\}}) \textnormal{; see \autoref{def:control}}}.
\end{aligned} \end{equation} \end{linenomath*}

\textit{1. For the case of $\Pe(\gamma)$ and $\ET(K)$:}
A direct consequence of \autoref{eq:Pe_J} and \autoref{eq:ET_J}, respectively, is that 
\begin{linenomath*} \begin{equation} \begin{aligned}
    J(s, \pi_{\{g',.\}}) = \frac{1}{N_s} \underbrace{\sum_{g} J(s, g, \pi_{\{g',.\}})}_{= 1} = \frac{1}{N_s}.
\end{aligned} \end{equation} \end{linenomath*}
As a result, we can rewrite \autoref{eq:decomJ} as
\begin{linenomath*} \begin{equation} \begin{aligned}
    J(s, \pi_{\{.,.\}}) &= \frac{1}{N_s} + \GSen(s, \pi_{\{.,.\}}).
\end{aligned} \end{equation} \end{linenomath*}
Hence, the proof of the first part is complete.

\textit{2. $J$-$\GSen$ relationship for $\OW(K, \gamma)$ and formulations with non-negative rewards:}
Using \autoref{eq:decomJ}, we have
\begin{linenomath*} \begin{equation} \begin{aligned}
    \GSen(s,\pi_{\{.,.\}})
    =
    J(s,\pi_{\{.,.\}})
    -
    \frac{1}{N_s}\sum_{g' \in \Ss} J(s,\pi_{\{g',.\}}).
\end{aligned} \end{equation} \end{linenomath*}
Moreover, the non-negative reward assumption implies that, for every $g \in \Ss$ and $g' \in \Ss$,
\[
J(s,g,\pi_{\{g',.\}}) \geq 0.
\]
Hence,
\begin{linenomath*} \begin{equation} \begin{aligned}
    \frac{1}{N_s}\sum_{g' \in \Ss} J(s,\pi_{\{g',.\}})
    &=
    \frac{1}{N_s^2}\sum_{g,g' \in \Ss} J(s,g,\pi_{\{g',.\}})
    \\
    &\geq
    \frac{1}{N_s^2}\sum_{g \in \Ss} J(s,g,\pi_{\{g,.\}})
    =
    \frac{1}{N_s} J(s,\pi_{\{.,.\}}).
\end{aligned} \end{equation} \end{linenomath*}
Since all the terms are non-negative, the inequality is tight if and only if $J(s,g,\pi_{\{g',.\}})=0$ for all $g \neq g'$.
Therefore,
\begin{linenomath*} \begin{equation} \begin{aligned}
    \GSen(s,\pi_{\{.,.\}})
    &\leq
    J(s,\pi_{\{.,.\}}) - \frac{1}{N_s}J(s,\pi_{\{.,.\}})
    =
    \frac{N_s-1}{N_s}J(s,\pi_{\{.,.\}}),
\end{aligned} \end{equation} \end{linenomath*}
which is equivalent to
\begin{linenomath*} \begin{equation} \begin{aligned}
    J(s,\pi_{\{.,.\}})
    \geq
    \frac{N_s}{N_s-1}\GSen(s,\pi_{\{.,.\}}).
\end{aligned} \end{equation} \end{linenomath*}
Hence, the proof of the second part is complete.

\textit{3. The maximally in-control policy for $\OW(K, \gamma)$:} 
Consider the environment $p(.|.,.)$ in \autoref{fig:envB}A, $K=2$, and $\gamma=1$.
The goal-conditioned policy in \autoref{fig:envB}B is optimal with respect to $J_\OW$, but the suboptimal policy in \autoref{fig:envB}C has a higher goal sensitivity.
Hence, a maximally in-control policy is not necessarily optimal for $\OW(K, \gamma)$.

\begin{figure}[t!]
    \centering 
    \includegraphics[width=1\textwidth]{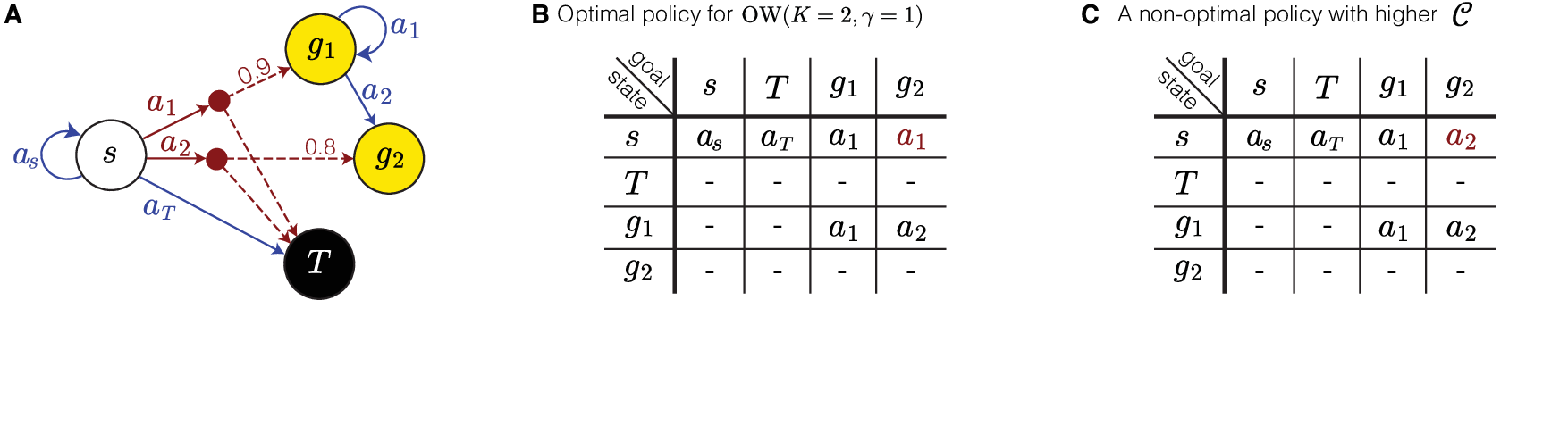}
    \caption{Counterexample showing that maximizing goal sensitivity $\GSen$ does not necessarily result in maximizing value $J$ for the $\OW(K,\gamma)$ formulation; see \autoref{theor:decomp}.
    }
    \label{fig:envB}
\end{figure}

For the bound, let $\pi^{\GSen*}_{\{.,.\}}$ be any maximally in-control policy. 
By part 2,
\begin{linenomath*} \begin{equation} \begin{aligned}
    J(s,\pi^{\GSen*}_{\{.,.\}})
    \geq
    \frac{N_s}{N_s-1}\GSen^*(s).
\end{aligned} \end{equation} \end{linenomath*}
Since $J^*(s)\leq 1$ for $\OW(K,\gamma)$, it follows that
\begin{linenomath*} \begin{equation} \begin{aligned}
    0 \le
    J^*(s)-J(s,\pi^{\GSen*}_{\{.,.\}})
    \le
    1-\frac{N_s}{N_s-1}\GSen^*(s).
\end{aligned} \end{equation} \end{linenomath*}
If $\GSen^*(s)=1-1/N_s$, the upper bound is zero, hence every maximally in-control policy is optimal.
$\hfill \square$

\subsection{Proof of \autoref{theor:MI_empowerment}}

To prove \autoref{theor:MI_empowerment}, we will use the following lemma to show that for consistent policies, $S_{\gamma,+}$ and $S_K$ are the best estimates of the goal $G$ for $\Pe(\gamma)$ and $\ET(K)$, respectively.

\begin{mylemma}[Optimality of estimating the commanded goal by reached states in $\Pe$ and $\ET$] \label{lemma:optimaldecoder}
    Consider problem formulations $\Pe(\gamma)$ and $\ET(K)$ together with a goal-conditioned policy $\pi_{\{.,.\}}$.
    Let $S' := S_{\gamma,+}$ for $\Pe(\gamma)$ and $S' := S_{K}$ for $\ET(K)$, respectively.
    Suppose our aim is to estimate $G$ based on $S'$ and consider the naïve identity estimator $\hat{G}(s') = s'$
    as well as the Bayes-optimal estimator
    \begin{linenomath*} \begin{equation} \begin{aligned}
        \hat{G}^*(s') &= \arg \max_{g} p^{\pi_{\{.,.\}}}(G = g | S' = s', S_0=s).
    \end{aligned} \end{equation} \end{linenomath*}
    If the policy $\pi_{\{.,.\}}$ is consistent (\autoref{def:consistency}), then $\hat{G}$ and $\hat{G}^*$ have the same error probability, i.e.,
    \begin{linenomath*} \begin{equation} \begin{aligned}
    p_e
    :=
    p^{\pi_{\{.,.\}}}(\hat{G}(S') \neq G \mid S_0=s)
    =
    p^{\pi_{\{.,.\}}}(\hat{G}^*(S') \neq G \mid S_0=s)
    =:
    p^*_e.
\end{aligned} \end{equation} \end{linenomath*}
\end{mylemma}
\begin{exmode}
\textit{Proof of Lemma:}
The error probability for the naïve identity estimator is given by
\begin{linenomath*} \begin{equation} \begin{aligned}
    p_e
    :&=
    p^{\pi_{\{.,.\}}}(\hat{G}(S') \neq G \mid S_0=s)\\
    &=
    1 - p^{\pi_{\{.,.\}}}(S' = G \mid S_0=s)
    \overset{\text{\autoref{eq:Pe_J} and \autoref{eq:ET_J}}}{=} 
    1 - J(s , \pi_{\{.,.\}}).
\end{aligned} \end{equation} \end{linenomath*}
where $J$ corresponds to either $\Pe(\gamma)$ or $\ET(K)$, depending on whether $S' = S_{\gamma,+}$ or $S' = S_{K}$, respectively.

The analogous error probability for the Bayes-optimal estimator is given by
\begin{linenomath*} \begin{equation} \begin{aligned} \label{eq:proof:peopt}
    p^*_e
    :&=
    p^{\pi_{\{.,.\}}}(\hat{G}^*(S') \neq G \mid S_0=s)\\
    &= 
    1 - p^{\pi_{\{.,.\}}}(\hat{G}^*(S') = G \mid S_0=s)\\
    &=
    1 - \sum_{g' \in \Ss} 
    p^{\pi_{\{.,.\}}}(S' = g' | S_0=s) 
    \underbrace{
    p^{\pi_{\{.,.\}}}(\hat{G}^*(S') = G | S' = g' , S_0=s)
    }_{= \max_{g} p^{\pi_{\{.,.\}}}(G = g | S' = g', S_0=s)}.
\end{aligned} \end{equation} \end{linenomath*}
Additionally, using the Bayes rule, we have
\begin{linenomath*} \begin{equation} \begin{aligned} \label{eq:proof:peopt2}
    \max_{g} p^{\pi_{\{.,.\}}}(G = g | S' = g', S_0=s) =
    \max_{g} 
    \frac
    {p^{\pi_{\{.,.\}}}(S' = g'|G = g , S_0=s) p_\goal(g)}
    {p^{\pi_{\{.,.\}}}(S' = g'| S_0=s)} 
\end{aligned} \end{equation} \end{linenomath*}
which, combined with \autoref{eq:proof:peopt}, gives
\begin{linenomath*} \begin{equation} \begin{aligned} \label{eq:proof:peopt3}
    p^*_e &= 
    1 - \sum_{g' \in \Ss} 
    \max_{g} 
    \underbrace{p^{\pi_{\{.,.\}}}(S' = g'|G = g , S_0=s)}_{= J(s, g', \pi_{\{g,.\}})} 
    \underbrace{p_\goal(g)}_{= 1/N_s}\\
    &=
    1 - \frac{1}{N_s} \sum_{g' \in \Ss} \max_{g} J(s, g', \pi_{\{g,.\}}) \leq p_e.
\end{aligned} \end{equation} \end{linenomath*}
where $J$ corresponds to either $\Pe(\gamma)$ or $\ET(K)$, depending on whether $S' = S_{\gamma,+}$ or $S' = S_{K}$, respectively.
If the policy $\pi_{\{.,.\}}$ is consistent (\autoref{def:consistency}), then we have $\max_{g} J(s, g', \pi_{\{g,.\}}) = J(s, g', \pi_{\{g',.\}})$ in \autoref{eq:proof:peopt3} and, as a result,
\begin{linenomath*} \begin{equation} \begin{aligned}
    p^*_e = 1 - J(s , \pi_{\{.,.\}}) = p_e.
\end{aligned} \end{equation} \end{linenomath*}
Hence, the proof is complete.
$\hfill \square$
\end{exmode}

We then split the proof of  \autoref{theor:MI_empowerment} into two parts.

\textit{1-2. The case of $\Pe(\gamma)$ and $\ET(K)$:} 
We prove the two inequalities in parallel using
\begin{linenomath*} \begin{equation} \begin{aligned}
    I^{\pi_{\{.,.\}}}(G;S_{\gamma,+}\mid S_0=s)
    &=
    H(G)-H(G \mid S_{\gamma,+}, S_0=s)
    \\
    I^{\pi_{\{.,.\}}}(G;S_{K}\mid S_0=s)
    &=
    H(G)-H(G \mid S_{K}, S_0=s)
\end{aligned} \end{equation} \end{linenomath*}
respectively for $\Pe(\gamma)$ and $\ET(K)$.
Since $G \sim {\rm Unif}(\Ss)$, we have $H(G)=\log N_s$.
Hence, we need to find a bound for the conditional entropies. 

To simplify the proof, we define the random variable $S'$ as $S_{\gamma,+}$ for $\Pe(\gamma)$ and as $S_{K}$ for $\ET(K)$.
Then, from an information-theoretic perspective, we can view the commanded goal $G$ as a message to be transmitted, but the receiver receives only $S'$, sampled by running the policy $\pi_{\{G,.\}}$.
Given this communication problem, we define the naïve identity decoder $\hat{G}(s') = s'$ and the Bayes-optimal decoder $\hat{G}^*$ as in \autoref{lemma:optimaldecoder}.
According to \autoref{lemma:optimaldecoder}, for a general policy, we have
\begin{linenomath*} \begin{equation} \begin{aligned}
    p_e = 
    1 - J(s , \pi_{\{.,.\}})
    \overset{\text{\autoref{theor:decomp}}}{=} 
    1 - \frac{1}{N_s} - \GSen(s,\pi_{\{.,.\}}) \geq p^*_e.
\end{aligned} \end{equation} \end{linenomath*}
where $J$ and $\GSen$ correspond to either $\Pe(\gamma)$ or $\ET(K)$, depending on whether $S' = S_{\gamma,+}$ or $S' = S_{K}$, respectively.
Hence, using Fano's inequality applied to the naïve identity decoder $\hat{G}(s') = s'$, we have
\begin{linenomath*} \begin{equation} \begin{aligned}
    H(G \mid S', S_0=s)
    &\leq
h(p_e) + p_e \log(N_s-1).
\end{aligned} \end{equation} \end{linenomath*}
This proves the general lower bound:
\begin{linenomath*} \begin{equation} \begin{aligned}
    I^{\pi_{\{.,.\}}}(G;S'\mid S_0=s)
    &\geq
    \Phi^{\rm down}_{N_s}\!\Big(\GSen(s,\pi_{\{.,.\}}) + N_s^{-1}\Big).
\end{aligned} \end{equation} \end{linenomath*}

For the upper bound, we need the reverse-Fano's inequality, which is given only for the optimal decoder~\cite{rioul2023interplay, tebbe1968uncertainty}.
If the policy $\pi_{\{.,.\}}$ is consistent (\autoref{def:consistency}), then we can use \autoref{lemma:optimaldecoder} and, as a result,
\begin{linenomath*} \begin{equation} \begin{aligned} \label{eq:proof:consistentpe}
    \text{for consistent policies:} \quad p^*_e = p_e =  
    1 - \frac{1}{N_s} - \GSen(s,\pi_{\{.,.\}}).
\end{aligned} \end{equation} \end{linenomath*}
Using \autoref{eq:proof:consistentpe}, Theorem 5 and Example 14 of~\cite{rioul2023interplay} directly yield the upper bound, for consistent policies:
\begin{linenomath*} \begin{equation} \begin{aligned}
    I^{\pi_{\{.,.\}}}(G;S'\mid S_0=s)
    &\leq
    \Phi^{\rm up}_{N_s}\!\Big(\GSen(s,\pi_{\{.,.\}}) + N_s^{-1}\Big).
\end{aligned} \end{equation} \end{linenomath*}
Hence, the proof of the bounds in the first two parts is complete.

It is straightforward to confirm that $\Phi^{\rm down}_{N_s}(x)$ is strictly increasing on $x \in [1/N_s, 1]$, and it follows from~\cite{rioul2023interplay} that $\Phi^{\rm up}_{N_s}(x)$ is strictly increasing on $x \in [1/N_s, 1]$.
Given that $J(s , \pi_{\{.,.\}}) \in [0,1]$ for both $\Pe(\gamma)$ and $\ET(K)$, it follows from \autoref{theor:decomp} that $x = \GSen(s,\pi_{\{.,.\}}) + N_s^{-1} \leq 1$;
meanwhile, from consistency, it follows that $x = \GSen(s,\pi_{\{.,.\}}) + N_s^{-1} \geq 1/N_s$.
Hence, $\Phi^{\rm down}_{N_s}(x)$ and $\Phi^{\rm up}_{N_s}(x)$ are strictly increasing for all consistent policies.

Hence, the proof of the first two parts is complete.

\textit{3. The case of $\OW(K, \gamma)$:} 
For each $g \in \Ss$, let $P_{\mathbf{F}}^{\pi_{\{g,.\}}}$ denote the distribution of $\mathbf{F}^{K,\gamma}$ under the policy $\pi_{\{g,.\}}$, and let
\begin{linenomath*} \begin{equation} \begin{aligned}
    P_{\mathbf{F}}^{\bar\pi_{\{-,.\}}^{\rm NM}} := \frac{1}{N_s}\sum_{g \in \Ss} P_{\mathbf{F}}^{\pi_{\{g,.\}}}.
\end{aligned} \end{equation} \end{linenomath*}
By \autoref{eq:OW_J}, we have
\begin{linenomath*} \begin{equation} \begin{aligned}
    J_\OW(s,g,\pi_{\{g',.\}},K,\gamma)
    =
    \EE^{\pi_{\{g',.\}}}\!\left[F^{K,\gamma}_g \mid S_0=s\right].
\end{aligned} \end{equation} \end{linenomath*}
This, together with the definition of $\GSen_\OW$ in \autoref{def:control}, implies
\begin{linenomath*} \begin{equation} \begin{aligned} 
    \GSen_\OW(s,\pi_{\{.,.\}},K,\gamma) 
    &= 
    \frac{1}{N_s^2} \sum_{g, g' \in \Ss} 
    \Big ( 
    \EE_{P_{\mathbf{F}}^{\pi_{\{g,.\}}}}[F^{K,\gamma}_g] - 
    \EE_{P_{\mathbf{F}}^{\pi_{\{g',.\}}}}[F^{K,\gamma}_g]
    \Big)\\
    &= 
    \frac{1}{N_s } \sum_{g \in \Ss} 
    \Big ( 
    \EE_{P_{\mathbf{F}}^{\pi_{\{g,.\}}}}[F^{K,\gamma}_g] - 
    \EE_{P_{\mathbf{F}}^{\bar\pi_{\{-,.\}}^{\rm NM}}}[F^{K,\gamma}_g]
    \Big)
\end{aligned} \end{equation} \end{linenomath*}
Because $0 \le F^{K,\gamma}_g \le 1$, the absolute value of each term is, by the definition of the total variation distance, bounded by the total variation:
\begin{linenomath*} \begin{equation} \begin{aligned}
    \Big| 
    \EE_{P_{\mathbf{F}}^{\pi_{\{g,.\}}}}[F^{K,\gamma}_g]
    -
    \EE_{P_{\mathbf{F}}^{\bar\pi_{\{-,.\}}^{\rm NM}}}[F^{K,\gamma}_g]
    \Big|
    \le
    \|P_{\mathbf{F}}^{\pi_{\{g,.\}}}-P_{\mathbf{F}}^{\bar\pi_{\{-,.\}}^{\rm NM}}\|_{\rm TV}.
\end{aligned} \end{equation} \end{linenomath*}
Hence
\begin{linenomath*} \begin{equation} \begin{aligned}
    \big| \GSen_\OW(s,\pi_{\{.,.\}},K,\gamma) \big|
    &\le
    \frac{1}{N_s } \sum_{g \in \Ss} 
    \Big |
    \EE_{P_{\mathbf{F}}^{\pi_{\{g,.\}}}}[F^{K,\gamma}_g] - 
    \EE_{P_{\mathbf{F}}^{\bar\pi_{\{-,.\}}^{\rm NM}}}[F^{K,\gamma}_g]
    \Big|
    \\
    &\le
    \frac{1}{N_s}\sum_{g \in \Ss}\|P_{\mathbf{F}}^{\pi_{\{g,.\}}}-P_{\mathbf{F}}^{\bar\pi_{\{-,.\}}^{\rm NM}}\|_{\rm TV}.
\end{aligned} \end{equation} \end{linenomath*}
By Pinsker's inequality,
\begin{linenomath*} \begin{equation} \begin{aligned}
    \|P_{\mathbf{F}}^{\pi_{\{g,.\}}}-P_{\mathbf{F}}^{\bar\pi_{\{-,.\}}^{\rm NM}}\|_{\rm TV}^2
    \le
    \frac{1}{2}D_{\rm KL}(P_{\mathbf{F}}^{\pi_{\{g,.\}}}\|P_{\mathbf{F}}^{\bar\pi_{\{-,.\}}^{\rm NM}}).
\end{aligned} \end{equation} \end{linenomath*}
Applying Jensen's inequality,
\begin{linenomath*} \begin{equation} \begin{aligned}
    \GSen_\OW(s,\pi_{\{.,.\}},K,\gamma)^2
    &\le
    \frac{1}{N_s}\sum_{g \in \Ss}\|P_{\mathbf{F}}^{\pi_{\{g,.\}}}-P_{\mathbf{F}}^{\bar\pi_{\{-,.\}}^{\rm NM}}\|_{\rm TV}^2
    \le
    \frac{1}{2N_s}\sum_{g \in \Ss} D_{\rm KL}(P_{\mathbf{F}}^{\pi_{\{g,.\}}}\|P_{\mathbf{F}}^{\bar\pi_{\{-,.\}}^{\rm NM}}).
\end{aligned} \end{equation} \end{linenomath*}
Finally, using the standard identity for mutual information under a uniform prior,
\begin{linenomath*} \begin{equation} \begin{aligned}
    I^{\pi_{\{.,.\}}}\!\left(G;\mathbf{F}^{K,\gamma}\mid S_0=s\right)
    =
    \frac{1}{N_s}\sum_{g \in \Ss} D_{\rm KL}(P_{\mathbf{F}}^{\pi_{\{g,.\}}}\|P_{\mathbf{F}}^{\bar\pi_{\{-,.\}}^{\rm NM}}).
\end{aligned} \end{equation} \end{linenomath*}
Hence, the proof is complete.
$\hfill \square$

\subsection{Proof of \autoref{prop:MIZunifvsf}}
\label{sec:app:proofMIZunifvsf}

Using the chain rule for MI, we first note that
\begin{linenomath*} \begin{equation} \begin{aligned}
    I^{\pi_{\{.,.\}}}\!\left(G; S' | S_0=s \right) &+
    \underbrace{
    I^{\pi_{\{.,.\}}}\!\left(Z; S' | G, S_0=s \right)
    }_{=0 \text{ for deterministic $f$}}\\
    &=\\
    I^{\pi_{\{.,.\}}}\!\left(Z; S' | S_0=s \right) &+
    \underbrace{
    I^{\pi_{\{.,.\}}}\!\left(G; S' | Z, S_0=s \right)
    }_{= 0 \text{ because } G \perp\!\!\!\perp S' | Z}
\end{aligned} \end{equation} \end{linenomath*}
which proves \autoref{eq:MIGSZS}:
\begin{linenomath*} \begin{equation} \begin{aligned}
    I^{\pi_{\{.,.\}}}\!\left(G; S' | S_0=s \right)
    = I^{\pi_{\{.,.\}}^\skill}_{Z \sim p_f}\!\left(Z; S' | S_0=s \right).
\end{aligned} \end{equation} \end{linenomath*}
Hence, we have
\begin{linenomath*} \begin{equation} \begin{aligned}
    I^{\pi_{\{.,.\}}^\skill}_{Z \sim {\rm Unif}(\Zz)}\!\left(Z; S' \mid S_0=s \right) - I^{\pi_{\{.,.\}}}\!\left(G; S' | S_0=s \right)
    &=\\
    I^{\pi_{\{.,.\}}^\skill}_{Z \sim {\rm Unif}(\Zz)}\!\left(Z; S' \mid S_0=s \right)
    &-
    I^{\pi_{\{.,.\}}^\skill}_{Z \sim p_f}\!\left(Z; S' | S_0=s \right),
\end{aligned} \end{equation} \end{linenomath*}
where we can bound the right-hand side to prove the proposition.
To do so, let 
\begin{linenomath*} \begin{equation} \begin{aligned}
    u := {\rm Unif}(\Zz),
    \qquad
    p := p_{f},
    \qquad
    q_z := p^{\pi_{\{z,.\}}^\skill}(S' = . \mid S_0=s).
\end{aligned} \end{equation} \end{linenomath*}
Also, define the  marginals on $S'$ as
\begin{linenomath*} \begin{equation} \begin{aligned}
    \Tilde{\mu}(s')
    &:=
    \sum_{z \in \Zz} u(z)\, q_z(s')
    \quad \text{and} \quad
    \mu(s')
    :=
    \sum_{z \in \Zz} p(z)\, q_z(s').
\end{aligned} \end{equation} \end{linenomath*}
Then, by definition of conditional MI under the two skill priors,
\begin{linenomath*} \begin{equation} \begin{aligned}
    I^{\pi_{\{.,.\}}^\skill}_{Z \sim {\rm Unif}(\Zz)}\!\left(Z; S' \mid S_0=s \right)
    &=
    H(\Tilde{\mu})
    -
    \sum_{z \in \Zz} u(z)\, H(q_z),
    \\
    I^{\pi_{\{.,.\}}^\skill}_{Z \sim p_f}\!\left(Z; S' | S_0=s \right)
    &=
    H(\mu)
    -
    \sum_{z \in \Zz} p(z)\, H(q_z).
\end{aligned} \end{equation} \end{linenomath*}
Hence,
\begin{linenomath*} \begin{equation} \begin{aligned} \label{eq:app:proof:dMIZ}
    \Bigl|
    I^{\pi_{\{.,.\}}^\skill}_{Z \sim {\rm Unif}(\Zz)}\!\left(Z; S' \mid S_0=s \right)
    &-
    I^{\pi_{\{.,.\}}^\skill}_{Z \sim p_f}\!\left(Z; S' | S_0=s \right)
    \Bigr|
    \\
    &\leq
    \bigl| H(\Tilde{\mu}) - H(\mu) \bigr|
    +
    \left|
    \sum_{z \in \Zz}
    \bigl(u(z)-p(z)\bigr)\, H(q_z)
    \right|.
\end{aligned} \end{equation} \end{linenomath*}

We first bound the entropy difference term. 
Since $\Tilde{\mu}$ and $\mu$ are obtained by passing $u$ and $p$ through the same channel $z \mapsto q_z$, their total variation distance can be bounded by the distance between $u$ and $p$:
\begin{linenomath*} \begin{equation} \begin{aligned}
    \delta_{\mu} := \| \Tilde{\mu} - \mu \|_{\rm TV}
    &=
    \frac12 \sum_{s'}
    \left|
    \Tilde{\mu}(s') - \mu(s')
    \right|\\
    &=
    \frac12 \sum_{s'}
    \left|
    \sum_{z \in \Zz}
    \bigl(u(z)-p(z)\bigr)\, q_z(s')
    \right|
    \\
    &\leq
    \frac12 \sum_{z \in \Zz}
    |u(z)-p(z)|
    \underbrace{\sum_{s'} q_z(s')}_{= 1}
    =
    \|u-p\|_{\rm TV}
    =
    \delta.
\end{aligned} \end{equation} \end{linenomath*}
Moreover, since $N_z \leq N_{s'}$, we have
\begin{linenomath*} \begin{equation} \begin{aligned}
    \delta_{\mu} \le \delta = \|u-p\|_{\rm TV} = \frac{1}{2} \sum_{z \in \Zz} \Big|p(z) - \frac{1}{N_z} \Big| \leq 1 - \frac{1}{N_z} \leq 1 - \frac{1}{N_{s'}}.
\end{aligned} \end{equation} \end{linenomath*}
Therefore, using Theorem 3 of~\cite{sason2013entropy}, we have
\begin{linenomath*} \begin{equation} \begin{aligned}
    \bigl| H(\Tilde{\mu}) - H(\mu) \bigr|
    &\leq
    h(\delta_{\mu} ) + \delta_{\mu} \log(N_{s'}-1).
\end{aligned} \end{equation} \end{linenomath*}
Finally, since the function
\begin{linenomath*} \begin{equation} \begin{aligned}
h(\delta)+\delta\log(N_{s'}-1)
\end{aligned} \end{equation} \end{linenomath*}
is increasing on $\left[0,1-\frac1{N_{s'}}\right]$, it follows that
\begin{linenomath*} \begin{equation} \begin{aligned}
    \bigl| H(\Tilde{\mu}) - H(\mu) \bigr|
    &\leq
    h(\delta) + \delta \log(N_{s'}-1).
\end{aligned} \end{equation} \end{linenomath*}

Next, we bound the second term in \autoref{eq:app:proof:dMIZ}.
Since each $H(q_z) \leq \log N_{s'}$,
\begin{linenomath*} \begin{equation} \begin{aligned}
    \left|
    \sum_{z \in \Zz}
    \bigl(u(z)-p(z)\bigr)\, H(q_z)
    \right| 
    &\leq
    \sum_{z \in \Zz}
    \left|u(z)-p(z)\right| H(q_z)\\
    &\leq 
    \log N_{s'} \sum_{z \in \Zz}
    \left|u(z)-p(z)\right| = 2 \delta \log N_{s'}
\end{aligned} \end{equation} \end{linenomath*}

Combining the two bounds yields
\begin{linenomath*}
\begin{equation}
\begin{aligned}
    \Bigl|
    I^{\pi_{\{.,.\}}^\skill}_{Z \sim {\rm Unif}(\Zz)}\!\left(Z; S' \mid S_0=s \right)
    - &
    I^{\pi_{\{.,.\}}^\skill}_{Z \sim p_f}\!\left(Z; S' | S_0=s \right)
    \Bigr| \\
    &\leq
    h(\delta)
    +
    \delta \log(N_{s'}-1)
    +
    2\delta \log N_{s'}
    \\
    &=
    h(\delta)
    +
    \delta \log\!\bigl(N_{s'}^2(N_{s'}-1)\bigr).
\end{aligned}
\end{equation}
\end{linenomath*}
This proves the result.
\hfill $\square$

\newpage
\section{An MI upper bound for the OW formulation} 
\label{append:upper}
Finding an inverse bound for the 3rd statement of \autoref{theor:MI_empowerment} for $\OW(K,\gamma)$ requires further assumptions instead of the mere consistency of $\pi_{\{.,.\}}$.
In this section, we provide one possible upper bound under three additional assumptions.

\subsection{Additional assumptions}

\paragraph{Assumption 1. Stochastic consistency.}
The first assumption concerns a stochastic variant of consistency in \autoref{def:consistency}.
Specifically, we define the non-Markovian, goal-independent mixture policy $\bar{\pi}_{\{-,.\}}^{\rm NM}$ as
\begin{linenomath*} \begin{equation} \begin{aligned} \label{def:mixturepol}
    \bar{\pi}_{\{-,.\}}^{\rm NM}
    :=
    \frac{1}{N_s} \sum_{g \in \Gg} \pi_{\{g,.\}},
\end{aligned} \end{equation} \end{linenomath*}
which is equivalent to the policy for a randomly sampled goal state, independently of the commanded goal.
Using \autoref{eq:OW_J}, a direct consequence of the consistency condition in \autoref{def:consistency} is
\begin{linenomath*} \begin{equation} \begin{aligned} \label{def:consistency_OWmix}
    \underbrace{\EE^{\pi_{\{g,.\}}}\!\left[\gamma^{T_g-1} \Iin_{T_g \leq K} \mid S_0 = s \right]}_{\text{pursuing $g$ and being rewarded by $g$}}
    \geq
    \underbrace{\EE^{\bar{\pi}_{\{-,.\}}^{\rm NM}}\!\left[\gamma^{T_g-1} \Iin_{T_g \leq K} \mid S_0 = s \right]}_{\text{following $\bar{\pi}_{\{-,.\}}^{\rm NM}$ while being rewarded by $g$}}
\end{aligned} \end{equation} \end{linenomath*}
for every $g,s \in \Ss$.
We define \textit{stochastic consistency} as a stronger version of this:
\begin{linenomath*} \begin{equation} \begin{aligned} \label{def:consistency_OWmix_stoch}
    \underbrace{p^{\pi_{\{g,.\}}}\!\left(\gamma^{T_g-1} \Iin_{T_g \leq K} \geq r \mid S_0 = s \right)}_{\text{pursuing $g$ and being rewarded by $g$}}
    \geq
    \underbrace{p^{\bar{\pi}_{\{-,.\}}^{\rm NM}}\!\left(\gamma^{T_g-1} \Iin_{T_g \leq K} \geq r \mid S_0 = s \right)}_{\text{following $\bar{\pi}_{\{-,.\}}^{\rm NM}$ while being rewarded by $g$}}
\end{aligned} \end{equation} \end{linenomath*}
for all $r \in [0,1]$ and every $g,s \in \Ss$.

\paragraph{Assumption 2. Probable supports of $F_g^{K,\gamma}$.}
As in the proof of \autoref{theor:MI_empowerment}, let $P_{\mathbf{F}}^{\pi_{\{g,.\}}}$ denote the distribution of $\mathbf{F}^{K,\gamma}$ under the policy $\pi_{\{g,.\}}$, and let
\begin{linenomath*} \begin{equation} \begin{aligned}
    P_{\mathbf{F}}^{\bar\pi_{\{-,.\}}^{\rm NM}}
    :=
    \frac{1}{N_s}\sum_{g \in \Ss} P_{\mathbf{F}}^{\pi_{\{g,.\}}}.
\end{aligned} \end{equation} \end{linenomath*}
The second assumption avoids cases in which some possible values of $F_g^{K,\gamma}$ have a non-zero but arbitrarily small probability under the mixture policy.
Specifically, we assume that there exists $\eta_{K,\gamma}>0$ such that every value in the support of $P_{F_g}^{\bar\pi_{\{-,.\}}^{\rm NM}}$ has a probability of at least $\eta_{K,\gamma}$.

\paragraph{Assumption 3. Bounded interference.}
The third assumption is less intuitive than the others and concerns how different dimensions of $\mathbf{F}^{K,\gamma}$ interfere in the evaluation of
$I^{\pi_{\{.,.\}}}\!\left(G;\mathbf{F}^{K,\gamma}\mid S_0=s\right)$.
Specifically, we have
\begin{linenomath*} \begin{equation} \begin{aligned}
    I^{\pi_{\{.,.\}}}\!\left(G;\mathbf{F}^{K,\gamma}\mid S_0=s\right)
    =
    \frac{1}{N_s}\sum_{g \in \Ss} D_{\rm KL}(P_{\mathbf{F}}^{\pi_{\{g,.\}}}\|P_{\mathbf{F}}^{\bar\pi_{\{-,.\}}^{\rm NM}}).
\end{aligned} \end{equation} \end{linenomath*}
By the chain rule for KL divergence, for every $g \in \Ss$, we can further decompose the MI into two terms as
\begin{linenomath*} \begin{equation} \begin{aligned}
    D_{\rm KL}(P_{\mathbf{F}}^{\pi_{\{g,.\}}}\|P_{\mathbf{F}}^{\bar\pi_{\{-,.\}}^{\rm NM}})
    &=
    \underbrace{D_{\rm KL}(P_{F_g}^{\pi_{\{g,.\}}}\|P_{F_g}^{\bar\pi_{\{-,.\}}^{\rm NM}})}_{\text{the term specific to $g$}}
    +
    \underbrace{\EE_{P_{F_g}^{\pi_{\{g,.\}}}}
    \!\left[
    D_{\rm KL}\!\left(
    P_{\mathbf{F}_{-g}\mid F_g}^{\pi_{\{g,.\}}}
    \middle\|
    P_{\mathbf{F}_{-g}\mid F_g}^{\bar\pi_{\{-,.\}}^{\rm NM}}
    \right)
    \right]}_{\text{interference of the other terms}},
\end{aligned} \end{equation} \end{linenomath*}
where $\mathbf{F}_{-g}:=\{F_{g'}^{K,\gamma}\}_{g' \neq g}$.
The bounded interference assumption states that there exists $\epsilon_{K,\gamma} < \infty$ such that the average of the interference terms is bounded by $\epsilon_{K,\gamma}$, i.e.,
\begin{linenomath*} \begin{equation} \begin{aligned}
    I^{\pi_{\{.,.\}}}\!\left(G;\mathbf{F}^{K,\gamma}\mid S_0=s\right)
    \leq
    \frac{1}{N_s}\sum_{g \in \Ss} D_{\rm KL}(P_{F_g}^{\pi_{\{g,.\}}}\|P_{F_g}^{\bar\pi_{\{-,.\}}^{\rm NM}})
    + \epsilon_{K,\gamma}.
\end{aligned} \end{equation} \end{linenomath*}

\subsection{An upper bound on the goal-discounted-first-visit MI}

\begin{myprop}[An upper bound on discounted first-visit information for $\OW$]
    \label{prop:OW_upper_bound_F}
    Consider the problem formulation $\OW(K,\gamma)$, with $K > 1$ and $\gamma \in (0,1)$, together with the goal-conditioned policy $\pi_{\{.,.\}}$.
    Similar to \autoref{theor:MI_empowerment}, define
    $\mathbf{F}^{K,\gamma}:=(F^{K,\gamma}_g)_{g\in\Ss}$, where
    $F^{K,\gamma}_g:=\gamma^{T_g-1}\Iin_{T_g\le K}$.
    Suppose Assumptions 1--3 above hold. Then
    \begin{linenomath*} \begin{equation} \begin{aligned}
        I^{\pi_{\{.,.\}}}\!\left(G;\mathbf F^{K,\gamma}\mid S_0=s\right)
        \le
        \frac{4}{\eta_{K,\gamma}\delta_{K,\gamma}^2}
        \GSen_\OW(s,\pi_{\{.,.\}},K,\gamma)
        +
        \epsilon_{K,\gamma},
    \end{aligned} \end{equation} \end{linenomath*}
    where
    \begin{linenomath*} \begin{equation} \begin{aligned}
        \delta_{K,\gamma}
        :=
        \min\{|x-y|:\ x\neq y,\ x,y\in \Ff^{K,\gamma}\} = \min \{ \gamma^{K-1}, \gamma^{K-2}(1 - \gamma)\},
    \end{aligned} \end{equation} \end{linenomath*}
    with $\Ff^{K,\gamma} = \{0, \gamma^{K-1}, \gamma^{K-2}, \dots, 1 \}$ the set of all possible values of $F^{K,\gamma}_g$.
\end{myprop}
\begin{exmode}
\textit{Proof:}
For each $g\in\Ss$, define
\begin{linenomath*} \begin{equation} \begin{aligned}
    \Delta_g
    :=
    \underbrace{\EE^{\pi_{\{g,.\}}}[F_g^{K,\gamma}\mid S_0=s]}_{J_\OW(s,g,\pi_{\{g,.\}},K,\gamma)}
    -
    \underbrace{\EE^{\bar\pi_{\{-,.\}}^{\rm NM}}[F_g^{K,\gamma}\mid S_0=s]}_{J_\OW(s,g,\bar\pi_{\{-,.\}}^{\rm NM},K,\gamma)}.
\end{aligned} \end{equation} \end{linenomath*}
By definition of $\GSen_\OW$,
\begin{linenomath*} \begin{equation} \begin{aligned}
    \GSen_\OW(s,\pi_{\{.,.\}},K,\gamma)
    =
    \frac{1}{N_s}\sum_{g\in\Ss}\Delta_g.
\end{aligned} \end{equation} \end{linenomath*}

We proceed in three steps.

\textit{Step 1}: Fix $g\in\Ss$ and let the distinct values in the support of $F_g^{K,\gamma}$ be
\[
0 \le x_0 < x_1 < \cdots < x_m \le 1,
\qquad
\{x_0,\dots,x_m\} \subseteq \Ff^{K,\gamma}.
\]
Define
\begin{linenomath*} \begin{equation} \begin{aligned}
\label{eq:app:proof:OWupperS1ai}
a_i
:=
P^{\pi_{\{g,.\}}}(F_g^{K,\gamma}\ge x_i\mid S_0=s)
-
P^{\bar\pi_{\{-,.\}}^{\rm NM}}(F_g^{K,\gamma}\ge x_i\mid S_0=s).
\end{aligned} \end{equation} \end{linenomath*}
By Assumption 1, we have $a_i\ge 0$ for every $i$.
Using the tail-sum formula for expectations on a finite ordered support,
\begin{linenomath*} \begin{equation} \begin{aligned}
    \label{eq:app:proof:OWupperS1}
    \Delta_g
    &=
\sum_{i=1}^m (x_i-x_{i-1})a_i
    \ge
    \delta_{K,\gamma} \sum_{i=1}^m a_i
\end{aligned} \end{equation} \end{linenomath*}
Now let $d_i:=P^{\pi_{\{g,.\}}}(F_g^{K,\gamma}=x_i\mid S_0=s)-P^{\bar\pi_{\{-,.\}}^{\rm NM}}(F_g^{K,\gamma}=x_i\mid S_0=s)$.
Since $d_i=a_i-a_{i+1}$ (using \autoref{eq:app:proof:OWupperS1ai}) with $a_{m+1}:=0$, we have
\begin{linenomath*} \begin{equation} \begin{aligned}\label{eq:app:proof:TVdelta}
    \|P_{F_g}^{\pi_{\{g,.\}}}-P_{F_g}^{\bar\pi_{\{-,.\}}^{\rm NM}}\|_{\rm TV}
    \overset{\text{TV definition}}{=}
    \sum_{i=0}^m (d_i)_+
    =
    \sum_{i=0}^m (a_i-a_{i+1})_+
    \le
    \sum_{i=1}^m a_i
    \overset{\text{\autoref{eq:app:proof:OWupperS1}}}{\le}
    \frac{\Delta_g}{\delta_{K,\gamma}},
\end{aligned} \end{equation} \end{linenomath*}
where $(x)_+ := \max(x,0)$ denotes the positive part of $x$.

\textit{Step 2}: By Assumption 2 and the bound of KL by the $\chi^2$-divergence,
\begin{linenomath*} \begin{equation} \begin{aligned}\label{eq:app:proof:OWupperS2}
    D_{\rm KL}(P_{F_g}^{\pi_{\{g,.\}}}\|P_{F_g}^{\bar\pi_{\{-,.\}}^{\rm NM}})
    &\le
    \sum_x
    \frac{(P_{F_g}^{\pi_{\{g,.\}}}(x)-P_{F_g}^{\bar\pi_{\{-,.\}}^{\rm NM}}(x))^2}
    {P_{F_g}^{\bar\pi_{\{-,.\}}^{\rm NM}}(x)}\\
    &\overset{\text{Assumption 2}}{\le}
    \frac{1}{\eta_{K,\gamma}}
    \left\|P_{F_g}^{\pi_{\{g,.\}}}-P_{F_g}^{\bar\pi_{\{-,.\}}^{\rm NM}}\right\|_2^2
    \\
    &\overset{\text{$\ell_2$ $\leq$ $\ell_1$-norms}}{\le}
    \frac{1}{\eta_{K,\gamma}}
    \left\|P_{F_g}^{\pi_{\{g,.\}}}-P_{F_g}^{\bar\pi_{\{-,.\}}^{\rm NM}}\right\|_1^2\\
    &\overset{\text{TV definition}}{=}
    \frac{4}{\eta_{K,\gamma}}
    \left\|P_{F_g}^{\pi_{\{g,.\}}}-P_{F_g}^{\bar\pi_{\{-,.\}}^{\rm NM}}\right\|_{\rm TV}^2
    \\
    &\overset{\text{\autoref{eq:app:proof:TVdelta}}}{\le}
    \frac{4}{\eta_{K,\gamma}\delta_{K,\gamma}^2}\Delta_g^2.
\end{aligned} \end{equation} \end{linenomath*}

\textit{Step 3}: We can now use Assumption 3 which, together with \autoref{eq:app:proof:OWupperS2}, implies
\begin{linenomath*} \begin{equation} \begin{aligned}
    I^{\pi_{\{.,.\}}}\!\left(G;\mathbf F^{K,\gamma}\mid S_0=s\right)
    &\overset{\text{Assumption 3}}{\le}
    \frac{1}{N_s}\sum_{g\in\Ss}
    D_{\rm KL}(P_{F_g}^{\pi_{\{g,.\}}}\|P_{F_g}^{\bar\pi_{\{-,.\}}^{\rm NM}})
    +
    \epsilon_{K,\gamma}
    \\
    &\overset{\text{\autoref{eq:app:proof:OWupperS2}}}{\le}
    \frac{4}{\eta_{K,\gamma}\delta_{K,\gamma}^2}
    \frac{1}{N_s}\sum_{g\in\Ss}\Delta_g^2
    +
    \epsilon_{K,\gamma}.
\end{aligned} \end{equation} \end{linenomath*}
Finally, since $0\le F_g^{K,\gamma}\le 1$, we have $0\le \Delta_g\le 1$, hence $\Delta_g^2\le \Delta_g$. Thus,
\begin{linenomath*} \begin{equation} \begin{aligned}
    I^{\pi_{\{.,.\}}}\!\left(G;\mathbf F^{K,\gamma}\mid S_0=s\right)
    &\le
    \frac{4}{\eta_{K,\gamma}\delta_{K,\gamma}^2}
    \frac{1}{N_s}\sum_{g\in\Ss}\Delta_g
    +
    \epsilon_{K,\gamma}
    \\
    &=
    \frac{4}{\eta_{K,\gamma}\delta_{K,\gamma}^2}
    \GSen_\OW(s,\pi_{\{.,.\}},K,\gamma)
    +
    \epsilon_{K,\gamma}.
\end{aligned} \end{equation} \end{linenomath*}
Hence, the proof is complete.
$\hfill \square$
\end{exmode} 
\newpage
\section{On the attainability of consistency} \label{append:assump}
Although we present the consistency condition in \autoref{def:consistency} as an intuitive and natural property of a goal-conditioned policy, it is not needed for most of our results, including those in \autoref{sec:divergence}--\ref{sec:control}, \autoref{prop:MIZunifvsf}, and the lower bounds in \autoref{theor:MI_empowerment}. 
Thus, most of our central claims hold for arbitrary policies, consistent or not. 
The main exception is the upper bounds in \autoref{theor:MI_empowerment}, which are important for practical implications of our results and require the downstream goal-reaching policy to be consistent. 

In the following proposition, we show that this requirement is quite mild.
Specifically, suppose we are given a fixed skill-conditioned policy.
We show that one can always construct a consistent downstream goal-reaching policy by choosing an appropriate goal-to-skill mapping.
As a result, there is always a downstream goal-reaching policy for which the bounds in \autoref{theor:MI_empowerment} hold.

\begin{myprop}[Attainability of consistency by goal-to-skill mapping] \label{prop:consistency}
    Consider a skill-conditioned policy $\pi_{\{.,.\}}^{\skill}$ and a downstream GCRL task defined by $R_t(s; g)$ and $\gamma_t(s; g)$. 
    For each initial state $s_0$ and commanded goal $g$, define a goal-to-skill mapping $f: \Ss \times \Gg \to \Zz$ by
    \begin{linenomath*} 
    \begin{equation} 
    \begin{aligned}
        f(s_0, g) \in \arg \max_{z \in \Zz} J(s_0, g, \pi_{\{z,.\}}^{\skill}) .
    \end{aligned} 
    \end{equation} 
    \end{linenomath*}
    We define the resulting downstream goal-conditioned policy by
    \begin{linenomath*} 
    \begin{equation} 
    \begin{aligned}
        \pi_{g,t|s_0}(a|s) 
        := p^{\pi_{g,t|s_0}}(A_t = a \mid S_t = s, S_0 = s_0)
        := \pi_{f(s_0,g),t}^{\skill}(a|s),
    \end{aligned} 
    \end{equation} 
    \end{linenomath*}
    which is additionally conditioned on the initial state $S_0=s_0$.
    Writing $J(s, g, \pi_{\{g,.|.\}}) := J(s, g, \pi_{\{g,.|s\}})$, the policy $\pi_{\{.,.|.\}}$ is consistent.
\end{myprop}
\begin{exmode}
\textit{Proof.}
By construction, for all $s \in \Ss$ and $g,g' \in \Gg$,
\begin{linenomath*} \begin{equation} \begin{aligned}
    J(s, g, \pi_{\{g,.|.\}}) 
    &= J(s, g, \pi_{\{f(s,g),.\}}^{\skill}) 
    = \max_{z \in \Zz} J(s, g, \pi_{\{z,.\}}^{\skill}) \\
    &\geq J(s, g, \pi_{\{f(s,g'),.\}}^{\skill})
    = J(s, g, \pi_{\{g',.|.\}}).
\end{aligned} \end{equation} \end{linenomath*}
Hence, the proof is complete.
$\hfill \square$
\end{exmode}
The subtlety is that, unlike the goal-to-skill mapping in \autoref{eq:skillpolicy}, the mapping in \autoref{prop:consistency} may depend on the initial state $s_0$. 
This type of state-dependent skill selection is standard in RL pretraining~\cite{gregor2016variational, zheng2024can, park2023metra} and does not affect our theoretical claims.

The only consequence of conditioning the goal-to-skill mapping on $s_0$ is in the bound proven in \autoref{prop:MIZunifvsf}. 
The reason is that, when $f$ depends on $s_0$, the downstream skill distribution also becomes dependent on the initial state:
\begin{linenomath*} \begin{equation} \begin{aligned} 
    p_{f,s_0}(z)
    =
    \sum_{g \in \Gg} p_\goal(g)\, \Iin_{z=f(s_0,g)} 
    =
    \frac{1}{N_s} \sum_{g \in \Ss} \Iin_{z=f(s_0,g)} .
\end{aligned} \end{equation} \end{linenomath*}
Consequently, the discrepancy term in \autoref{prop:MIZunifvsf} should be replaced, for each initial state $s$, by
\begin{linenomath*} 
\begin{equation} 
\begin{aligned}
    \delta_s := \left\| p_{f,s} - \mathrm{Unif}(\Zz) \right\|_{\mathrm{TV}} .
\end{aligned} 
\end{equation} 
\end{linenomath*}
The bound in \autoref{prop:MIZunifvsf} then applies pointwise in $s$ with $\delta_s$. 
If one evaluates performance under an initial-state distribution, the corresponding gap can be averaged over that distribution.

\newpage
\section{Extension to general goal distributions} 
\label{append:goal_dist}

In this section, we show how our results generalize beyond the case of a uniform goal distribution.
Specifically, we continue to assume that $\Gg = \Ss$, but now allow $p_{\goal}$ to be an arbitrary distribution over $\Ss$, rather than restricting it to ${\rm Unif}(\Ss)$.

\cbl{Throughout this section, the parts highlighted in blue indicate the changes induced by taking $p_{\goal} \neq {\rm Unif}(\Ss)$.
The main modifications consist of replacing some equalities with inequalities and allowing certain bounds to depend on $p_{\goal}$.
\textit{However, the main conclusions in the paper remain valid.}}

\subsection{The results in \autoref{sec:divergence} remain intact}

An important observation is that the choice of $p_{\goal}$ does not affect the goal-conditioned values in \autoref{eq:general_value} and therefore has no impact on \autoref{eq:Pe_J}--\ref{eq:OW_J}.
This means that all results in \autoref{sec:divergence} (i.e., \autoref{prop:settings_contradicting} and all propositions in \autoref{append:props}, as summarized in \autoref{fig:taxonomy}) are independent of the choice of $p_{\goal}$ and remain valid for any non-uniform goal distribution.

\subsection{Generalization of the results in \autoref{sec:control}}

The main effect of $p_{\goal}$ is on the test-time performance in \autoref{eq:general_expvalue}
\begin{linenomath*} \begin{equation} \begin{aligned} 
    \label{eq:general_expvalue_pg}
    J(s, \pi_{\{.,.\}}) &:= 
    \EE_{G \sim p_\goal} [J(s, G, \pi_{\{.,.\}})] =
    \sum_{g \in \Ss} \cbl{p_\goal(g)} J(s, g, \pi_{\{g,.\}}).
\end{aligned} \end{equation} \end{linenomath*}
Accordingly, for the generalization of \autoref{sec:control}--\ref{sec:empowerment}, we need to adapt the definition of the goal sensitivity as
\begin{linenomath*} \begin{equation} \begin{aligned} \label{def:control_pg}
    \GSen(s, \pi_{\{.,.\}}) :&= 
    \EE_{G,G' \sim p_\goal}
    \Big[ J(s, G, \pi_{\{G,.\}}) - J(s, G, \pi_{\{G',.\}}) \Big]\\
    &= \sum_{g,g' \in \Ss} \cbl{p_\goal(g) p_\goal(g')}
    \Big( 
    J(s, g, \pi_{\{g,.\}}) - 
    J(s, g, \pi_{\{g',.\}}) 
    \Big)
\end{aligned} \end{equation} \end{linenomath*}
Using the new measure of goal sensitivity in \autoref{def:control_pg}, we now present the generalization of our main statements in \autoref{sec:control}.

\cbl{The main change in the results is that the previously proven equalities for $\Pe$ and $\ET$ will now be relaxed to inequalities, and the bounds proven for $\OW$ will depend on $p_\goal$.}

\subsubsection{Generalization of \autoref{theor:decomp}}
\begin{mytheor}[Generalization of \autoref{theor:decomp} to the case of non-uniform $p_\goal$] \label{theor:decomp_pg}
    Consider a GCRL formulation defined by $R_t(s; g)$, $\gamma_t(s; g)$, and \cbl{the goal distribution $p_\goal$}. 
    Let $p_\goal^{\min} := \min_g p_\goal(g)$ and $p_\goal^{\max} := \max_g p_\goal(g)$, then,
    \begin{enumerate}[leftmargin=2.5em]
        \item For $\Pe(\gamma)$ and $\ET(K)$, we have
        \begin{linenomath*} \begin{equation} \begin{aligned}
            \cbl{p_\goal^{\min}} + \GSen(s, \pi_{\{.,.\}})
            \leq
            J(s, \pi_{\{.,.\}}) 
            \leq 
            \cbl{p_\goal^{\max}} + \GSen(s, \pi_{\{.,.\}})
        \end{aligned} \end{equation} \end{linenomath*}
        and, as a result,
        \begin{linenomath*} \begin{equation} \begin{aligned}
            0
            \leq
            J^*(s)
            -
            J(s,\pi^{\GSen*}_{\{.,.\}})
            \leq 
            \cbl{p_\goal^{\max} - p_\goal^{\min}}.
        \end{aligned} \end{equation} \end{linenomath*}
        Equalities hold if $p_\goal = {\rm Unif}(\Ss)$.
        
        \item For $\OW(K, \gamma)$ and any formulation with non-negative rewards (i.e., $R_t(s; g) \geq 0$), we have
        \begin{linenomath*} \begin{equation} \begin{aligned}
            J(s,\pi_{\{.,.\}})
            \geq
            \cbl{\frac{1}{1 - p_\goal^{\min}}}\GSen(s,\pi_{\{.,.\}}).
        \end{aligned} \end{equation} \end{linenomath*}
        Equality holds if $J(s, g, \pi_{\{g',.\}}) = 0$ for all $g' \neq g$ and $p_\goal = {\rm Unif}(\Ss)$.

        \item For $\OW(K, \gamma)$,
        \begin{linenomath*} \begin{equation} \begin{aligned}
            0 \le J^*(s) - J(s, \pi^{\GSen*}_{\{.,.\}}) \le 
            1 - \cbl{\frac{1}{1 - p_\goal^{\min}}}\GSen^*(s)
        \end{aligned} \end{equation} \end{linenomath*}
        so larger $\GSen^*(s)$ and $p_\goal^{\min}$ yield a tighter bound.
    \end{enumerate}
\end{mytheor}
\begin{exmode}
\textit{Proof:}
As a direct consequence of \autoref{eq:general_expvalue_pg}, we have
\begin{linenomath*} \begin{equation} \begin{aligned} J(s, \pi_{\{.,.\}}) - 
    J(s, \pi_{\{g',.\}}) = 
    \sum_{g \in \Ss} 
    p_\goal(g) \Big(J(s, g, \pi_{\{g,.\}}) - J(s, g, \pi_{\{g',.\}})\Big)
\end{aligned} \end{equation} \end{linenomath*}
for an arbitrary goal $g' \in \Ss$.
By further averaging over $g'$, we have
\begin{linenomath*} \begin{equation} \begin{aligned} \label{eq:decomJ_pg}
    J(s, \pi_{\{.,.\}}) = 
    &\underbrace{
    \sum_{g' \in \Ss} p_\goal(g') J(s, \pi_{\{g',.\}})
    }_{\textnormal{goal-agnostic value of $\pi$}} \\ 
    & + 
    \underbrace{
    \sum_{g, g' \in \Ss} p_\goal(g) p_\goal(g')
    \Big ( J(s, g, \pi_{\{g,.\}}) - J(s, g, \pi_{\{g',.\}}) \Big)
    }_{= \GSen(s, \pi_{\{.,.\}}) \textnormal{; see \autoref{def:control_pg}}}.
\end{aligned} \end{equation} \end{linenomath*}

\textit{1. For the case of $\Pe(\gamma)$ and $\ET(K)$:}
A direct consequence of \autoref{eq:Pe_J} and \autoref{eq:ET_J}, respectively, is that 
\begin{linenomath*} \begin{equation} \begin{aligned}
    J(s, \pi_{\{g',.\}}) = \sum_{g} p_\goal(g) J(s, g, \pi_{\{g',.\}}) 
    \in [ p_\goal^{\min} , p_\goal^{\max}],
\end{aligned} \end{equation} \end{linenomath*}
where $p_\goal^{\min} := \min_g p_\goal(g)$ and $p_\goal^{\max} := \max_g p_\goal(g)$.
As a result, for \autoref{eq:decomJ_pg}, we have
\begin{linenomath*} \begin{equation} \begin{aligned}
    p_\goal^{\min} + \GSen(s, \pi_{\{.,.\}})
    \leq
    J(s, \pi_{\{.,.\}}) 
    \leq 
    p_\goal^{\max} + \GSen(s, \pi_{\{.,.\}}).
\end{aligned} \end{equation} \end{linenomath*}

Therefore, for the maximally in-control policy $\pi^{\GSen*}_{\{.,.\}}$, we have
\begin{linenomath*} \begin{equation} \begin{aligned}
    J(s,\pi^{\GSen*}_{\{.,.\}})
    \geq 
    p_\goal^{\min} + \GSen^*(s)
\end{aligned} \end{equation} \end{linenomath*}
and
\begin{linenomath*} \begin{equation} \begin{aligned}
    J^*(s)
    \leq 
    p_\goal^{\max} + \GSen(s, \pi^*_{\{.,.\}})
    \leq 
    p_\goal^{\max} + \GSen^*(s).
\end{aligned} \end{equation} \end{linenomath*}
Hence, it follows that
\begin{linenomath*} \begin{equation} \begin{aligned}
    0
    \leq
    J^*(s)
    -
    J(s,\pi^{\GSen*}_{\{.,.\}})
    \leq 
    p_\goal^{\max}
    - p_\goal^{\min}
\end{aligned} \end{equation} \end{linenomath*}
Hence, the proof of the first part is complete.

\textit{2. $J$-$\GSen$ relationship for $\OW(K, \gamma)$ and formulations with non-negative rewards:}
Using \autoref{eq:decomJ_pg}, we have
\begin{linenomath*} \begin{equation} \begin{aligned}
    \GSen(s,\pi_{\{.,.\}})
    =
    J(s,\pi_{\{.,.\}})
    -
    \sum_{g' \in \Ss} p_\goal(g') J(s, \pi_{\{g',.\}})
\end{aligned} \end{equation} \end{linenomath*}
Moreover, the non-negative reward assumption implies that, for every $g \in \Ss$ and $g' \in \Ss$,
\[
J(s,g,\pi_{\{g',.\}}) \geq 0.
\]
Hence,
\begin{linenomath*} \begin{equation} \begin{aligned}
    \sum_{g' \in \Ss} p_\goal(g') J(s, \pi_{\{g',.\}})
    &=
    \sum_{g,g' \in \Ss} p_\goal(g') p_\goal(g) J(s,g,\pi_{\{g',.\}})
    \\
    &\geq
    \sum_{g \in \Ss} p_\goal^2(g) J(s,g,\pi_{\{g,.\}})
    \\
    &\geq
    p_\goal^{\min} J(s,\pi_{\{.,.\}}).
\end{aligned} \end{equation} \end{linenomath*}
Since all the terms are non-negative, the inequality is tight if $J(s,g,\pi_{\{g',.\}})=0$ for all $g \neq g'$ and $p_\goal = {\rm Unif}(\Ss)$.
Therefore,
\begin{linenomath*} \begin{equation} \begin{aligned}
    \GSen(s,\pi_{\{.,.\}})
    &\leq
    J(s,\pi_{\{.,.\}}) - p_\goal^{\min} J(s,\pi_{\{.,.\}}),
\end{aligned} \end{equation} \end{linenomath*}
which is equivalent to
\begin{linenomath*} \begin{equation} \begin{aligned}
    J(s,\pi_{\{.,.\}})
    \geq
    \frac{1}{1 - p_\goal^{\min}}\GSen(s,\pi_{\{.,.\}}).
\end{aligned} \end{equation} \end{linenomath*}
Hence, the proof of the second part is complete.

\textit{3. The maximally in-control policy for $\OW(K, \gamma)$:} 
Let $\pi^{\GSen*}_{\{.,.\}}$ be a maximally in-control policy. 
By part 2, we have
\begin{linenomath*} \begin{equation} \begin{aligned}
    J(s,\pi^{\GSen*}_{\{.,.\}})
    \geq
    \frac{1}{1 - p_\goal^{\min}}\GSen^*(s).
\end{aligned} \end{equation} \end{linenomath*}
Since $J^*(s)\leq 1$ for $\OW(K,\gamma)$, it follows that
\begin{linenomath*} \begin{equation} \begin{aligned}
    0 \le
    J^*(s)-J(s,\pi^{\GSen*}_{\{.,.\}})
    \le
    1-\frac{1}{1 - p_\goal^{\min}}\GSen^*(s).
\end{aligned} \end{equation} \end{linenomath*}
Hence, the proof is complete.
$\hfill \square$
\end{exmode}

\subsection{Generalization of the results in \autoref{sec:empowerment}}

For generalization of the results in \autoref{sec:empowerment}, we need to adapt the consistency condition to a stronger version which additionally depends on goal probabilities,
\begin{linenomath*} \begin{equation} \begin{aligned} \label{def:consistency_pg}
    \underbrace{J(s, g, \pi_{\{g,.\}})}_{\text{pursuing $g$ and being rewarded by $g$}}
    \geq
    \cbl{\max \left\{ 1, \frac{p_\goal(\textcolor{myred2}{g'})}{p_\goal(g)} \right\}}
    \cdot
    \underbrace{J(s, g, \pi_{\{\textcolor{myred2}{g'},.\}})}_{\text{pursuing \textcolor{myred2}{$g'$} while being rewarded by $g$}}
\end{aligned} \end{equation} \end{linenomath*}
for every $s,g,\textcolor{myred2}{g'} \in \Ss$.
The new condition implies that, if $\textcolor{myred2}{g'}$ has a very high probability of being sampled as a goal, then consistency requires $J(s, g, \pi_{\{g,.\}})$ to be larger than $J(s, g, \pi_{\{\textcolor{myred2}{g'},.\}})$ by a large margin.
It is straightforward to see that the strong consistency in \autoref{def:consistency_pg} naturally implies the weaker consistency in \autoref{def:consistency}; the two are equivalent for $p_\goal = {\rm Unif}(\Ss)$.
The main reason we need this stronger notion of consistency is to generalize \autoref{lemma:optimaldecoder}; however, we note that attainability of consistency according to \autoref{def:consistency_pg} is non-trivial and cannot be guaranteed as in \autoref{append:assump}.

Using the new measure of goal sensitivity in \autoref{def:control_pg} and the new condition for consistency in \autoref{def:consistency_pg}, we now present the generalization of our main statements in \autoref{sec:empowerment}.

\cbl{The main change in the results is that both the lower and upper bounds for $\Pe$ and $\ET$ will depend on $p_\goal$, with the upper bound requiring the strong consistency condition in \autoref{def:consistency_pg}.}

\subsubsection{Generalization of \autoref{lemma:optimaldecoder}}

\begin{mylemma}[Generalization of \autoref{lemma:optimaldecoder} to the case of non-uniform $p_\goal$] \label{lemma:optimaldecoder_pg}
    Consider problem formulations $\Pe(\gamma)$ and $\ET(K)$ together with a goal-conditioned policy $\pi_{\{.,.\}}$ and \cbl{the goal distribution $p_\goal$}.
    Let $S' := S_{\gamma,+}$ for $\Pe(\gamma)$ and $S' := S_{K}$ for $\ET(K)$, respectively.
    Suppose our aim is to estimate $G$ based on $S'$ and consider the naïve identity estimator $\hat{G}(s') = s'$
    as well as the Bayes-optimal estimator
    \begin{linenomath*} \begin{equation} \begin{aligned}
        \hat{G}^*(s') &= \arg \max_{g} p^{\pi_{\{.,.\}}}(G = g | S' = s', S_0=s).
    \end{aligned} \end{equation} \end{linenomath*}
    If the policy $\pi_{\{.,.\}}$ is \cbl{consistent according to \autoref{def:consistency_pg}}, then $\hat{G}$ and $\hat{G}^*$ have the same error probability:
    \begin{linenomath*} \begin{equation} \begin{aligned}
    p_e
    :=
    p^{\pi_{\{.,.\}}}(\hat{G}(S') \neq G \mid S_0=s)
    =
    p^{\pi_{\{.,.\}}}(\hat{G}^*(S') \neq G \mid S_0=s)
    =:
    p^*_e.
\end{aligned} \end{equation} \end{linenomath*}
\end{mylemma}
\begin{exmode}
\textit{Proof:}
The error probability for the naïve identity estimator is given by
\begin{linenomath*} \begin{equation} \begin{aligned}
    p_e
    :&=
    p^{\pi_{\{.,.\}}}(\hat{G}(S') \neq G \mid S_0=s)\\
    &=
    1 - p^{\pi_{\{.,.\}}}(S' = G \mid S_0=s)
    \overset{\text{\autoref{eq:Pe_J} and \autoref{eq:ET_J}}}{=} 
    1 - J(s , \pi_{\{.,.\}}).
\end{aligned} \end{equation} \end{linenomath*}
where $J$ corresponds to either $\Pe(\gamma)$ or $\ET(K)$, depending on whether $S' = S_{\gamma,+}$ or $S' = S_{K}$, respectively.

The analogous error probability for the Bayes-optimal estimator is given by
\begin{linenomath*} \begin{equation} \begin{aligned} \label{eq:proof:peopt_pg}
    p^*_e
    :&=
    p^{\pi_{\{.,.\}}}(\hat{G}^*(S') \neq G \mid S_0=s)\\
    &= 
    1 - p^{\pi_{\{.,.\}}}(\hat{G}^*(S') = G \mid S_0=s)\\
    &=
    1 - \sum_{g' \in \Ss} 
    p^{\pi_{\{.,.\}}}(S' = g' | S_0=s) 
    \underbrace{
    p^{\pi_{\{.,.\}}}(\hat{G}^*(S') = G | S' = g' , S_0=s)
    }_{= \max_{g} p^{\pi_{\{.,.\}}}(G = g | S' = g', S_0=s)}.
\end{aligned} \end{equation} \end{linenomath*}
Additionally, using the Bayes rule, we have
\begin{linenomath*} \begin{equation} \begin{aligned} \label{eq:proof:peopt_pg2}
    \max_{g} p^{\pi_{\{.,.\}}}(G = g | S' = g', S_0=s) =
    \max_{g} 
    \frac
    {p^{\pi_{\{.,.\}}}(S' = g'|G = g , S_0=s) p_\goal(g)}
    {p^{\pi_{\{.,.\}}}(S' = g'| S_0=s)} 
\end{aligned} \end{equation} \end{linenomath*}
which, combined with \autoref{eq:proof:peopt_pg}, gives
\begin{linenomath*} \begin{equation} \begin{aligned} \label{eq:proof:peopt_pg3}
    p^*_e &= 
    1 - \sum_{g' \in \Ss} 
    \max_{g} 
    \underbrace{p^{\pi_{\{.,.\}}}(S' = g'|G = g , S_0=s)}_{= J(s, g', \pi_{\{g,.\}})} 
    p_\goal(g)\\
    &=
    1 - \sum_{g' \in \Ss} \max_{g} p_\goal(g) J(s, g', \pi_{\{g,.\}}) \leq p_e.
\end{aligned} \end{equation} \end{linenomath*}
where $J$ corresponds to either $\Pe(\gamma)$ or $\ET(K)$, depending on whether $S' = S_{\gamma,+}$ or $S' = S_{K}$, respectively.
If the policy $\pi_{\{.,.\}}$ is consistent according to \autoref{def:consistency_pg}, then we have 
\begin{linenomath*} \begin{equation} \begin{aligned}
    \max_{g} p_\goal(g) J(s, g', \pi_{\{g,.\}}) = p_\goal(g') J(s, g', \pi_{\{g',.\}})
\end{aligned} \end{equation} \end{linenomath*}
Using this equality in \autoref{eq:proof:peopt_pg3}, we have
\begin{linenomath*} \begin{equation} \begin{aligned}
    p^*_e = 1 - J(s , \pi_{\{.,.\}}) = p_e.
\end{aligned} \end{equation} \end{linenomath*}
Hence, the proof is complete.
$\hfill \square$
\end{exmode}

\subsubsection{Generalization of \autoref{theor:MI_empowerment}}

\begin{mytheor}[Generalization of \autoref{theor:MI_empowerment} to the case of non-uniform $p_\goal$]
\label{theor:MI_empowerment_pg}
    Let
    $\GSen_{\Pe}(s,\pi_{\{.,.\}},\gamma)$,
    $\GSen_{\ET}(s,\pi_{\{.,.\}},K)$, and
    $\GSen_{\OW}(s,\pi_{\{.,.\}},K,\gamma)$
    be the goal sensitivities associated with
    $\Pe(\gamma)$, $\ET(K)$, and $\OW(K,\gamma)$ for policy $\pi_{\{.,.\}}$ and \cbl{the goal distribution $p_\goal$}. 
    \begin{enumerate}[leftmargin=1.5em]
    \item For $\Pe(\gamma)$ and $\ET(K)$, we have attainable lower bounds,
    \begin{linenomath*} 
    \begin{empheq}[box=\fbox]{equation*}
    \begin{aligned}
    I^{\pi_{\{.,.\}}}\!\left(G;\textcolor{mypurple}{S_{\gamma,+}}\mid S_0=s\right)
    &\ge
    \Phi^{\rm down}_{N_s}\!\Big(
    J_{\textcolor{mypurple}{\Pe}}(s,\pi_{\{.,.\}},\gamma)
    \Big)
,\\
    I^{\pi_{\{.,.\}}}\!\left(G;\textcolor{mypurple}{S_K}\mid S_0=s\right)
    &\ge
    \Phi^{\rm down}_{N_s}\!\Big(
    J_{\textcolor{mypurple}{\ET}}(s,\pi_{\{.,.\}},K)
    \Big)
,
    \end{aligned}
    \end{empheq}
    \end{linenomath*} 
    where $\Phi^{\rm down}_N(x):= \cbl{H[p_\goal]} -h(x)-(1-x)\log(N-1)$ , with $h$ the binary entropy and $H[p_\goal]$ the entropy of the goal distribution.
    $\Phi^{\rm down}_{N_s}(x)$ is increasing for $x = J \in [1/N_s, 1]$.
    
    \item For $\Pe(\gamma)$ and $\ET(K)$, if $\pi_{\{.,.\}}$ is \cbl{\emph{consistent} according to \autoref{def:consistency_pg}}, then we have attainable upper bounds,
    \begin{linenomath*} 
    \begin{empheq}[box=\fbox]{equation*}
    \begin{aligned}
    I^{\pi_{\{.,.\}}}\!\left(G;\textcolor{mypurple}{S_{\gamma,+}}\mid S_0=s\right)
    &\le
    \Phi^{\rm up}_{N_s}\!\Big(
    J_{\textcolor{mypurple}{\Pe}}(s,\pi_{\{.,.\}},\gamma)
    \Big)
,\\
    I^{\pi_{\{.,.\}}}\!\left(G;\textcolor{mypurple}{S_K}\mid S_0=s\right)
    &\le
    \Phi^{\rm up}_{N_s}\!\Big(
    J_{\textcolor{mypurple}{\ET}}(s,\pi_{\{.,.\}},K)
    \Big)
,
    \end{aligned}
    \end{empheq}
    \end{linenomath*} 
    where
    $\Phi^{\rm up}_N(x):=\cbl{H[p_\goal]}
    -\big(\lceil x^{-1}\rceil x-1\big)\lfloor x^{-1}\rfloor\log\lfloor x^{-1}\rfloor
    -\big(1-\lfloor x^{-1}\rfloor x\big)\lceil x^{-1}\rceil\log\lceil x^{-1}\rceil$, with $\lfloor\cdot\rfloor$ denoting the floor and $\lceil\cdot\rceil := \lfloor\cdot\rfloor + 1$.
    $\Phi^{\rm up}_{N_s}(x)$ is increasing for $x = J \in [1/N_s, 1]$.
    
    \item For $\OW(K,\gamma)$, let
    $\mathbf{F}^{K,\gamma}:=(F^{K,\gamma}_g)_{g\in\Ss}$, where
    $F^{K,\gamma}_g:=\gamma^{T_g-1}\Iin_{T_g\le K}$ and
    $T_g:=\min\{t\ge1:S_t=g\}$. Then
    \begin{linenomath*} \begin{empheq}[box=\fbox]{equation*}
    I^{\pi_{\{.,.\}}}\!\left(G;\textcolor{mypurple}{\mathbf{F}^{K,\gamma}}\mid S_0=s\right)
    \ge
    2\,\GSen_{\textcolor{mypurple}{\OW}}(s,\pi_{\{.,.\}},K,\gamma)^2.
    \end{empheq}
    \end{linenomath*} 
    \end{enumerate}
\end{mytheor}
\begin{exmode}
\textit{Proof:}
We split the proof into two parts.

\textit{1-2. The case of $\Pe(\gamma)$ and $\ET(K)$:} 
We prove the two inequalities in parallel using
\begin{linenomath*} \begin{equation} \begin{aligned}
    I^{\pi_{\{.,.\}}}(G;S_{\gamma,+}\mid S_0=s)
    &=
    H(G)-H(G \mid S_{\gamma,+}, S_0=s)
    \\
    I^{\pi_{\{.,.\}}}(G;S_{K}\mid S_0=s)
    &=
    H(G)-H(G \mid S_{K}, S_0=s)
\end{aligned} \end{equation} \end{linenomath*}
respectively for $\Pe(\gamma)$ and $\ET(K)$.
Since $G \sim p_\goal$, the marginal goal entropy $H(G)= H[p_\goal]$ is constant.
Hence, we need to find a bound for the conditional entropies. 

To simplify the proof, we define the random variable $S'$ as $S_{\gamma,+}$ for $\Pe(\gamma)$ and as $S_{K}$ for $\ET(K)$.
Then, from an information-theoretic perspective, we can view the commanded goal $G$ as a message to be transmitted, but the receiver receives only $S'$, sampled by running the policy $\pi_{\{G,.\}}$.
Given this communication problem, we define the naïve identity decoder $\hat{G}(s') = s'$ and the Bayes-optimal decoder $\hat{G}^*$ as in \autoref{lemma:optimaldecoder_pg}.
According to \autoref{lemma:optimaldecoder_pg}, for a general policy, we have
\begin{linenomath*} \begin{equation} \begin{aligned}
1 - J(s , \pi_{\{.,.\}}) = p_e \geq p^*_e.
\end{aligned} \end{equation} \end{linenomath*}
where $J$ corresponds to either $\Pe(\gamma)$ or $\ET(K)$, depending on whether $S' = S_{\gamma,+}$ or $S' = S_{K}$, respectively.
Hence, using Fano's inequality applied to the naïve identity decoder $\hat{G}(s') = s'$, we have
\begin{linenomath*} \begin{equation} \begin{aligned}
    H(G \mid S', S_0=s)
    &\leq
h(p_e) + p_e \log(N_s-1)\\
    &=
    h(1 - J(s , \pi_{\{.,.\}}) ) + (1 - J(s , \pi_{\{.,.\}}) ) \log(N_s-1)
.
\end{aligned} \end{equation} \end{linenomath*}
This proves the general lower bound:
\begin{linenomath*} \begin{equation} \begin{aligned}
    I^{\pi_{\{.,.\}}}(G;S'\mid S_0=s)
    &\geq
    \Phi^{\rm down}_{N_s}\!\Big(J(s,\pi_{\{.,.\}})\Big)
.
\end{aligned} \end{equation} \end{linenomath*}

For the upper bound, we need the reverse-Fano's inequality, which is given only for the optimal decoder~\cite{rioul2023interplay, tebbe1968uncertainty}.
If the policy $\pi_{\{.,.\}}$ is consistent according to \autoref{def:consistency_pg}, then we can use \autoref{lemma:optimaldecoder_pg} and, as a result,
\begin{linenomath*} \begin{equation} \begin{aligned} \label{eq:proof:consistentpe_pg}
    \text{for consistent policies:} \quad p^*_e = p_e = 
    1 - J(s , \pi_{\{.,.\}}).
\end{aligned} \end{equation} \end{linenomath*}
Using \autoref{eq:proof:consistentpe_pg} and Example 14 of~\cite{rioul2023interplay} directly yield the upper bound, for consistent policies:
\begin{linenomath*} \begin{equation} \begin{aligned}
    I^{\pi_{\{.,.\}}}(G;S'\mid S_0=s)
    &\leq
    \Phi^{\rm up}_{N_s}\!\Big(J(s,\pi_{\{.,.\}})\Big).
\end{aligned} \end{equation} \end{linenomath*}
Hence, the proof of the first two parts is complete.

\textit{3. The case of $\OW(K, \gamma)$:} 
For each $g \in \Ss$, let $P_{\mathbf{F}}^{\pi_{\{g,.\}}}$ denote the distribution of $\mathbf{F}^{K,\gamma}$ under the policy $\pi_{\{g,.\}}$, and let
\begin{linenomath*} \begin{equation} \begin{aligned}
    P_{\mathbf{F}}^{\bar\pi_{\{-,.\}}^{\rm NM}} := \sum_{g \in \Ss} p_\goal(g) P_{\mathbf{F}}^{\pi_{\{g,.\}}}.
\end{aligned} \end{equation} \end{linenomath*}
By \autoref{eq:OW_J}, we have
\begin{linenomath*} \begin{equation} \begin{aligned}
    J_\OW(s,g,\pi_{\{g',.\}},K,\gamma)
    =
    \EE^{\pi_{\{g',.\}}}\!\left[F^{K,\gamma}_g \mid S_0=s\right].
\end{aligned} \end{equation} \end{linenomath*}
This, together with the definition of $\GSen_\OW$ in \autoref{def:control_pg}, implies
\begin{linenomath*} \begin{equation} \begin{aligned} 
    \GSen_\OW(s,\pi_{\{.,.\}},K,\gamma) 
    &= 
    \sum_{g, g' \in \Ss} 
    p_\goal(g)
    p_\goal(g')
    \Big ( 
    \EE_{P_{\mathbf{F}}^{\pi_{\{g,.\}}}}[F^{K,\gamma}_g] - 
    \EE_{P_{\mathbf{F}}^{\pi_{\{g',.\}}}}[F^{K,\gamma}_g]
    \Big)\\
    &= 
    \sum_{g \in \Ss} 
    p_\goal(g)
    \Big ( 
    \EE_{P_{\mathbf{F}}^{\pi_{\{g,.\}}}}[F^{K,\gamma}_g] - 
    \EE_{P_{\mathbf{F}}^{\bar\pi_{\{-,.\}}^{\rm NM}}}[F^{K,\gamma}_g]
    \Big)
\end{aligned} \end{equation} \end{linenomath*}
Because $0 \le F^{K,\gamma}_g \le 1$, the absolute value of each term is, by the definition of the total variation distance, bounded by the total variation:
\begin{linenomath*} \begin{equation} \begin{aligned}
    \Big | 
    \EE_{P_{\mathbf{F}}^{\pi_{\{g,.\}}}}[F^{K,\gamma}_g]
    -
    \EE_{P_{\mathbf{F}}^{\bar\pi_{\{-,.\}}^{\rm NM}}}[F^{K,\gamma}_g]
    \Big | 
    \le
    \|P_{\mathbf{F}}^{\pi_{\{g,.\}}}-P_{\mathbf{F}}^{\bar\pi_{\{-,.\}}^{\rm NM}}\|_{\rm TV}.
\end{aligned} \end{equation} \end{linenomath*}
Hence
\begin{linenomath*} \begin{equation} \begin{aligned}
    \big| \GSen_\OW(s,\pi_{\{.,.\}},K,\gamma) \big|
    &\le
    \sum_{g \in \Ss} 
    p_\goal(g)
    \Big | 
    \EE_{P_{\mathbf{F}}^{\pi_{\{g,.\}}}}[F^{K,\gamma}_g] - 
    \EE_{P_{\mathbf{F}}^{\bar\pi_{\{-,.\}}^{\rm NM}}}[F^{K,\gamma}_g]
    \Big |
    \\
    &\le
    \sum_{g \in \Ss}
    p_\goal(g)
    \|P_{\mathbf{F}}^{\pi_{\{g,.\}}}-P_{\mathbf{F}}^{\bar\pi_{\{-,.\}}^{\rm NM}}\|_{\rm TV}.
\end{aligned} \end{equation} \end{linenomath*}
By Pinsker's inequality,
\begin{linenomath*} \begin{equation} \begin{aligned}
    \|P_{\mathbf{F}}^{\pi_{\{g,.\}}}-P_{\mathbf{F}}^{\bar\pi_{\{-,.\}}^{\rm NM}}\|_{\rm TV}^2
    \le
    \frac{1}{2}D_{\rm KL}(P_{\mathbf{F}}^{\pi_{\{g,.\}}}\|P_{\mathbf{F}}^{\bar\pi_{\{-,.\}}^{\rm NM}}).
\end{aligned} \end{equation} \end{linenomath*}
Applying Jensen's inequality,
\begin{linenomath*} \begin{equation} \begin{aligned}
    \GSen_\OW(s,\pi_{\{.,.\}},K,\gamma)^2
    &\le
    \sum_{g \in \Ss} p_\goal(g) \|P_{\mathbf{F}}^{\pi_{\{g,.\}}}-P_{\mathbf{F}}^{\bar\pi_{\{-,.\}}^{\rm NM}}\|_{\rm TV}^2\\
    &\le
    \frac{1}{2}\sum_{g \in \Ss} p_\goal(g) D_{\rm KL}(P_{\mathbf{F}}^{\pi_{\{g,.\}}}\|P_{\mathbf{F}}^{\bar\pi_{\{-,.\}}^{\rm NM}}).
\end{aligned} \end{equation} \end{linenomath*}
Finally, using the following standard identity for mutual information,
\begin{linenomath*} \begin{equation} \begin{aligned}
    I^{\pi_{\{.,.\}}}\!\left(G;\mathbf{F}^{K,\gamma}\mid S_0=s\right)
    =
    \sum_{g \in \Ss} p_\goal(g) D_{\rm KL}(P_{\mathbf{F}}^{\pi_{\{g,.\}}}\|P_{\mathbf{F}}^{\bar\pi_{\{-,.\}}^{\rm NM}}).
\end{aligned} \end{equation} \end{linenomath*}
Hence, the proof is complete.
$\hfill \square$
\end{exmode}

\subsubsection{Generalization of \autoref{prop:MIZunifvsf}}

\begin{myprop}[Generalization of \autoref{prop:MIZunifvsf} to the case of non-uniform $p_\goal$] \label{prop:MIZunifvsf_pg}
    Consider a policy $\pi_{\{.,.\}}^\skill$ and a \emph{deterministic} mapping $f: \Ss \to \Zz$, together with the resulting downstream policy $\pi_{\{g,.\}}$ and induced skill distribution
    \begin{linenomath*} \begin{equation} \begin{aligned}
        p_f(z) = \sum_{g \in \Ss} \cbl{p_\goal(g)}\, \Iin_{z=f(g)}
    \end{aligned} \end{equation} \end{linenomath*}
    Suppose $S'$ is defined based on the agent trajectory and takes at most $N_{s'} > 1$ values, e.g., $N_{s'} = N_s$ when $S' = S_K$.
    Then, if $N_z \leq N_{s'}$,
    \begin{linenomath*} \begin{equation} \begin{aligned}
        \Bigl|
        J_{\rm MISL}(\pi_{\{.,.\}}^\skill; s)
        -
        I^{\pi_{\{.,.\}}}\!\left(G; S' \mid S_0=s \right)
        \Bigr|
        \leq
        h(\delta)
        +
        \delta \log \big( N_{s'}^2 (N_{s'}-1) \big),
    \end{aligned} \end{equation} \end{linenomath*}
    with $\delta := || p_{f} - {\rm Unif}(\Zz) ||_{\rm TV}$ the total variation distance and $h$ the binary entropy function.
\end{myprop}
\begin{exmode}
\textit{Proof:}
The proof of \autoref{prop:MIZunifvsf} is essentially independent of the choice of $p_\goal$, hence the statement and its proof remain intact.
$\hfill \square$
\end{exmode}

\newpage
\section*{NeurIPS Paper Checklist}

\begin{enumerate}

\item {\bf Claims}
    \item[] Question: Do the main claims made in the abstract and introduction accurately reflect the paper's contributions and scope?
    \item[] Answer: \answerYes{} \item[] Justification: The contributions are explicitly enumerated in the introduction and are accompanied by references to the corresponding results, including Theorems, Propositions, and Figures.
    Additionally, the introduction directs readers to the summary of the standing assumptions in \autoref{sec:prelim} and to the discussion of limitations in \autoref{sec:conc}.

\item {\bf Limitations}
    \item[] Question: Does the paper discuss the limitations of the work performed by the authors?
    \item[] Answer: \answerYes{} \item[] Justification: \autoref{sec:conc} includes an explicit limitation section where we discuss the main assumptions of our theoretical framework and whether and how they can be relaxed.
    The assumptions are also summarized in \autoref{sec:prelim}.

\item {\bf Theory assumptions and proofs}
    \item[] Question: For each theoretical result, does the paper provide the full set of assumptions and a complete (and correct) proof?
    \item[] Answer: \answerYes{} \item[] Justification: All theoretical statements in the main text are accompanied by an explicit statement of the assumptions (the summary presented in \autoref{sec:prelim} and visualized in \autoref{fig:taxonomy} and \autoref{fig:GSenEmp}), with the proofs provided in \autoref{append:proofs}.
    Furthermore, \autoref{append:props}, \ref{append:props_snes}, \ref{append:props_emp}, \ref{append:upper}, \ref{append:assump}, and \ref{append:goal_dist} provide the precise statements, assumptions, and proofs for all other theoretical claims made in the main text.

    \item {\bf Experimental result reproducibility}
    \item[] Question: Does the paper fully disclose all the information needed to reproduce the main experimental results of the paper to the extent that it affects the main claims and/or conclusions of the paper (regardless of whether the code and data are provided or not)?
    \item[] Answer: \answerNA{} \item[] Justification: The paper does not include any experiment.

\item {\bf Open access to data and code}
    \item[] Question: Does the paper provide open access to the data and code, with sufficient instructions to faithfully reproduce the main experimental results, as described in supplemental material?
    \item[] Answer: \answerNA{} \item[] Justification: The paper does not include any experiment.

\item {\bf Experimental setting/details}
    \item[] Question: Does the paper specify all the training and test details (e.g., data splits, hyperparameters, how they were chosen, type of optimizer) necessary to understand the results?
    \item[] Answer: \answerNA{} \item[] Justification: The paper does not include any experiment.

\item {\bf Experiment statistical significance}
    \item[] Question: Does the paper report error bars suitably and correctly defined or other appropriate information about the statistical significance of the experiments?
    \item[] Answer: \answerNA{} \item[] Justification: The paper does not include any experiment.

\item {\bf Experiments compute resources}
    \item[] Question: For each experiment, does the paper provide sufficient information on the computer resources (type of compute workers, memory, time of execution) needed to reproduce the experiments?
    \item[] Answer: \answerNA{} \item[] Justification: The paper does not include any experiment.
    
\item {\bf Code of ethics}
    \item[] Question: Does the research conducted in the paper conform, in every respect, with the NeurIPS Code of Ethics \url{https://neurips.cc/public/EthicsGuidelines}?
    \item[] Answer: \answerYes{} \item[] Justification: The research fully adheres to the NeurIPS Code of Ethics.

\item {\bf Broader impacts}
    \item[] Question: Does the paper discuss both potential positive societal impacts and negative societal impacts of the work performed?
    \item[] Answer: \answerNA{} \item[] Justification: The paper develops a foundational mathematical theory for RL pretraining, but it does not introduce new algorithms, models, or datasets that could be directly deployed. 
    Hence, all societal impacts are indirect, mediated through potential, future extension of the work to practical domains.
    
\item {\bf Safeguards}
    \item[] Question: Does the paper describe safeguards that have been put in place for responsible release of data or models that have a high risk for misuse (e.g., pre-trained language models, image generators, or scraped datasets)?
    \item[] Answer: \answerNA{} \item[] Justification: The paper develops a foundational mathematical theory for RL pretraining, but it does not introduce new models or datasets that could be misused.

\item {\bf Licenses for existing assets}
    \item[] Question: Are the creators or original owners of assets (e.g., code, data, models), used in the paper, properly credited and are the license and terms of use explicitly mentioned and properly respected?
    \item[] Answer: \answerNA{} \item[] Justification: The paper uses no code, datasets, or models; all results are purely theoretical.

\item {\bf New assets}
    \item[] Question: Are new assets introduced in the paper well documented and is the documentation provided alongside the assets?
    \item[] Answer: \answerNA{} \item[] Justification: The paper uses no code, datasets, or models; all results are purely theoretical.

\item {\bf Crowdsourcing and research with human subjects}
    \item[] Question: For crowdsourcing experiments and research with human subjects, does the paper include the full text of instructions given to participants and screenshots, if applicable, as well as details about compensation (if any)? 
    \item[] Answer: \answerNA{} \item[] Justification: The paper does not involve crowdsourcing nor research with human subjects.

\item {\bf Institutional review board (IRB) approvals or equivalent for research with human subjects}
    \item[] Question: Does the paper describe potential risks incurred by study participants, whether such risks were disclosed to the subjects, and whether Institutional Review Board (IRB) approvals (or an equivalent approval/review based on the requirements of your country or institution) were obtained?
    \item[] Answer: \answerNA{} \item[] Justification: The paper does not involve crowdsourcing nor research with human subjects.

\item {\bf Declaration of LLM usage}
    \item[] Question: Does the paper describe the usage of LLMs if it is an important, original, or non-standard component of the core methods in this research? Note that if the LLM is used only for writing, editing, or formatting purposes and does \emph{not} impact the core methodology, scientific rigor, or originality of the research, declaration is not required.
\item[] Answer: \answerYes{} \item[] Justification: The extent of the use of LLMs is described in the appendix.

\end{enumerate} 


\medskip


{
\small

\begin{thebibliography}{10}

\bibitem{mnih2015human}
V.~Mnih, K.~Kavukcuoglu, D.~Silver, A.~A. Rusu, J.~Veness, M.~G. Bellemare, A.~Graves, M.~Riedmiller, A.~K. Fidjeland, G.~Ostrovski, {\em et~al.}, ``Human-level control through deep reinforcement learning,'' {\em Nature}, vol.~518, pp.~529--533, 2015.

\bibitem{silver2016mastering}
D.~Silver, A.~Huang, C.~J. Maddison, A.~Guez, L.~Sifre, G.~Van Den~Driessche, J.~Schrittwieser, I.~Antonoglou, V.~Panneershelvam, M.~Lanctot, {\em et~al.}, ``Mastering the game of {G}o with deep neural networks and tree search,'' {\em Nature}, vol.~529, no.~7587, pp.~484--489, 2016.

\bibitem{brown2020language}
T.~Brown, B.~Mann, N.~Ryder, M.~Subbiah, J.~D. Kaplan, P.~Dhariwal, A.~Neelakantan, P.~Shyam, G.~Sastry, A.~Askell, {\em et~al.}, ``Language models are few-shot learners,'' in {\em Advances in Neural Information Processing Systems}, vol.~33, pp.~1877--1901, 2020.

\bibitem{wei2021finetuned}
J.~Wei, M.~Bosma, V.~Y. Zhao, K.~Guu, A.~W. Yu, B.~Lester, N.~Du, A.~M. Dai, and Q.~V. Le, ``Finetuned language models are zero-shot learners,'' {\em arXiv preprint arXiv:2109.01652}, 2021.

\bibitem{radford2021learning}
A.~Radford, J.~W. Kim, C.~Hallacy, A.~Ramesh, G.~Goh, S.~Agarwal, G.~Sastry, A.~Askell, P.~Mishkin, J.~Clark, {\em et~al.}, ``Learning transferable visual models from natural language supervision,'' in {\em International Conference on Machine Learning}, pp.~8748--8763, PMLR, 2021.

\bibitem{he2022masked}
K.~He, X.~Chen, S.~Xie, Y.~Li, P.~Doll{\'a}r, and R.~Girshick, ``Masked autoencoders are scalable vision learners,'' in {\em Proceedings of the IEEE/CVF conference on computer vision and pattern recognition}, pp.~16000--16009, 2022.

\bibitem{kaelbling1993learning}
L.~P. Kaelbling, ``Learning to achieve goals,'' in {\em IJCAI}, 1993.

\bibitem{sutton2011horde}
R.~S. Sutton, J.~Modayil, M.~Delp, T.~Degris, P.~M. Pilarski, A.~White, and D.~Precup, ``Horde: A scalable real-time architecture for learning knowledge from unsupervised sensorimotor interaction,'' in {\em The 10th international conference on autonomous agents and multiagent systems-volume 2}, pp.~761--768, 2011.

\bibitem{schaul2015universal}
T.~Schaul, D.~Horgan, K.~Gregor, and D.~Silver, ``Universal value function approximators,'' in {\em International Conference on Machine Learning}, pp.~1312--1320, PMLR, 2015.

\bibitem{veeriah2018many}
V.~Veeriah, J.~Oh, and S.~Singh, ``Many-goals reinforcement learning,'' {\em arXiv preprint arXiv:1806.09605}, 2018.

\bibitem{florensa2018automatic}
C.~Florensa, D.~Held, X.~Geng, and P.~Abbeel, ``Automatic goal generation for reinforcement learning agents,'' in {\em International Conference on Machine Learning}, pp.~1515--1528, PMLR, 2018.

\bibitem{park2025dual}
S.~Park, D.~Mann, and S.~Levine, ``Dual goal representations,'' {\em arXiv preprint arXiv:2510.06714}, 2025.

\bibitem{bortkiewicz2024accelerating}
M.~Bortkiewicz, W.~Pa{\l}ucki, V.~Myers, T.~Dziarmaga, T.~Arczewski, {\L}.~Kuci{\'n}ski, and B.~Eysenbach, ``Accelerating goal-conditioned {RL} algorithms and research,'' {\em arXiv preprint arXiv:2408.11052}, 2024.

\bibitem{wang2023optimal}
T.~Wang, A.~Torralba, P.~Isola, and A.~Zhang, ``Optimal goal-reaching reinforcement learning via quasimetric learning,'' in {\em International Conference on Machine Learning}, pp.~36411--36430, PMLR, 2023.

\bibitem{eysenbach2020c}
B.~Eysenbach, R.~Salakhutdinov, and S.~Levine, ``C-learning: Learning to achieve goals via recursive classification,'' in {\em International Conference on Learning Representations}, 2021.

\bibitem{eysenbach2022contrastive}
B.~Eysenbach, T.~Zhang, S.~Levine, and R.~Salakhutdinov, ``Contrastive learning as goal-conditioned reinforcement learning,'' in {\em Advances in Neural Information Processing Systems} (A.~H. Oh, A.~Agarwal, D.~Belgrave, and K.~Cho, eds.), 2022.

\bibitem{mendonca2021discovering}
R.~Mendonca, O.~Rybkin, K.~Daniilidis, D.~Hafner, and D.~Pathak, ``Discovering and achieving goals via world models,'' in {\em Advances in Neural Information Processing Systems} (M.~Ranzato, A.~Beygelzimer, Y.~Dauphin, P.~Liang, and J.~W. Vaughan, eds.), vol.~34, pp.~24379--24391, Curran Associates, Inc., 2021.

\bibitem{ma2022offline}
J.~Y. Ma, J.~Yan, D.~Jayaraman, and O.~Bastani, ``Offline goal-conditioned reinforcement learning via $ f $-advantage regression,'' in {\em Advances in Neural Information Processing Systems}, vol.~35, 2022.

\bibitem{park2023metra}
S.~Park, O.~Rybkin, and S.~Levine, ``{METRA}: Scalable unsupervised {RL} with metric-aware abstraction,'' in {\em International Conference on Learning Representations}, 2024.

\bibitem{gregor2016variational}
K.~Gregor, D.~J. Rezende, and D.~Wierstra, ``Variational intrinsic control,'' {\em arXiv preprint arXiv:1611.07507}, 2016.

\bibitem{achiam2018variational}
J.~Achiam, H.~Edwards, D.~Amodei, and P.~Abbeel, ``Variational option discovery algorithms,'' {\em arXiv preprint arXiv:1807.10299}, 2018.

\bibitem{eysenbach2018diversity}
B.~Eysenbach, A.~Gupta, J.~Ibarz, and S.~Levine, ``Diversity is all you need: Learning skills without a reward function,'' in {\em International Conference on Learning Representations}, 2019.

\bibitem{sharma2019dynamics}
A.~Sharma, S.~Gu, S.~Levine, V.~Kumar, and K.~Hausman, ``Dynamics-aware unsupervised discovery of skills,'' in {\em International Conference on Learning Representations}, 2020.

\bibitem{zheng2024can}
C.~Zheng, J.~Tuyls, J.~Peng, and B.~Eysenbach, ``Can a {MISL} fly? {A}nalysis and ingredients for mutual information skill learning,'' in {\em International Conference on Learning Representations}, 2025.

\bibitem{park2022lipschitz}
S.~Park, J.~Choi, J.~Kim, H.~Lee, and G.~Kim, ``Lipschitz-constrained unsupervised skill discovery,'' in {\em International Conference on Learning Representations}, 2022.

\bibitem{levy2023hierarchical}
A.~Levy, S.~Rammohan, A.~Allievi, S.~Niekum, and G.~Konidaris, ``Hierarchical empowerment: Towards tractable empowerment-based skill learning,'' {\em arXiv preprint arXiv:2307.02728}, 2023.

\bibitem{eysenbach2021information}
B.~Eysenbach, R.~Salakhutdinov, and S.~Levine, ``The information geometry of unsupervised reinforcement learning,'' in {\em International Conference on Learning Representations}, 2022.

\bibitem{laskin2022cic}
M.~Laskin, H.~Liu, X.~B. Peng, D.~Yarats, A.~Rajeswaran, and P.~Abbeel, ``{CIC}: Contrastive intrinsic control for unsupervised skill discovery,'' {\em arXiv preprint arXiv:2202.00161}, 2022.

\bibitem{ghosh2019learning}
D.~Ghosh, A.~Gupta, A.~Reddy, J.~Fu, C.~Devin, B.~Eysenbach, and S.~Levine, ``Learning to reach goals via iterated supervised learning,'' in {\em International Conference on Learning Representations}, 2019.

\bibitem{pong2019skew}
V.~H. Pong, M.~Dalal, S.~Lin, A.~Nair, S.~Bahl, and S.~Levine, ``Skew-fit: State-covering self-supervised reinforcement learning,'' {\em arXiv preprint arXiv:1903.03698}, 2019.

\bibitem{warde2018unsupervised}
D.~Warde-Farley, T.~Van~de Wiele, T.~Kulkarni, C.~Ionescu, S.~Hansen, and V.~Mnih, ``Unsupervised control through non-parametric discriminative rewards,'' {\em arXiv preprint arXiv:1811.11359}, 2018.

\bibitem{agarwal2023f}
S.~Agarwal, I.~Durugkar, P.~Stone, and A.~Zhang, ``f-policy gradients: A general framework for goal-conditioned {RL} using f-divergences,'' in {\em Advances in Neural Information Processing Systems}, vol.~36, 2023.

\bibitem{modirshanechi2025integrative}
A.~Modirshanechi, P.~Dayan, and E.~Schulz, ``An integrative framework for the human sense of control,'' {\em PsyArXiv}, 2025.

\bibitem{andrychowicz2017hindsight}
M.~Andrychowicz, F.~Wolski, A.~Ray, J.~Schneider, R.~Fong, P.~Welinder, B.~McGrew, J.~Tobin, O.~Pieter~Abbeel, and W.~Zaremba, ``Hindsight experience replay,'' in {\em Advances in Neural Information Processing Systems} (I.~Guyon, U.~V. Luxburg, S.~Bengio, H.~Wallach, R.~Fergus, S.~Vishwanathan, and R.~Garnett, eds.), vol.~30, Curran Associates, Inc., 2017.

\bibitem{bansal2017hamilton}
S.~Bansal, M.~Chen, S.~Herbert, and C.~J. Tomlin, ``Hamilton-{J}acobi reachability: {A} brief overview and recent advances,'' in {\em 2017 IEEE 56th Annual Conference on Decision and Control (CDC)}, pp.~2242--2253, IEEE, 2017.

\bibitem{chilakamarri2025reachability}
V.~K. Chilakamarri, Z.~Feng, and S.~Bansal, ``Reachability analysis for black-box dynamical systems,'' in {\em 2025 IEEE International Conference on Robotics and Automation (ICRA)}, pp.~3552--3558, IEEE, 2025.

\bibitem{bansal2021deepreach}
S.~Bansal and C.~J. Tomlin, ``Deepreach: A deep learning approach to high-dimensional reachability,'' in {\em 2021 IEEE International Conference on Robotics and Automation (ICRA)}, pp.~1817--1824, IEEE, 2021.

\bibitem{abate2008probabilistic}
A.~Abate, M.~Prandini, J.~Lygeros, and S.~Sastry, ``Probabilistic reachability and safety for controlled discrete time stochastic hybrid systems,'' {\em Automatica}, vol.~44, no.~11, pp.~2724--2734, 2008.

\bibitem{thorpe2019model}
A.~J. Thorpe and M.~M. Oishi, ``Model-free stochastic reachability using kernel distribution embeddings,'' {\em IEEE Control Systems Letters}, vol.~4, no.~2, pp.~512--517, 2019.

\bibitem{thorpe2021approximate}
A.~J. Thorpe, V.~Sivaramakrishnan, and M.~M. Oishi, ``Approximate stochastic reachability for high dimensional systems,'' in {\em 2021 American Control Conference (ACC)}, pp.~1287--1293, IEEE, 2021.

\bibitem{sontag2013mathematical}
E.~D. Sontag, {\em Mathematical Control Theory: Deterministic Finite Dimensional Systems}.
\newblock Springer New York, NY, 2013.

\bibitem{ogata2020modern}
K.~Ogata, {\em Modern Control Engineering}.
\newblock Prentice Hall, 5th~ed., 2010.

\bibitem{klyubin2005empowerment}
A.~Klyubin, D.~Polani, and C.~Nehaniv, ``Empowerment: a universal agent-centric measure of control,'' in {\em 2005 IEEE Congress on Evolutionary Computation}, vol.~1, pp.~128--135, 2005.

\bibitem{salge2014empowerment}
C.~Salge, C.~Glackin, and D.~Polani, ``Empowerment--an introduction,'' {\em Guided Self-Organization: Inception}, pp.~67--114, 2014.

\bibitem{jung2011empowerment}
T.~Jung, D.~Polani, and P.~Stone, ``Empowerment for continuous agent—environment systems,'' {\em Adaptive Behavior}, vol.~19, no.~1, pp.~16--39, 2011.

\bibitem{leibfried2019unified}
F.~Leibfried, S.~Pascual-D\'{\i}az, and J.~Grau-Moya, ``A unified {Bellman} optimality principle combining reward maximization and empowerment,'' in {\em Advances in Neural Information Processing Systems} (H.~Wallach, H.~Larochelle, A.~Beygelzimer, F.~d\textquotesingle Alch\'{e}-Buc, E.~Fox, and R.~Garnett, eds.), vol.~32, Curran Associates, Inc., 2019.

\bibitem{bharadhwaj2022information}
H.~Bharadhwaj, M.~Babaeizadeh, D.~Erhan, and S.~Levine, ``Information prioritization through empowerment in visual model-based {RL},'' in {\em International Conference on Learning Representations}, 2022.

\bibitem{becker2021exploration}
P.~Becker-Ehmck, M.~Karl, J.~Peters, and P.~van~der Smagt, ``Exploration via empowerment gain: Combining novelty, surprise and learning progress,'' in {\em ICML 2021 Workshop on Unsupervised Reinforcement Learning}, 2021.

\bibitem{gruaz2024merits}
L.~Gruaz, A.~Modirshanechi, S.~Becker, and J.~Brea, ``Merits of curiosity: A simulation study,'' {\em Open Mind}, vol.~9, pp.~1037--1065, 2025.

\bibitem{cao2025towards}
H.~Cao, F.~Feng, M.~Fang, S.~Dong, T.~Yang, J.~Huo, and Y.~Gao, ``Towards empowerment gain through causal structure learning in model-based reinforcement learning,'' in {\em The Thirteenth International Conference on Learning Representations}, 2025.

\bibitem{burda2018large}
Y.~Burda, H.~Edwards, D.~Pathak, A.~Storkey, T.~Darrell, and A.~A. Efros, ``Large-scale study of curiosity-driven learning,'' in {\em International Conference on Learning Representations}, 2019.

\bibitem{pathak2017curiosity}
D.~Pathak, P.~Agrawal, A.~A. Efros, and T.~Darrell, ``Curiosity-driven exploration by self-supervised prediction,'' in {\em Proceedings of the 34th International Conference on Machine Learning - Volume 70}, ICML'17, pp.~2778--2787, JMLR.org, 2017.

\bibitem{capdepuy2011informational}
P.~Capdepuy, {\em Informational principles of perception-action loops and collective behaviours}.
\newblock PhD thesis, University of Hertfordshire, 2011.

\bibitem{myers2024learning}
V.~Myers, E.~Ellis, S.~Levine, B.~Eysenbach, and A.~Dragan, ``Learning to assist humans without inferring rewards,'' in {\em Advances in Neural Information Processing Systems}, 2024.

\bibitem{abel2025plasticity}
D.~Abel, M.~Bowling, A.~Barreto, W.~Dabney, S.~Dong, S.~Hansen, A.~Harutyunyan, K.~Khetarpal, C.~Lyle, R.~Pascanu, {\em et~al.}, ``Plasticity as the mirror of empowerment,'' {\em arXiv preprint arXiv:2505.10361}, 2025.

\bibitem{choi2021variational}
J.~Choi, A.~Sharma, H.~Lee, S.~Levine, and S.~S. Gu, ``Variational empowerment as representation learning for goal-conditioned reinforcement learning,'' in {\em International Conference on Machine Learning}, pp.~1953--1963, PMLR, 2021.

\bibitem{mohamed2015variational}
S.~Mohamed and D.~Jimenez~Rezende, ``Variational information maximisation for intrinsically motivated reinforcement learning,'' in {\em Advances in Neural Information Processing Systems} (C.~Cortes, N.~Lawrence, D.~Lee, M.~Sugiyama, and R.~Garnett, eds.), vol.~28, Curran Associates, Inc., 2015.

\bibitem{poole2019variational}
B.~Poole, S.~Ozair, A.~Van Den~Oord, A.~Alemi, and G.~Tucker, ``On variational bounds of mutual information,'' in {\em International Conference on Machine Learning}, pp.~5171--5180, PMLR, 2019.

\bibitem{reizinger2025skill}
P.~Reizinger, B.~Mucs{\'a}nyi, S.~Guo, B.~Eysenbach, B.~Sch{\"o}lkopf, and W.~Brendel, ``Skill learning via policy diversity yields identifiable representations for reinforcement learning,'' {\em arXiv preprint arXiv:2507.14748}, 2025.

\bibitem{puterman2014markov}
M.~L. Puterman, {\em Markov Decision Processes: Discrete Stochastic Dynamic Programming}.
\newblock John Wiley \& Sons, 1994.

\bibitem{dayan1993improving}
P.~Dayan, ``Improving generalization for temporal difference learning: The successor representation,'' {\em Neural Computation}, vol.~5, no.~4, pp.~613--624, 1993.

\bibitem{pong2018temporal}
V.~Pong, S.~Gu, M.~Dalal, and S.~Levine, ``Temporal difference models: Model-free deep {RL} for model-based control,'' {\em arXiv preprint arXiv:1802.09081}, 2018.

\bibitem{nair2018visual}
A.~V. Nair, V.~Pong, M.~Dalal, S.~Bahl, S.~Lin, and S.~Levine, ``Visual reinforcement learning with imagined goals,'' in {\em Advances in Neural Information Processing Systems}, vol.~31, 2018.

\bibitem{krass2002achieving}
D.~Krass and O.~J. Vrieze, ``Achieving target state-action frequencies in multichain average-reward {M}arkov decision processes,'' {\em Mathematics of Operations Research}, vol.~27, no.~3, pp.~545--566, 2002.

\bibitem{dufour2022maximizing}
F.~Dufour and T.~Prieto-Rumeau, ``Maximizing the probability of visiting a set infinitely often for a countable state space {M}arkov decision process,'' {\em Journal of Mathematical Analysis and Applications}, vol.~505, no.~2, p.~125639, 2022.

\bibitem{yu2020meta}
T.~Yu, D.~Quillen, Z.~He, R.~Julian, K.~Hausman, C.~Finn, and S.~Levine, ``Meta-world: A benchmark and evaluation for multi-task and meta reinforcement learning,'' in {\em Proceedings of the Conference on Robot Learning} (L.~P. Kaelbling, D.~Kragic, and K.~Sugiura, eds.), vol.~100 of {\em Proceedings of Machine Learning Research}, pp.~1094--1100, PMLR, 2020.

\bibitem{cover1999elements}
T.~M. Cover, {\em Elements of Information Theory}.
\newblock John Wiley \& Sons, 1999.

\bibitem{rioul2023interplay}
O.~Rioul, ``The interplay between error, total variation, alpha-entropy and guessing: {Fano} and {Pinsker} direct and reverse inequalities,'' {\em Entropy}, vol.~25, no.~7, p.~978, 2023.

\bibitem{tebbe1968uncertainty}
D.~Tebbe and S.~Dwyer, ``Uncertainty and the probability of error (corresp.),'' {\em IEEE Transactions on {I}nformation {T}heory}, vol.~14, no.~3, pp.~516--518, 1968.

\bibitem{sason2013entropy}
I.~Sason, ``Entropy bounds for discrete random variables via maximal coupling,'' {\em IEEE Transactions on Information Theory}, vol.~59, no.~11, pp.~7118--7131, 2013.

\end{thebibliography}
 
}
\end{document}